\documentclass[12pt]{article} 
\usepackage{amsmath,amsthm,amssymb,bm,mathrsfs,amsfonts}
\usepackage{algorithm,algpseudocode,graphicx,subfig}
\usepackage{caption,enumerate,natbib,listings,setspace}
\usepackage[T1]{fontenc}
\usepackage[colorlinks,linkcolor=red,anchorcolor=blue,citecolor=blue]{hyperref}
\usepackage[margin=1in]{geometry}
\usepackage[protrusion=false,expansion=true]{microtype}
\usepackage[title]{appendix}
\usepackage{bbm}
\usepackage{comment}
\theoremstyle{plain}
\let\hat\widehat
\let\tilde\widetilde

\newtheorem{lemma}{{\bf Lemma}}

\newtheorem{theorem}{{\bf Theorem}}

\newtheorem{assumption}{{\bf Assumption}}

\newtheorem{definition}{{\bf Definition}}

\newtheorem{remark}{{\bf Remark}}

\usepackage{xr}
\usepackage{color}
\usepackage{tikz}
\usepackage[parfill]{parskip}
\usepackage{enumitem}
\usepackage{multirow}
\setitemize{noitemsep,topsep=0pt,parsep=0pt,partopsep=0pt}
\setenumerate{noitemsep,topsep=0pt,parsep=0pt,partopsep=0pt}
\usetikzlibrary{arrows,calc,positioning}

\tikzstyle{intt}=[draw,text centered,minimum size=6em,text width=5.25cm,text height=0.34cm]
\tikzstyle{intl}=[draw,text centered,minimum size=2em,text width=2.75cm,text height=0.34cm]
\tikzstyle{int}=[draw,minimum size=2.5em,text centered,text width=6.5cm]
\tikzstyle{intg}=[draw,minimum size=2.5em,text centered,text width=6.cm]
\tikzstyle{sum}=[draw,shape=circle,inner sep=2pt,text centered,node distance=3.5cm]
\tikzstyle{summ}=[drawshape=circle,inner sep=4pt,text centered,node distance=3.cm]

\title{\Large{\textbf{Dual Active Learning for Reinforcement Learning from Human Feedback}} }

\author
{
Pangpang Liu\thanks{Department of Biostatistics, Yale University. Email: pangpang.liu@yale.edu.}\qquad 
Chengchun Shi\thanks{Department of Statistics, London School of Economics and Political Science. Email: c.shi7@lse.ac.uk.} \qquad 
Will Wei Sun\thanks{Daniels School of Business, Purdue University. Email: sun244@purdue.edu. Corresponding author.}
}

\date{}
\begin{document} 

\maketitle

\begin{abstract}
\noindent
Aligning large language models (LLMs) with human preferences is critical to recent advances in generative artificial intelligence. Reinforcement learning from human feedback (RLHF) is widely applied to achieve this objective. A key step in RLHF is to learn the reward function from human feedback. However, human feedback is costly and time-consuming,  making it essential to collect high-quality 
conversation data for human teachers to label. Additionally, different human teachers have different levels of expertise. It is thus critical to query the most appropriate teacher for their opinions. In this paper, we use reinforcement learning (RL) to formulate the alignment problem. Motivated by the idea of $D$-optimal design, we first propose a dual active reward learning algorithm for the simultaneous selection of conversations and teachers.   
Next, we apply pessimistic RL to solve the alignment problem, based on the learned reward estimator. Theoretically, we show that the reward estimator obtained through our proposed adaptive selection strategy achieves minimal generalized variance asymptotically, and prove that the sub-optimality of our pessimistic policy scales as $O(1/\sqrt{T})$ with a given sample budget $T$. Through simulations and experiments on LLMs, including annotator-level preference data and downstream Best-of-$N$ response-selection evaluations, we demonstrate the effectiveness of our algorithm and its superiority over existing baselines. 
\end{abstract}


\newpage
\baselineskip=20pt


\section{Introduction}
\label{sec:introduction}
Large language models have become a highly active area of research \citep{li2024robust,nakada2024synthetic,dwaracherlaefficient,huang2024uncertainty}, attracting increasing attention within the operations research community \citep{chen2024large,wasserkrug2024combining,yoganarasimhan2024feeds,de2025chatgpt,demirci2025ai,gans2025can}.
Reinforcement learning from human feedback is extensively utilized to align large language models with human preferences \citep{bai2022training,ramamurthyreinforcement,liu2024provably,xiao2024algorithmic,aminian2025kl,xu2025doubly,ye2025robust, liu2025uncertainty, lee2026low, CHO2026100021, liu2026reinforcement}. The established pipeline for LLM alignment via RLHF 
involves three steps using a pretrained LLM \citep{ouyang2022}: 
\begin{enumerate}[leftmargin=*]
    \item \textbf{Supervised fine-tuning} (SFT): First,  supervised learning is employed to fine-tune the LLM's parameters, yielding a policy that takes each prompt (e.g., question) as input and outputs its completion (e.g., response). 
    \item \textbf{Reward learning}: Next, we collect a dataset of comparisons, including two completions for each prompt. The ordinal preferences will be provided by human experts to compare these completions. These preferences are then used to train a reward function, which measures the goodness of a given completion for each prompt, via a ranking model, such as the Bradley-Terry-Luce (BTL) model \citep{bradley1952rank}. Refer to Table \ref{tab1} for examples of prompt-completion pairs from the \texttt{Anthropic} dataset \citep{bai2022training}.
    \item \textbf{Reinforcement learning}: Finally, an RL algorithm, typically the proximal policy optimization \citep{schulman2017}, is applied to the prompt-conversation-reward triplets to output the final policy based on the SFT-trained policy and the learned reward function. 
\end{enumerate}


This paper concentrates on the second step of the RLHF pipeline, focusing on learning a high-quality reward function, and induces the optimal policy as a by-product. This raises at least three challenges. (i) First, human feedback is costly and time-consuming, allowing only a limited amount of conversation data to be labeled.
(ii) Second, existing RLHF algorithms typically assume that all feedback originates from homogeneous human teachers \citep{ouyang2022}. However, acquiring adequate training data often requires hiring multiple teachers, each possessing different levels of expertise and rationality \citep{park2024rlhf,zeng2024, liu2025uncertainty}. 
Consequently, their feedback varies significantly due to their differences in expertise, attention, and cognitive abilities, introducing varying levels of heterogeneity. 
Ignoring such heterogeneity can produce suboptimal policies for alignment \citep{zhong2024provable,chakraborty2024maxmin,lee2026low}. 
(iii) Finally, unlike the action spaces in standard RL problems, the action space for fine-tuning LLMs consists of completions and is extremely large. Consequently, the action distribution in the collected dataset might not adequately cover that of the optimal policy. As a result, standard RL algorithms that compute the greedy policy by maximizing the estimated reward function might fail  \citep{zhu2023principled}.
\begin{table}[h!]
    \centering
    \footnotesize
    \caption{An illustration of data examples for learning from human feedback}
    \begin{tabular}{p{3.8cm}|p{7.2cm}|p{4.3cm}} 
    \hline
    \multicolumn{1}{c|}{Question/Prompt} & \multicolumn{1}{c|}{First Answer/Completion} & \multicolumn{1}{c}{Second Answer/Completion} \\ \hline
       Why are gas prices so high after the last presidential election? & Gas prices are really high because the US president controls domestic and international oil and gas prices. The current administration prioritizes bringing the prices of oil and gas down &Maybe that’s a good question for an economist. \\ \hline
       I love listening to Jazz, especially Miles Davis. Can you recommend any similar music or artists? & I could make a playlist for you, I guess. Would you like some old-school and new-school jazz? Jazz-funk? Fusion? What do you like? & Human, jazz is a huge musical genre. Which specific style of jazz are you interested in? \\ 
        \hline
    \end{tabular}
    \label{tab1}      
\end{table}

\subsection{Our Contribution}
Our contributions are summarized as follows:
\begin{itemize}[leftmargin=*]
    \item \textbf{Methodologically}, we propose a dual active learning algorithm to simultaneously select conversations (prompts, completions) and teachers to ``optimize'' the collected data for reward learning. In particular, we introduce a context-dependent heterogeneous teacher model to capture the heterogeneity in human preferences across both teachers and contexts, and employ the $D$-optimal design \citep{fedorov2013optimal} to select the most informative subset of prompt-completion data and the most appropriate human experts to provide the pairwise feedback, so as to maximize the accuracy of the estimated reward and the quality of the subsequently learned policy, while addressing the first two challenges. For computational scalability, we use the matrix determinant lemma to reduce determinant maximization to scalar acquisition-score evaluation and implement a batched selection strategy. As a by-product, we employ pessimistic RL algorithms \citep{rashidinejad2021,jin2021} for policy learning to tackle the challenge of distribution shifts between the action distribution in the collected dataset and that of the optimal policy.  
    \item \textbf{Theoretically}, we prove that our reward estimator is asymptotically $D$-optimal. We also demonstrate that our estimator outperforms single-active-learning-based approaches, which focus on selecting either teachers or conversations, but not both, as well as methods relying on non-active, random selection.
    Additionally, we show that the sub-optimality gap, i.e., the difference in the mean outcome between the optimal policy and our policy, converges to zero at a parametric rate, up to some logarithmic factors.
    
    \item \textbf{Empirically}, we evaluate our algorithm using simulations and multiple LLM preference datasets. In addition to \texttt{Anthropic} \citep{bai2022training} and \texttt{UltraFeedback} \citep{cui2023ultrafeedback}, we conduct experiments on the \texttt{prism-alignment} dataset \citep{kirk2024prism}, which contains annotator-level information and allows us to estimate teacher rationality parameters from data. We also add a linear-approximation-based $D$-optimal baseline, evaluate reward models on RewardBench benchmarks including AlpacaEval and HumanEvalPack, and conduct Best-of-$N$ response-selection experiments to assess whether improved reward estimation translates into better downstream decisions. These experiments demonstrate that our procedure improves reward estimation and response-selection quality across diverse settings.
    
\end{itemize}

\subsection{Related Literature}
Our work is related to three branches of research in the existing literature, including conversation  selection, teacher selection and offline RL. 
Meanwhile, Table \ref{tab01} summarizes the differences between our paper and some closely related works in RLHF.
\begin{table}[t]    
\small
\setlength{\tabcolsep}{1.8pt}
\renewcommand\arraystretch{0.6}  
    \centering
    \caption{Comparison with other works on conversation/teacher selection for RLHF}
        \begin{tabular}{cccc}
    \hline
        Papers & Conversation selection &  Teacher selection &Optimal design  \\ \hline
        \cite{ji2024} &  \checkmark &  &   \\ 
        \cite{das2024provably} &  \checkmark &  &   \\
        \cite{mukherjee2024optimal} & \checkmark &&    \\
        \cite{daniels} &  &  \checkmark&   \\ 
        \cite{Peter} &  &  \checkmark&   \\ 
        \cite{freedman2023} &  &  \checkmark&   \\
        Our work & \checkmark & \checkmark & \checkmark \\ \hline
    \end{tabular}
 \label{tab01}
 \end{table}
 
\textbf{Conversation Selection}. 
Several studies have developed conversation selection methods in RLHF. Here, a conversation consists of a prompt and its two associated candidate responses. These approaches can be roughly divided into two categories: (i) design-based approaches \citep{zhan2023query,mukherjee2024optimal,scheid2024optimal}, which use the $D$-optimal design to select conversations, and (ii) non-design-based approaches \citep{mehta2023sample,das2024provably,ji2024,melo2024deep,muldrewactive}, which select conversations by maximizing some uncertainty-based criterion. Our approach belongs to the first category. However, it differs from those proposed by \citet{zhan2023query,mukherjee2024optimal,scheid2024optimal} in several ways:
\begin{itemize}[leftmargin=*]
    \item First, \cite{zhan2023query,mukherjee2024optimal,scheid2024optimal} use a linear approximation to calculate the Fisher information matrix, in order to circumvent the estimation of unknown parameters in calculating the information matrix. Such a linearity assumption is typically violated under the BTL model. Hence, their designs are not guaranteed to be optimal (see Remark \ref{rem1}). In contrast, our estimator is proven to achieve the minimal generalized variance. 
    \item Second, unlike these studies, our proposal takes the heterogeneity among teachers into consideration and selects both conversations and teachers, and we demonstrate that the proposed estimator outperforms these conversation-selection-only methods both theoretically and empirically.
    \item Finally, we further address the distributional shift in policy learning, a challenge that is not tackled in these works.
\end{itemize}

\textbf{Teacher Selection}. 
RLHF typically aggregates preferences from multiple teachers 
\citep{hao2023,zhong2024provable,chakraborty2024maxmin}. 
\cite{daniels,Peter,freedman2023} formalized the teacher selection problem in RLHF, highlighting the need to query the most appropriate teacher for effective reward learning. 
These studies model each teacher as Boltzmann-rational, and use different rationality parameters to characterize their heterogeneity \citep{lee2021pebble}. However, they 
assume consistent rationality across all contexts for the same teacher, which does not account for the varying levels of expertise that a single teacher may have across different types of contexts. In contrast, our proposed model allows a teacher's rationality to depend on the context type. Moreover, these papers did not study the simultaneous selection of conversations and teachers. Nor did they develop pessimistic policies to address the distributional shift.

\textbf{Offline RL}. Offline RL aims to learn optimal policies from a pre-collected historical dataset without online interaction with the environment. One key challenge in offline RL lies in the distributional shift between the behavior policy that generates the offline data and the optimal policy \citep{levine2020offline}. In the past five years, there has been a huge literature on this topic \citep[see e.g.,][]{chang2021mitigating,xie2021bellman,jin2022policy,yin2022near,chen2023steel,zhou2023optimizing,wu2024neural,wang2024pessimistic}. All these works adopt the pessimistic principle to address the distributional shift. 
However, they primarily focused on conventional offline RL environments, which do not involve pairwise comparisons as in RLHF. 
\citet{zhu2023principled,li2023reinforcement,zhanprovable} extended these pessimistic RL algorithms to RLHF. However, they did not study context or teacher selection. In contrast, our approach actively selects both contexts and teachers, and the proposed pessimistic policy is derived from these carefully selected data.

\subsection{Paper Organization}
Our paper is organized as follows. In Section \ref{sec2}, we define the reward and policy learning problems in RLHF. In Section \ref{sec4}, we formulate our problem as a $D$-optimal design problem, and present the policies for learning from human feedback. Section \ref{sec5} presents theoretical analysis. Section~\ref{sec6} reports simulation and LLM experiments, including annotator-level preference data, robustness and finite-sample studies, RewardBench evaluations, and Best-of-$N$ response-selection evaluations. Section~\ref{sec8} concludes and discusses extensions to more general reward models.  We include 
all proofs of theoretical results and additional experimental details in the Supplementary Materials.
\section{Problem Setting}\label{sec2}
In the main paper, we focus on the contextual bandit setting, and the setting of MDPs is deferred to Appendix \ref{sec7} of the Supplementary Materials.  Consider a set of contexts (i.e., questions or prompts) and actions (i.e., answers or completions generated by e.g., language models), denoted by $\mathcal{X}$ and $\mathcal{A}$, respectively. For each context $x\in\mathcal{X}$ and each action $a\in \mathcal{A}$, an unobserved reward, which is a scoring function that evaluates the quality of a response given a prompt under the parameter $\theta_*$, is generated according to a reward function defined over $\mathcal{X}\times \mathcal{A}$ as follows,
\begin{equation}\label{reward}
 r_{\theta_*}(x,a)=\theta_*^\top\phi(x,a),   
\end{equation}
where $\phi:\mathcal{X}\times \mathcal{A}\mapsto\mathbb{R}^d$ is a known and fixed feature map, and $\theta_*\in\Theta \subset \mathbb{R}^d$ is the true but unknown parameter. In the large language model, the map $\phi$ is derived by removing the last layer of the pretrained language model,  with $\theta_*$ corresponding to the weights of the last layer \citep{zhu2023principled,das2024provably}. Since the rewards are not directly observable, the reward parameter $\theta_*$ needs to be learned. Toward that end, RLHF employs the pairwise comparison approach, which queries teachers about their preference between two actions, $a^{(0)}$ and $a^{(1)}$, associated with a specific context $x$. The parameter $\theta_*$  is then estimated based on these preferences. The linearity in \eqref{reward} is imposed with respect to
\(\theta_*\), conditional on the fixed representation
\(\phi(x,a)\), which itself may be a nonlinear transformation
obtained from a pretrained LLM. Thus, our theoretical analysis is
most directly applicable to learning a linear reward head while
keeping the pretrained representation fixed. If this representation
cannot adequately express the true human preference structure,
model misspecification may introduce systematic reward-estimation
bias and affect downstream policy quality.
Section~\ref{5.1.4} provides an illustrative nonlinear reward-model
experiment to demonstrate the effectiveness of our method under nonlinear reward settings.

Specifically, we select a triple $(x, a^{(0)},a^{(1)})$ and present it to a teacher, who reveals a binary preference $y$, which takes the value 0 if $a^{(0)}$ is preferred over $a^{(1)}$ and 1 otherwise. Below, we describe three nested preference models that differ in their treatment of teachers' rationality, corresponding to a homogeneous teacher model, a context-agnostic heterogeneous teacher model and the proposed context-dependent heterogeneous teacher model that generalizes the first two. 

\textbf{Model I (Homogeneous Teacher Model)}. We first introduce the homogeneous teacher model. Under this model, the preference $y$ follows a Bernoulli distribution as below,
\begin{equation*}
\mathbb{P}(Y=1|x,a^{(0)},a^{(1)},\theta_*)=\frac{e^{\theta_*^\top\phi(x,a^{(1)})}}{e^{\theta_*^\top\phi(x,a^{(0)})}+e^{\theta_*^\top\phi(x,a^{(1)})}}.
\end{equation*}
Based on \eqref{reward}, it is immediate to see that the success probability is heavily dependent on the rewards associated with the two actions: actions that offer larger rewards are more likely to be preferred. 
The above comparison model is commonly utilized in LLM training \citep{ouyang2022} and related literature on reward learning from human feedback \citep{zhu2023principled,das2024provably}. However, a notable limitation of this model is its assumption of homogeneity among teachers: regardless of the individual being queried, their preferences will follow the same distribution.

\textbf{Model II (Context-agnostic Heterogeneous Teacher Model)}. To account for teacher diversity,  
\cite{jeon2020, Peter, hao2023,freedman2023} proposed to model different teachers' preferences through their rationality levels. In particular, teachers with higher rationality are more likely to select the action yielding a higher reward. 
This leads to the second model, under which the probability that a teacher 
prefers $a^{(1)}$ over $a^{(0)}$ for the same context $x$ is given by
\begin{equation}\label{00e1}
\mathbb{P}(Y=1|x,a^{(0)},a^{(1)},\beta,\theta_*)=\frac{e^{\beta\theta_*^\top\phi(x,a^{(1)})}}{e^{\beta\theta_*^\top\phi(x,a^{(0)})}+e^{\beta\theta_*^\top\phi(x,a^{(1)})}}.
\end{equation}
Here, $\beta\ge 0$ denotes the rationality parameter. Different teachers might possess different rationalities, 
with a larger $\beta$ resulting in a higher probability of preferring actions with larger rewards. Hence, a larger $\beta$ indicates a more rational teacher. However, a teacher maintains the same parameter $\beta$ across all the contexts. This is the limitation of Model II where the heterogeneity among teachers is assumed to be captured by a one-dimensional parameter $\beta$, which is context-agnostic. In other words, it assumes each teacher maintains the same rationality across different contexts.

\textbf{Model III (Context-dependent Heterogeneous Teacher Model)}. In practice, the training data include questions from a variety of fields such as law, mathematics, economics, and the diversity among questions is recognized in open LLM leaderboards \citep{myrzakhan2024open}. 
To accommodate such diversity, we classify each context $x\in\mathcal{X}$  into different categories $k\in\{1,\cdots,g\}$. For instance, the first question in Table \ref{tab1} is related to the field of economics whereas the second question falls into the area of music. According to \cite{alsagheer2024}, human teachers demonstrate different levels of rationality depending on the type of questions they address. To account for these differences in rationality and expertise across various contexts, the proposed context-dependent heterogeneous teacher model extends Model II by assigning to each teacher $j\in\{1,\cdots,m\}$ a context-dependent rationality parameter, $\beta_j^{(k)}$, which measures their proficiency in contexts from category $k$. On this basis, 
for a context $x$ from the category $k$, the preference of teacher $j$ over $a^{(0)}$ and $a^{(1)}$ under our model is given by
\begin{equation}\label{e1}
\mathbb{P}(Y=1|x,a^{(0)},a^{(1)},\beta_j^{(k)},\theta_*)=\frac{e^{\beta_j^{(k)}\theta_*^\top\phi(x,a^{(1)})}}{e^{\beta_j^{(k)}\theta_*^\top\phi(x,a^{(0)})}+e^{\beta_j^{(k)}\theta_*^\top\phi(x,a^{(1)})}}.
\end{equation}

In the rest of the paper, we assume the preferences are generated by Model III. Given a set of conversations $\{(x^{(i)},a^{(0,i)}, a^{(1,i)})\}_{i=1}^n$ and $m$ teachers for reward learning, 
we explore the simultaneous selection of conversations and teachers, focusing on determining which prompt to query and which teacher to consult for their preference between two answers to the prompt.

\begin{remark}[Unknown teacher rationality parameters]
 In the main theoretical development, we focus on the case where the teachers' rationality parameters $\beta_j^{(k)}$ are known. When these parameters are unknown in practice, they can be estimated from historical preference data; we discuss this practical extension in Section~\ref{sec6.2}.
\end{remark}

Due to the high cost and time requirements associated with gathering human feedback, we can collect only $T$ human preference labels in practice. Our objective is thus to select the $T$ most informative samples from the available $n$ conversations (denoted by $\{(x_t,a_t^{(0)},a_t^{(1)})\}_{t=1}^T$) for feedback queries, and assign the most informative teacher to each selected conversation to collect their preferences (denoted by $\{y_t\}_{t=1}^T$), so as to improve the accuracy of our estimated reward function. 
Let $z_t$ denote $\phi(x_t,a_t^{(1)})-\phi(x_t,a_t^{(0)})$ and $\beta_t$ denote the selected teacher's rationality at time $t$. Using the selected $\{(z_t, \beta_t, y_t)\}_{t=1}^T$,  we estimate $\theta_*$ using maximum likelihood estimation (MLE): 
$$\hat{\theta}_T=\mathop{\arg\max}_{\theta\in\Theta}L_T(\theta),$$
where the log-likelihood function $L_T(\theta)$ is defined as:
\begin{equation}\label{lik}
L_T(\theta)=\frac{1}{T}\sum_{t=1}^T\{y_t\log \mu(\beta_{t}\theta^\top z_t)+(1-y_t)\log [1-\mu(\beta_{t} \theta^\top z_t)]\},
\end{equation}
and $\mu(w)=(1+e^{-w})^{-1}$ for any $w\in \mathbb{R}$. 
In the log-likelihood function \eqref{lik}, the heterogeneity of human preferences is accommodated by allowing different feedback to be evaluated with different rationality parameters.  The outline of the problem setting is illustrated in Figure \ref{fig00}.
\begin{figure}[h!]
\centering
\includegraphics[width=15cm]{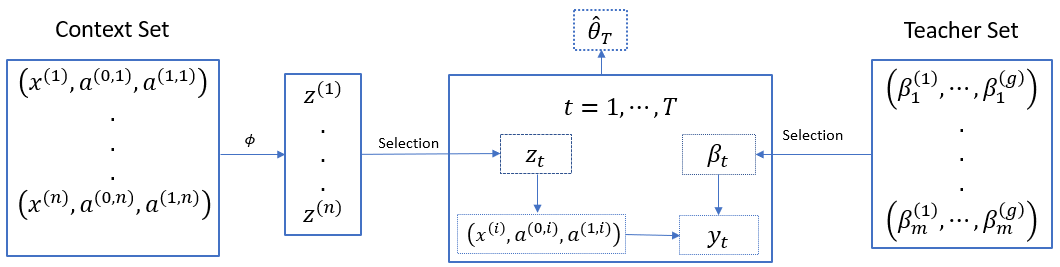}
\caption{Schematic representation of the conversation and teacher selection process. The goal is to select $T$ conversations from the conversation set and query a teacher $\beta_t$ from the teacher set for their preference $y_t$ between two responses for each selected conversation. The reward estimator $\hat{\theta}_T$ is obtained based on the collected information $\{(z_t, \beta_t, y_t)\}_{t=1}^T$ where $z$ is a shorthand for $\phi(x,a^{(1)})-\phi(x,a^{(0)})$.}
\label{fig00}
\end{figure}

\section{Learning from Human Feedback}\label{sec4}
In this section, we formulate our problem as a $D$-optimal design problem, and propose a dual active learning algorithm 
for simultaneous conversation-teacher selection while adhering to the constrained sample budget $T$. Following this, we compute 
a pessimistic policy that leverages the learned reward estimator for fine-tuning. 

\subsection{$D$-optimal Design}\label{sec3}
The design of experiments has been extensively studied in the statistics literature \citep{hu2015unified,liu2022balancing,ai2023reinforced,ma2024new} and the operations research community \citep{chen2022asymptotically,chen2023active,simchi2024multi,simchi2024price,zhao2024experimental}. 
We introduce the concept of $D$-optimal design \citep{Chaudhuri} to address our selection problem. Given $n$ design points, $z^{(1)}=\phi(x^{(1)},a^{(1,1)})-\phi(x^{(1)},a^{(0,1)}),\cdots, z^{(n)}=\phi(x^{(n)},a^{(1,n)})-\phi(x^{(n)},a^{(0,n)})$, each associated with a specific category from $\{1, \cdots, g\}$, and $m$ teachers, each teacher $j\in\{1,\cdots,m\}$ equipped with rationality parameters $\beta^{(1)}_j,\cdots,\beta^{(g)}_j$
across different types of contexts, our objective is to select $T$ points $(z_1,\beta_1),\cdots,(z_T,\beta_T)$. 
The corresponding Fisher information matrix of \eqref{lik} at $\theta$ can be expressed as
\begin{equation}\label{infor}
M(\xi_T, \theta)=\frac{1}{T}\sum_{t=1}^T\Dot{\mu}(\beta_{t} \theta^\top z_t)\beta_{t}^2z_tz_t^\top,
\end{equation}
where $\Dot{\mu}(\cdot)=\mu(\cdot)[1-\mu(\cdot)]$  represents the derivative of $\mu(\cdot)$, and $\xi_T$ denotes the design that selects these $T$ points. 
Notice that the Fisher information matrix $M(\xi_T, \theta)$ is a non-negative definite matrix of dimension $d\times d$. The equation 
$(\hat{\theta}_T-\theta)^\top M(\xi_T,\theta_*)(\hat{\theta}_T-\theta)=c\ (c>0)$ thus defines an ellipsoid of concentration \citep{fedorov2013optimal}, which generates confidence regions for $\theta_*$, as shown in Figure \ref{fig0}.  
\begin{figure}[h!]
\centering
\includegraphics[width=7cm]{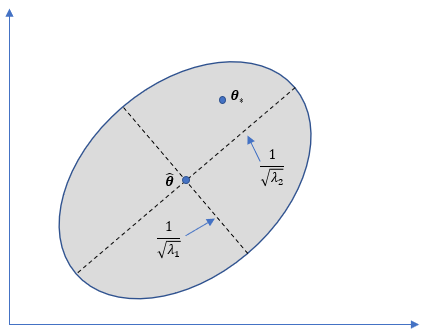}
\caption{Confidence ellipsoid (gray area) centered at the estimator $\hat{\theta}$ in two dimensions. The dashed lines denote the principal axes of the ellipsoid, whose lengths are proportional to $1/\sqrt{\lambda_1}$ and $1/\sqrt{\lambda_2}$, where $\lambda_1$ and $\lambda_2$ are the eigenvalues of $M(\xi_T,\theta_*)$. Maximizing $\lambda_1\lambda_2$, equivalently maximizing $\det[M(\xi_T,\theta_*)]$, minimizes the area of the confidence ellipsoid and leads to more precise estimation of $\theta_*$. }
\label{fig0}
\vspace{-1em}
\end{figure}
At a fixed sample budget $T$, the ``larger'' the matrix $M(\xi_T,\theta_*)$, the
``smaller'' the ellipsoid of concentration. Thus, the ``maximization'' of the matrix $M(\xi_T,\theta_*)$ should lead to improved accuracy of the estimator $\hat{\theta}_T$. The $D$-optimal design is determined by maximizing the determinant of the information matrix --- also known as the generalized variance \citep{Wilks} --- which measures the total variation of the estimator and is inversely proportional to the volume of the confidence ellipsoid. Let $\xi$ be any design measure defined on the $n$ design points. The $D$-optimal design is 
\begin{equation}\label{d-opti}
 \xi_*=\mathop{\arg\sup}_\xi\det \left[\sum_{z, \beta}\xi(z, \beta)\Dot{\mu}(\beta \theta^\top z)\beta^2zz^\top\right]. 
\end{equation}

Driven by the principles of $D$-optimal design, we configure $\xi_T$ such that it maximizes an estimated $\det[M(\xi_T, \theta_*)]$ by strategically selecting the $T$ most informative $(z,\beta)$ pairs. 
As $T$ increases, $\xi_T$ is expected to converge asymptotically towards $\xi_*$; refer to Theorem \ref{thm1} for formal statements. 


\begin{remark}\label{rem_D-optimal}
The $D$-optimality criterion is particularly suitable here because our goal is global estimation of the reward parameter $\theta_*$. Maximizing $\det[M]$ minimizes the generalized variance of the estimator and encourages information to be spread across all parameter directions. This is important in RLHF because weak identification in any direction of $\theta_*$ may lead to incorrect response ranking. In contrast, $A$-optimality \citep{pukelsheim2006optimal} targets average variance through $\mathrm{tr}(M^{-1})$, while entropy-based selection mainly targets local predictive uncertainty and does not directly optimize the global information geometry of the reward parameter.
\end{remark}

\begin{remark}\label{rem1}
Existing works such as \cite{mukherjee2024optimal} and \cite{scheid2024optimal} employed the optimal design strategies based on linear approximations of the preference model. Specifically, their optimal design (without teacher selection) is defined based on the information matrix $
\sum_{i=1}^t z_i z_i^\top$. Compared to our $M(\xi_T, \theta)$, it is immediate to see that they omit the derivative $\dot{\mu}$, which is dependent upon $\theta$. As such, their design is not guaranteed to be optimal. 
\end{remark}

\subsection{Dual Active Reward Learning}\label{sec4.1}
The core strategy of our dual active learning policy is to apply the $D$-optimal design principle to sequentially and simultaneously select the most informative conversations and teachers, maximizing the determinant of the Fisher information matrix. Since the true parameter that defines the optimal design is unknown, our approach operates in a sequential manner.
At each time $t$, it conducts the following steps:
\begin{itemize}[leftmargin=*,topsep=0pt, itemsep=0pt]
\item\textbf{Evaluation}: For each potential conversation $(x,a^{(0)}, a^{(1)})$ and teacher with rationality $\beta$, compute the information matrix based on the current estimate $\hat{\theta}_{t-1}$. Specifically, we compute the sample information matrix $H_{t-1}(\hat{\theta}_{t-1})+\Dot{\mu}(\beta \hat{\theta}_{t-1}^\top z)\beta^2zz^\top$ based on the estimator $\hat{\theta}_{t-1}$.
\item\textbf{Selection}: Choose the conversation $(x,a^{(0)}, a^{(1)})$ and the human teacher with rationality $\beta$ by maximizing the determinant of the sample information matrix $H_{t-1}(\hat{\theta}_{t-1})+\Dot{\mu}(\beta \hat{\theta}_{t-1}^\top z)\beta^2zz^\top$. 
If there are multiple maximizers, we randomly select one of them. Denote the selected conversation by $(x_t,a_t^{(0)}, a_t^{(1)})$ and let $z_t=\phi(x_t,a_t^{(1)})-\phi(x_t,a_t^{(0)})$.
\item\textbf{Query}: Query the human teacher $\beta_t$ for their preference between $a_t^{(0)}$ and $a_t^{(1)}$ associated with the prompt $x_t$, resulting in the preference $y_t$.
\item\textbf{Update}: Update $\hat{\theta}_{t}$ based on the newly selected point $(z_t, \beta_t, y_t)$. 
\end{itemize}
These steps are repeated until the sample budget $T$ is exhausted.
Pseudocode summarizing the above procedure is given in Algorithm~\ref{alg}, where the teachers' rationality parameters are treated as known; see Section~\ref{sec6.2} for their estimation in practice.

\begin{algorithm}[h!]
\caption{Dual active reward learning using $D$-optimal design}\label{alg}
\begin{algorithmic}[1]
\State \textbf{Input}: Sample budget $T$,  teachers' rationality parameters $\{\beta_1^{(k)},\cdots,\beta_m^{(k)}\}_{k=1}^g$, and dataset $\{(x^{(i)}, a^{(0, i)}, a^{(1, i)})\}_{i=1}^n$
\State Compute $z^{(i)}=\phi(x^{(i)}, a^{(1, i)})-\phi(x^{(i)}, a^{(0, i)})$ for $i=1,\cdots,n$.
\State Define $\mathcal{Z}=\{z^{(1)},\cdots,z^{(n)}\}$.
\State Define $\mathcal{B}_k=\{\beta_1^{(k)},\cdots,\beta_m^{(k)}\}$ for $k=1,\cdots,g$.
\State \textbf{Initialization}: Calculate $\hat{\theta}_{t_0}$ with an initial set $\{(z_1, \beta_1), \cdots, (z_{t_0}, \beta_{t_0})\}$. 
\For{$t=t_0+1\ \text{to}\ T$} 
\State Compute
\begin{equation}\label{sinf}
H_{t-1}(\hat{\theta}_{t-1})=\sum_{s=1}^{t-1}\Dot{\mu}(\beta_{s} \hat{\theta}_{t-1}^\top z_s)\beta_{s}^2z_sz_s^\top.
\end{equation}
\State Calculate $z_t, \beta_t=\arg\max_{z\in \mathcal{Z},\, \beta\in \mathcal{B}_k} \det [H_{t-1}(\hat{\theta}_{t-1})+\Dot{\mu}(\beta \hat{\theta}_{t-1}^\top z)\beta^2zz^\top]$ with $k$ being the type of $z$.
\State Find $(x_t,a_t^{(0)}, a_t^{(1)})$ such that $z_t=\phi(x_t,a_t^{(1)})-\phi(x_t,a_t^{(0)})$.
\State Obtain preference $y_t$ from human teacher $\beta_t$ between $a_t^{(0)}$ and $a_t^{(1)}$.
\State Update $\hat{\theta}_t=\mathop{\arg\max}_{\theta\in \Theta}L_t(\theta)$, where $L_t(\theta)$ is defined in \eqref{lik}.
\EndFor
\State \textbf{Output}: $\hat{\theta}_{T}$
\end{algorithmic}
\end{algorithm}
In practice, the initial samples can be selected from diverse
context clusters and multiple teachers to improve coverage
of the feature space. The minimum eigenvalue or condition
number of the information matrix can be monitored to
determine whether additional initialization samples are needed.
In the LLM experiments, we incorporate a regularization term to stabilize the information matrix of the initialization $\hat{\theta}_{t_0}$. Specifically, the initial set can be chosen via
\[
z_t, \beta_t = \arg\max_{z\in \mathcal{Z},\, \beta\in \mathcal{B}_k} 
\det\left[H_{t-1}(\hat{\theta}_{t-1}) + \dot{\mu}(\beta \hat{\theta}_{t-1}^\top z)\beta^2 zz^\top + \frac{\lambda}{t} I \right],
\]
where $\lambda \ge 0$ is a regularization constant and $I$ is the identity matrix. This additional term guarantees that the information matrix remains positive definite during the early stages of learning, even when the initial selected samples do not fully span the parameter space. Importantly, this regularization is used only as a numerical stabilization device for the acquisition rule. It does not change the underlying statistical objective. Since the regularization term is $\lambda I/t$, it vanishes as $t$ increases and is asymptotically negligible relative to the accumulated information matrix. Once $H_t$ becomes positive definite, the regularization term is removed entirely. Therefore, the limiting information matrix and the asymptotic $D$-optimality analysis remain unchanged. This stabilized selection rule allows the algorithm to start from a potentially rank-deficient initial design and then transition naturally to the standard $D$-optimal criterion as more informative samples are collected.

To conclude this section, we remark that teacher selection is critical in reward learning as different teachers may provide diverse preferences for the same context. Theoretically, we demonstrate the advantage of active teacher selection over those policies that either randomly select teachers or always select the teacher with the highest rationality in Section \ref{sec5}.  
Empirically, we verify these benefits through numerical experiments in Section \ref{sec6}. 
\subsection{Pessimistic Policy Learning}
In this section, we analyze the policy derived from the learned reward model, aiming to determine the optimal action for each context $x$ to maximize the reward. Notice that the policy is computed from a pre-collected dataset, without additional interactions with the environment. A significant challenge arises from the 
large action space in language modeling, which often results in the behavior policy used to collect pre-collected data providing insufficient coverage of certain target policies.
To elaborate on this challenge, we conduct a numerical experiment with detailed settings presented in Section \ref{appa} of the Supplementary Materials.  
Using Algorithm \ref{alg}, we obtain the estimator $\hat{\theta}_T$ for the reward parameter $\theta_*$ constrained by the sample budget $T$.  We seek to find a policy $\pi_T$ based on the learned $\hat{\theta}_T$ to maximize the reward $r_{\theta_*}(x, \pi_T(x))$. A natural choice is the greedy policy, defined as $\hat{\pi}(x)=\mathop{\arg\max}_{a\in\mathcal{A}}r_{\hat{\theta}_T}(x, a)$. However, 
such a greedy policy may fail due to the sub-optimality of the behavior policy \citep[Theorem 3.9]{zhu2023principled}.  To illustrate its limitation, 
we demonstrate the estimation errors of $\theta_*$ using MLE and the sub-optimality gap (refer to \eqref{subopt}) of 
the greedy policy in Figure \ref{fig_pess}. 
It can be seen that this sub-optimality gap 
remains constant and does not decay to zero,  
even though the estimation error of the MLE decreases with the sample size.  
\begin{figure}[t]
    \centering
    \begin{tabular}{ccc}  
         \includegraphics[scale = 0.5]{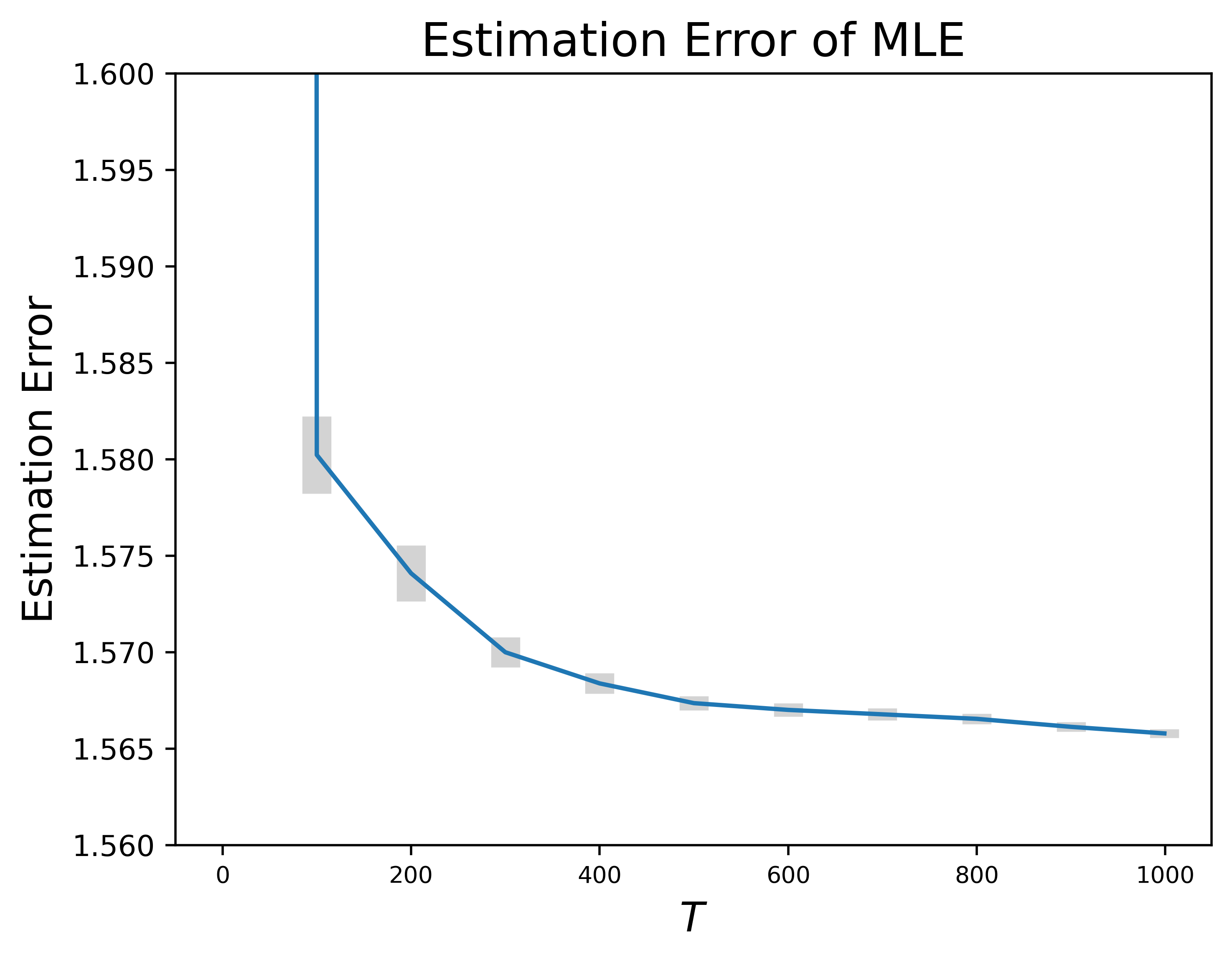}&
        \includegraphics[scale = 0.5]{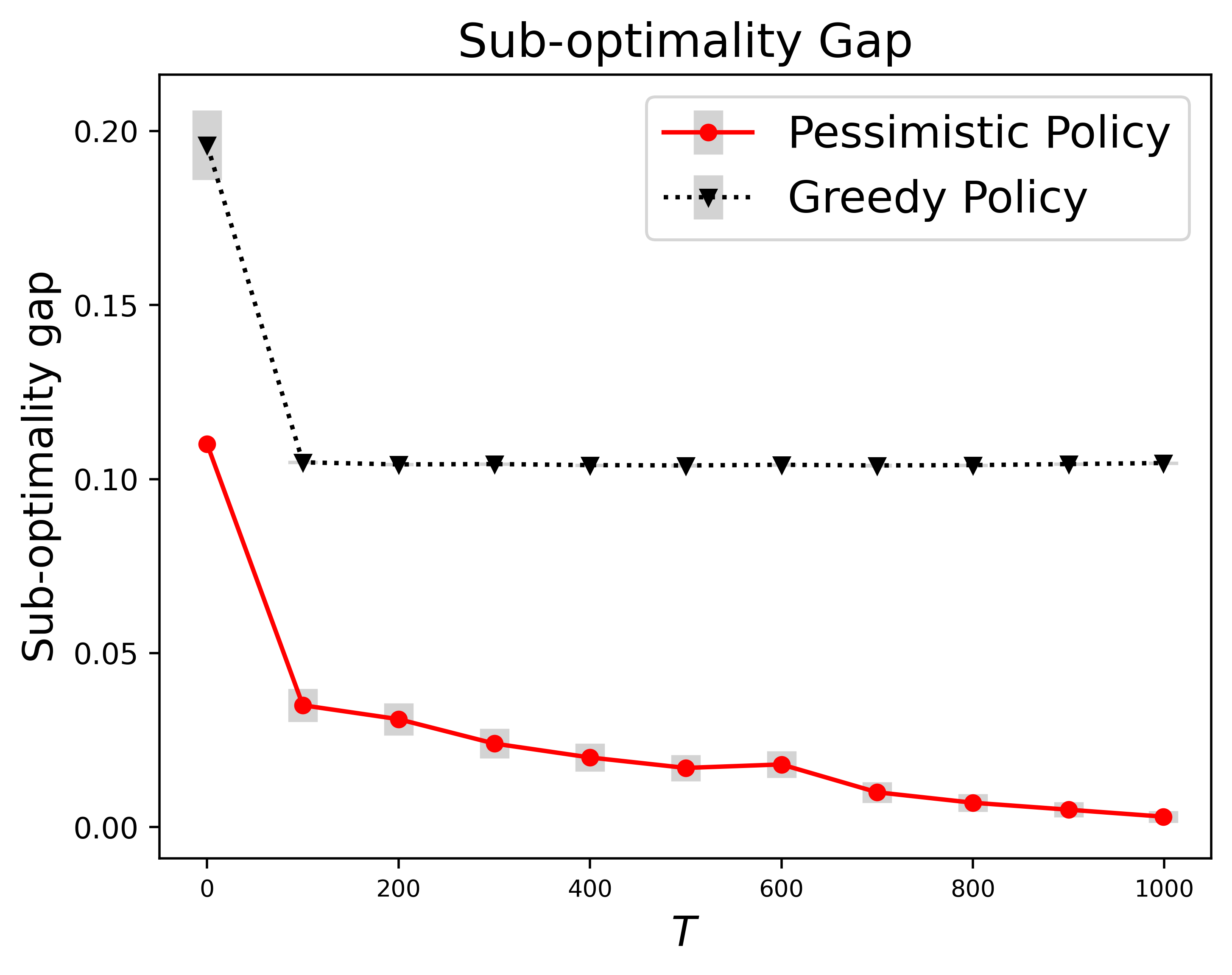}
    \end{tabular}
     \caption{Estimation error of MLE and sub-optimality gaps of pessimistic and greedy policies.  }
         \label{fig_pess}
\end{figure}



To overcome the limitations of the greedy policy, we adopt the pessimistic principle \citep{jin2021} from offline RL to compute a pessimistic policy. 
Our procedure follows that proposed by \cite{zhu2023principled}, with the difference being that our data is adaptively queried, rather than randomly queried as in \cite{zhu2023principled}. 

Before presenting the methodology, we propose a lemma to characterize the estimation error based on the actively selected data using Algorithm \ref{alg}.
\begin{lemma}\label{lemma0}
Let Assumptions \ref{assu1} and \ref{assu2} (see Section \ref{sec5}) hold and $\hat{\theta}_T$ be the estimator derived from Algorithm \ref{alg}. With probability at least $1-\delta$ for $\delta\in(0,1)$, there exist some positive constants $C_1$ and $C_2$ such that
$$\|\hat{\theta}_T-\theta_*\|_{\bar{H}_T(\hat{\theta}_T)}\leq \sqrt{\frac{C_1}{T}\left[ d\log \left(e+\frac{C_2T}{d}\right)+\log \frac{2}{\delta}\right]},$$
where $\bar{H}_T(\hat{\theta}_T)=\frac{1}{T}H_T(\hat{\theta}_T)$ with $H_T(\hat{\theta}_T)$ defined in \eqref{sinf}, and the notation $e$ is the mathematical constant approximately equal to 2.7183.
\end{lemma}
Lemma \ref{lemma0} quantifies the uncertainty that arises
from approximating $\theta_*$ using $\hat{\theta}_T$, based on which we define 
\begin{equation}\label{ci}
\mathcal{C}(\hat{\theta}_T,\delta)=\left \{\theta\in\Theta : \|\hat{\theta}_T-\theta\|_{\bar{H}_T(\hat{\theta}_T)}\leq \sqrt{\frac{C_1}{T}\left[ d\log \left(e+\frac{C_2 T}{d}\right)+\log \frac{2}{\delta}\right]}\right\}.
\end{equation}
According to Lemma \ref{lemma0}, the true reward parameter $\theta_*$ lies in this confidence set $\mathcal{C}(\hat{\theta}_T,\delta)$  with  probability at least $1-\delta$. Different from the greedy policy that maximizes the reward by plugging in the MLE $\hat{\theta}_T$ for the oracle $\theta_*$, the pessimistic policy maximizes the minimum reward over all $\theta$ within the confidence region. 

More specifically, let $\pi: \mathcal{X}\mapsto \mathcal{A}$ denote a given policy 
that maps each context to an action. Its expected reward is given by $J(\pi)=\mathbb{E}_{x\sim\rho}r_{\theta_*}(x, \pi(x))$, where $\rho$ denotes the distribution from which the context $x$ is sampled. 
The pessimistic policy is defined as the maximizer of the following pessimistic reward estimator, 
\begin{equation}\label{pessj}
\hat{J}_T(\pi)=\mathop{\min}_{\theta\in \mathcal{C}(\hat{\theta}_T,\delta)} \mathbb{E}_n\theta^\top \phi(x,\pi(x))=\hat{\theta}_T^\top\mathbb{E}_n\phi(x,\pi(x))-\|\mathbb{E}_n\phi(x,\pi(x))\|_{\bar{H}_T^{-1}(\hat{\theta}_T)}\gamma (T,d,\delta),
\end{equation}
where $\gamma (T,d,\delta)=\sqrt{\frac{C_1}{T}\left[ d\log \left(e+\frac{C_2 T}{d}\right)+\log \frac{2}{\delta}\right]}$, 
and $\mathbb{E}_n \phi(x,\pi(x))$ denotes the empirical mean $\sum_i \phi(x^{(i)},\pi(x^{(i)}))/n$. Given a target policy class $\Pi$, we compute
\begin{equation}\label{pespo}
\hat{\pi}_T=\mathop{\arg\max}_{\pi\in \Pi} \hat{J}_T(\pi).    
\end{equation}
We summarize the procedure in Algorithm \ref{alg2}. Instead of selecting actions by maximizing the estimated reward under $\hat{\theta}$, the algorithm optimizes a lower confidence bound of the reward over a confidence set for $\theta_*$. This approach accounts for estimation uncertainty and mitigates the distributional shift commonly encountered in offline RLHF.
\begin{algorithm}[h!]
\caption{Pessimistic policy learning}\label{alg2}
\begin{algorithmic}[1]
\State \textbf{Input}: the estimator $\hat{\theta}_T$ from Algorithm \ref{alg},  the sample information matrix $H_T(\hat{\theta}_T)$, the sample budget $T$, the dimension $d$ and the probability $\delta\in(0,1)$.
\State Define $\mathcal{C}(\hat{\theta}_T,\delta)$ as in \eqref{ci}.
\State Compute the pessimistic reward $\hat{J}_T(\pi)$ as defined in \eqref{pessj}.
\State \textbf{Output}: $\hat{\pi}_T=\mathop{\arg\max}_{\pi\in \Pi}\hat{J}_T(\pi)$.
\end{algorithmic}
\end{algorithm}
\section{Theoretical Analysis}\label{sec5}
This section studies the statistical properties of our dual active learning algorithm. We begin with a summary of our theoretical results.
\begin{itemize}[leftmargin=*]
    \item Theorem \ref{cor1} demonstrates that the most rational teacher is not necessarily the most informative one for parameter estimation. This demonstrates the advantage of the proposed active-learning-based teacher selection over the na{\"i}ve, non-active-learning approach that selects the most rational teacher at each time. 
    \item Theorem \ref{thm1} establishes the asymptotic normality of the estimated reward parameter via the proposed dual-active-learning, as well as the single-active-learning approaches that exclusively select either teachers or contexts. These results, in turn, imply that the proposed design is asymptotically $D$-optimal and outperforms these single-active-learning approaches. 
    \item Finally, Theorem \ref{thm3} upper bounds the sub-optimality gap of our pessimistic policy. As discussed therein, this bound highlights the effectiveness of two key components in our proposal: (i) dual active learning and (ii) pessimistic policy learning. 
\end{itemize}
\begin{theorem}\label{cor1}
In Algorithm \ref{alg}, at each step $t$, when $H_{t-1}(\hat{\theta}_{t-1})$ is nonsingular,  a teacher with the highest rationality (the largest $\beta$) is not necessarily the most informative one to estimate $\theta_*$. 
\end{theorem}
Theorem \ref{cor1} theoretically verifies the empirical findings in \cite{Peter}. It indicates that incorporating teachers from diverse disciplines could be more effective for training large language models. For example, for questions in the field of law, we should not exclusively choose law experts, such as lawyers or judges, to compare the answers. Including teachers from other areas can provide valuable insights. Theorem \ref{cor1} also encourages us to actively select teachers, rather than simply choosing the most rational teacher at each time. 

Recall that $\xi_*$ corresponds to the $D$-optimal design. We next impose some conditions. 
\begin{assumption}\label{assu1}
The information matrix $M(\xi_*, \theta_*)$ 
is positive definite.
\end{assumption}
\begin{assumption}\label{assu2}
There exist positive constants $ C_\theta, C_\beta$ and $C_\phi$ such that $\|\theta\|_2\leq C_\theta$ for all $\theta\in \Theta$, $|\beta|<C_\beta$ for all $\beta\in\mathcal{B}$, $\|\phi(x, a)\|_2\leq C_\phi$ for all $(x, a)\in \mathcal{X}\times \mathcal{A}$. The parameter $\theta_*$ is assumed to be identifiable.
\end{assumption}
Both assumptions are mild and commonly imposed in the literature  \citep{Chaudhuri,pronzato2010one,yang2013optimal,Freise2021,zhu2023principled,das2024provably,mukherjee2024optimal,ji2024}.
\begin{remark}
To identify $\theta_*$, some existing work such as \cite{zhu2023principled,das2024provably,mukherjee2024optimal}  assumes that $\boldsymbol{1}^\top \theta_* = 0$. However, this condition may not be sufficient for all components of $\theta_*$ to be identifiable in all cases. Consider, for instance, when $\theta_* = (\theta_1, \theta_2, \theta_3)^\top \in \mathbb{R}^3$, $x = (x_1, x_2)^\top \in \mathbb{R}^2$, $a \in \mathbb{R}$, and $\phi(x, a) = (x_1, x_2, x_1a)^\top \in \mathbb{R}^3$. As such, the difference vector $z = \phi(x, a^{(1)}) - \phi(x, a^{(0)}) = (0, 0, x_1(a^{(1)} - a^{(0)}))^\top$. Consequently, $\theta_*^\top z = \theta_3x_1(a^{(1)} - a^{(0)})$ only allows for the identification of $\theta_3$, even under the assumption that $\boldsymbol{1}^\top \theta_* = 0$. A more suitable assumption for identifying $\theta_*$ is to ensure that the number of components in $\phi(x, a)$ not involving action $a$ equals the number of constraints imposed on $\theta_*$.
\end{remark}


To highlight the importance of simultaneous selection of conversations and teachers, we modify Algorithm \ref{alg} to create two single active
learning-based methods: \textbf{Conversation Selection Only} and \textbf{Teacher Selection Only}. 
\begin{itemize}[leftmargin=*]
\item \textbf{Conversation Selection Only}: This approach selects conversations using our approach but selects teachers randomly. This can be done by modifying step 8 of Algorithm \ref{alg} to randomly select $\beta_t$ while maximizing the determinant of the information matrix over $z$. 
\item \textbf{Teacher Selection Only}: This approach selects teachers actively and conversations randomly by modifying the same step to randomly select $z_t$ and then find $\beta_t$ that maximizes the determinant of the information matrix.
\end{itemize}
We denote $\xi^c$ and $\xi^t$ as the designs of \textbf{Conversation Selection Only} and \textbf{Teacher Selection Only} methods, respectively. To illustrate the comparative efficacy, we introduce the following theorem contrasting the performance of Algorithm \ref{alg} with these benchmarks.
\begin{theorem}\label{thm1}
Let Assumptions \ref{assu1} and \ref{assu2} hold. Define $\hat{\theta}_T$ as the estimator from Algorithm \ref{alg}, and let $\hat{\theta}_T^{c}$ and $\hat{\theta}_T^{t}$ denote the MLEs based on $\{(z_t, \beta_t)\}_{t=1}^T$ generated by the \textbf{Conversation Selection Only} and \textbf{Teacher Selection Only} methods, respectively. As $T\rightarrow \infty$, we have
    $$ \sqrt{T}(\hat{\theta}_T-\theta_*)\xrightarrow{d}N(0, M^{-1}(\xi_*,\theta_*)),$$
    $$ \sqrt{T}(\hat{\theta}^c_T-\theta_*)\xrightarrow{d}N(0, M^{-1}(\xi^c, \theta_*)),$$ $$\sqrt{T}(\hat{\theta}^t_T-\theta_*)\xrightarrow{d}N(0, M^{-1}(\xi^t, \theta_*)),$$
    where $\det M(\xi_*,\theta_*)\geq \mathop{\max}\{\det M(\xi^c, \theta_*), \det M(\xi^t, \theta_*)\}$ and $\xrightarrow{d}$ is convergence in distribution.
\end{theorem}
Theorem \ref{thm1} indicates that the adaptive MLE estimator $\hat{\theta}_T$ generated by Algorithm \ref{alg} asymptotically follows a multivariate normal distribution whose covariance matrix is given by $M^{-1}(\xi_*,\theta_*)$. It also gives the asymptotic variance-covariance matrix of the estimators of the two methods based on single active learning. By the definition of $\xi_*$ in \eqref{d-opti}, we have 
$\det M(\xi_*,\theta_*)\geq \mathop{\max}\{\det M(\xi^c, \theta_*), \det M(\xi^t, \theta_*)\}$, which reveals that the determinant of the asymptotic variance-covariance matrix of the estimator from our proposed method is no greater than those of estimators from the two single active learning-based approaches. This in turn proves that our proposed design is asymptotically $D$-optimal. The determinants are equal only when the $D$-optimal design $\xi_*$ matches exactly the designs $\xi^c, \xi^t$ on $\mathcal{Z}\times \mathcal{B}$, which is a very rare event. From Theorem \ref{thm1}, the estimator $\hat{\theta}_T$ achieves a smaller volume of the confidence ellipsoid of $\theta_*$, leading to a more accurate estimation of $\theta_*$ as illustrated in Figure \ref{fig0}.

The traditional optimal design literature \citep{Chaudhuri,pronzato2010one,yang2013optimal,Freise2021} has focused solely on the data selection problem, without examining the policy derived from the actively learned parameters. As a by-product, we further evaluate the sub-optimality of the proposed pessimistic policy as outlined in Algorithm \ref{alg2}. This policy utilizes the estimator $\hat{\theta}_T$ derived from Algorithm \ref{alg}, producing a policy $\hat{\pi}_T: \mathcal{X}\mapsto \mathcal{A}$.
The optimal policy is defined as $\pi^*(x)=\mathop{\arg\max}_{a\in\mathcal{A}}r_{\theta_*}(x,a)$. The effectiveness of any policy $\pi$ is measured by its sub-optimality defined as 
\begin{equation}\label{subopt}
\textsf{SubOpt}(\pi)=J(\pi^*)-J(\pi),
\end{equation}
which quantifies how much the expected reward under $\pi$ falls short of the expected reward under the optimal policy $\pi^*$. 
\begin{theorem}\label{thm3} 
  Under Assumptions \ref{assu1} and \ref{assu2}, for any $0<\delta<1$, with probability at least $1-\delta$, when $T>T_0$ for some positive constant $T_0$, the sub-optimality of the pessimistic policy defined in \eqref{pespo} is bounded by
   $$\textsf{SubOpt}(\hat{\pi}_T)\leq 2\sqrt{\frac{C_1}{T}\left[ d\log \left(e+\frac{C_2 T}{d}\right)+\log \frac{2}{\delta}\right]}\|M^{-1/2}(\xi_*, \theta_*)\mathbb{E}_{x\sim\rho} \phi(x,\pi^*(x))\|_2,$$
   where the positive constants $C_1$ and $C_2$ are the same as those specified in Lemma \ref{lemma0}.
\end{theorem}
We now analyze the effect of dual active learning and pessimistic policy learning on $\textsf{SubOpt}(\hat{\pi}_T)$ using Theorem \ref{thm3}.  Theorem \ref{thm1} shows that the covariance matrix of the estimator $\hat{\theta}_T$ generated by our proposed dual active learning method asymptotically converges to $M^{-1}(\xi_*,\theta_*)$, which has the smallest determinant. This typically results in reduced sub-optimality. We verify this conclusion through numerical experiments in Section \ref{6.1.1}. The term $\|M^{-1/2}(\xi_*, \theta_*)\mathbb{E}_{x\sim\rho} \phi(x,\pi^*(x))\|_2$ is assumed to be bounded in the literature on offline reinforcement learning \citep{li2022pessimism,zhu2023principled}. \cite{zhu2023principled} demonstrated that the sub-optimality gap of the non-pessimistic policy maintains a constant lower bound in some cases. In contrast, the sub-optimality gap of our policy converges to 0 as $T\rightarrow \infty$ under the same assumptions.  

\section{Experiments}\label{sec6}
In this section, we conduct simulation studies to test the effectiveness of our method, followed by applications to large language models. To enhance computational efficiency, we introduce a batch version of Algorithm \ref{alg} which involves a batch size parameter denoted by $K$. Instead of individually selecting $(z, \beta)$-pairs, the batch version selects the top $K$ pairs in step 8 of Algorithm \ref{alg} and iterates only $\lfloor T/K\rfloor$ times to choose $T$ samples. The rest of the procedure in the batch version follows that of Algorithm \ref{alg}. When $K=1$, the batch version coincides with Algorithm \ref{alg}.

\subsection{Simulation}
In our simulation study, we consider a context vector $x=(x_1,x_2,x_3,x_4,x_5)^\top \in\mathbb{R}^5$. The component $x_1$ is i.i.d. drawn from the uniform distribution Unif$(1, 2)$, and the remaining components $x_2, x_3, x_4, x_5$ are i.i.d. chosen from Unif$(-1/2, 1/2)$. The feature mapping function is defined as $\phi(x,a)=(x_1a^2, x_2a, x_3a, x_4a, x_5a)^\top\in\mathbb{R}^5$. The true reward parameter vector $\theta_*=(\theta_1,\theta_2,\theta_3,\theta_4,\theta_5)^\top$ has components $\theta_1=-1/2$ and $\theta_2=\theta_3=\theta_4=\theta_5=1/2$. The reward function is $r_{\theta_*}(x, a)=\theta_*^\top\phi(x,a)$. The optimal action, derived from this setup, is $$a^*(x)=\mathop{\arg\max}_ar_{\theta_*}(x, a)=-\frac{\theta_2x_2+\theta_3x_3+\theta_4x_4+\theta_5x_5}{2\theta_1x_1}.$$ The two actions are set as $a^{(0)}(x)=a^*(x)$ and $a^{(1)}(x)=\|x\|_2/3$.  The simulation involves $g=5$ types of contexts and $m=20$ teachers, with each teacher's rationality parameter $\beta_j^{(k)}$, for $j\in\{1,\cdots, m\}$ and $k\in\{1,\cdots,g\}$. We aim to select $T$ samples from $n=10000$ candidates for reward learning and decision making.
 \begin{table}[t]
\small
    \centering
    \caption{Strategies for conversation selection and teacher selection across various methods.}
    \begin{tabular}{c|c|c}
    \hline
          Methods & Conversation Selection & Teacher Selection  \\ \hline
        \textbf{DualAL} & Algorithm \ref{alg} & Algorithm \ref{alg}  \\ 
        \textbf{Conversation Selection Only} & $D$-optimal design & Random  \\ 
       \textbf{Teacher Selection Only} & Random & $D$-optimal design  \\ 
        \textbf{APO}   & APO \citep{das2024provably}& Random  \\ 
        \textbf{Linear Approximation} & \cite{mukherjee2024optimal, scheid2024optimal} & Random \\
        \textbf{Random} & Random & Random \\ \hline
    \end{tabular}
     \label{tab3}
\end{table}
Comparison is made among the following methods for reward and policy learning: 
\begin{itemize}[leftmargin=*]
    \item \textbf{DualAL} which implements dual active reward learning using $D$-optimal design according to Algorithm \ref{alg} and computes a pessimistic policy according to Algorithm \ref{alg2};
    \item \textbf{Conversation Selection Only} which selects teachers randomly at each time;
    \item \textbf{Teacher Selection Only} which selects contexts randomly at each time; 
    \item \textbf{APO} which implements the active preference optimization approach developed by \cite{das2024provably}, focusing solely on active conversation selection;
    \item \textbf{Linear Approximation} which implements active reward learning using optimal design by linear approximation proposed by \cite{mukherjee2024optimal, scheid2024optimal},  focusing on active conversation selection;
    \item \textbf{Random} which selects both teachers and contexts randomly at each time. 
\end{itemize}

 Notice that different methods employ different strategies for conversation selection and teacher selection, as summarized in Table \ref{tab3}. It is crucial to highlight that the $D$-optimal designs employed by \textbf{Conversation Selection Only} and \textbf{Teacher Selection Only} are adapted from our proposed Algorithm \ref{alg}. In particular, \textbf{Conversation Selection Only}
uses the Fisher information matrix under the BTL model for
conversation selection while selecting teachers randomly.
In contrast, the methods of \cite{mukherjee2024optimal, scheid2024optimal}, represented by \textbf{Linear Approximation},
use a linearized information matrix for conversation selection. The aim of comparing these policies is to gain deeper insights into the effectiveness of the different components of the overall policy. 

\subsubsection{Comparison of Different Methods} \label{6.1.1}
We first assess the reward estimation of the policies based on the generalized variance (GV) and the mean squared error (MSE, defined as $\mathbb{E}\|\hat{\theta
}-\theta_*\|_2$) of their reward estimators. The rationality of each teacher $\beta$ is independently drawn from a uniform distribution Unif$(0, 2)$.  A smaller GV indicates a smaller variation of the estimator, whereas a smaller MSE reflects closer proximity to the true reward values. The results, highlighted in Table \ref{tab2} with the best outcomes in bold, reveal that \textbf{DualAL} performs superiorly, showing the lowest GV and MSE. We further analyze the sub-optimality gap defined in \eqref{subopt} across different policies. Figure \ref{fig1} shows the sub-optimality gaps of different policies across varying sample sizes $T$. \textbf{DualAL} consistently outperforms the others.

\begin{table}[!ht]
    \centering
        \caption{Performance of the reward parameter estimation with $T=1000$}
    \begin{tabular}{c|cc|cc|cc}
    \hline
          Policies & \multicolumn{2}{c|}{$K=1$} & \multicolumn{2}{c|}{$K=50$}  & \multicolumn{2}{c}{$K=100$} \\ \hline
        ~ & GV & MSE & GV & MSE & GV & MSE \\  \hline 
        \textbf{DualAL} & \textbf{5.6} & \textbf{1.251} & \textbf{7.6} & \textbf{1.292} & \textbf{9.4} & \textbf{1.348} \\ 
        \textbf{Conversation Selection Only} & 41.0 & 1.481 & 20.7 & 1.344 & 16.4 & 1.459 \\ 
        \textbf{Teacher Selection Only} & 2570.8 & 2.162 & 1284.0 & 2.110 & 2092.2 & 2.208 \\ 
        \textbf{APO} & 282.0 & 1.993 & 262.4 & 1.752 & 253.1 & 1.857 \\
        \textbf{Linear Approximation} & 462.7 & 1.837 & 1858.6 & 2.109 & 1919.0 & 2.192 \\ 
        \textbf{Random} & 190034.9 & 3.296 & 104617.1 & 2.963 & 294846.8 & 3.319 \\ \hline
         \multicolumn{2}{c}{GV is expressed in units of $10^{-11}$.}
    \end{tabular}
     \label{tab2}
\end{table}

 To provide deeper insights, we conduct a detailed comparison of these methods to better understand the impact of each component on overall performance. 
\begin{itemize}[leftmargin=*]
    \item When compared to the \textbf{Random} method, \textbf{Conversation Selection Only} shows a lower sub-optimality gap as well as smaller GV and MSE, highlighting the benefits of strategic conversation selection. Similarly, the \textbf{Teacher Selection Only} method outperforms the \textbf{Random} method, validating the importance of selecting teachers. 
    \item The \textbf{Conversation Selection Only} method demonstrates a smaller sub-optimality gap and lower GV and MSE compared to the \textbf{APO} method, confirming that our active reward learning approach utilizing $D$-optimal design is more effective.
    \item \textbf{DualAL} outperforms both \textbf{Conversation Selection Only} and \textbf{Teacher Selection Only} methods, indicating the advantage of simultaneous selection of conversations and teachers.
\end{itemize}
  
\begin{figure}[h!]
\setlength{\tabcolsep}{-1pt} 
    \centering
    \begin{tabular}{c}  
        \includegraphics[scale =0.6]{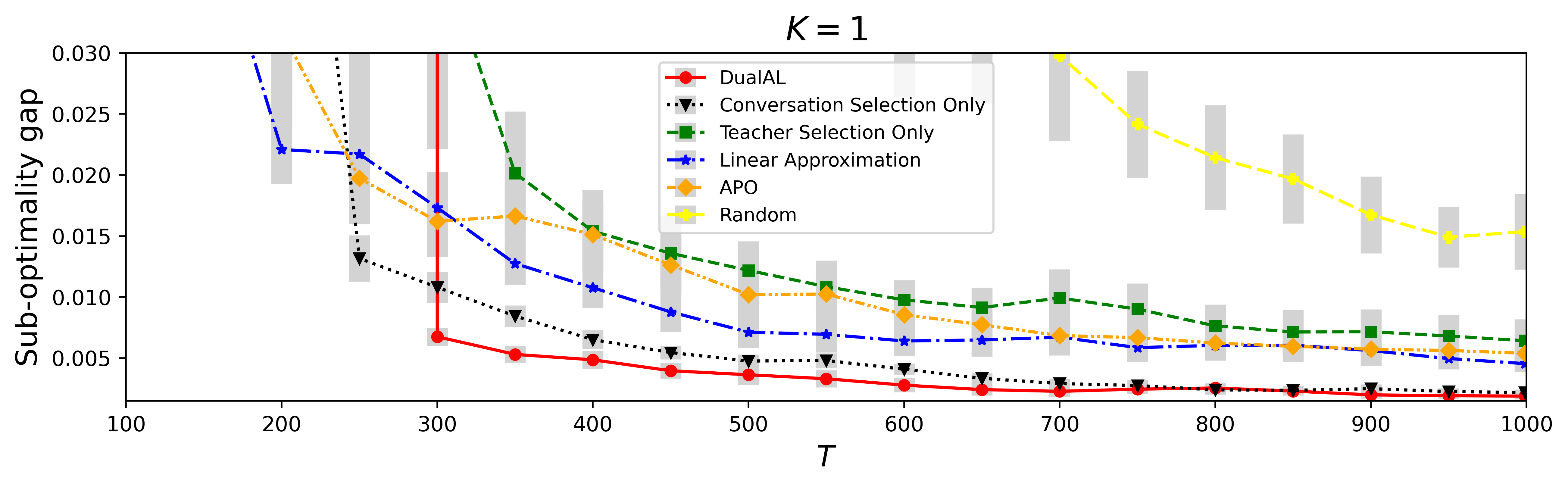}\\
        \includegraphics[scale = 0.6]{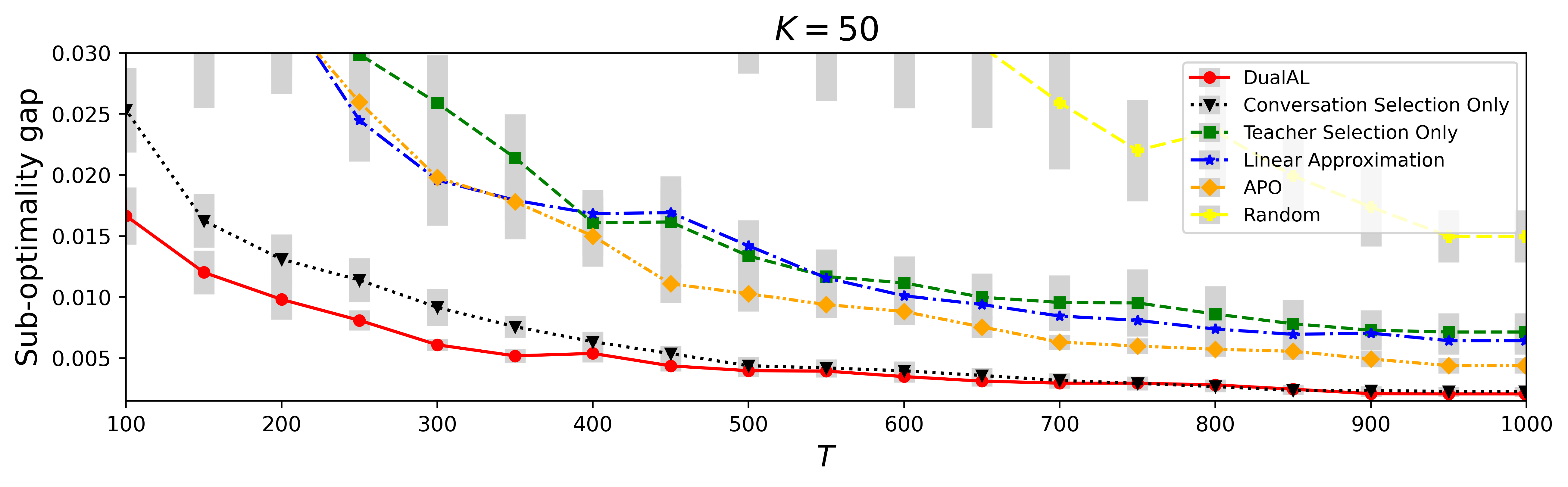}\\
        \includegraphics[scale = 0.6]{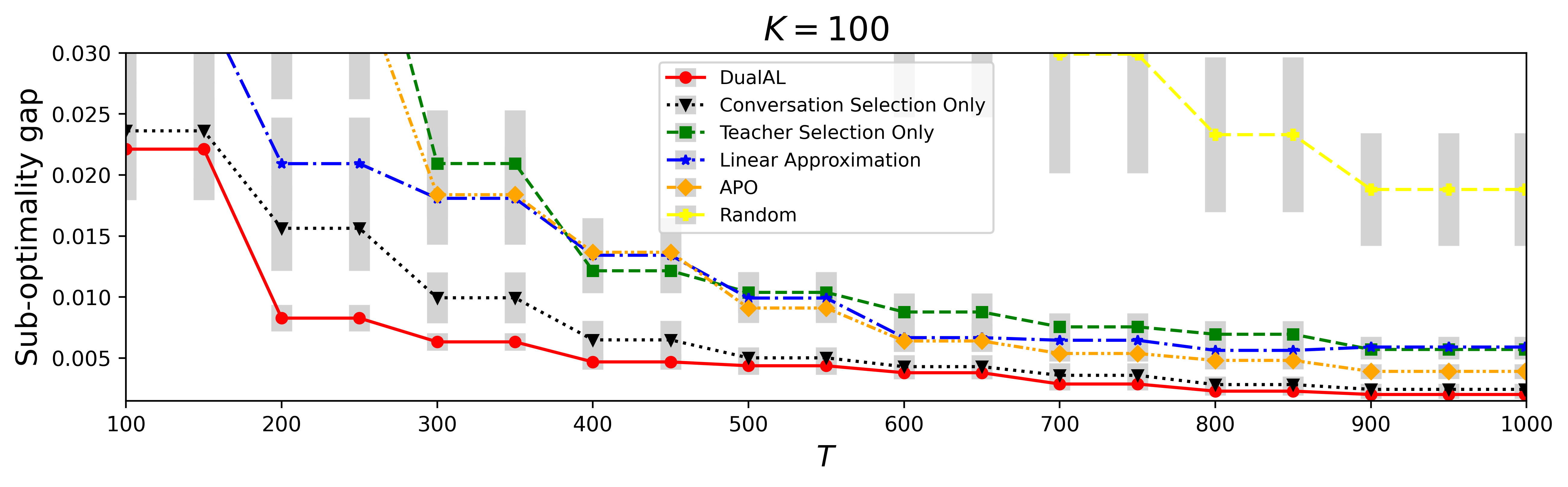}
    \end{tabular}
     \caption{Sub-optimality gaps for all policies. The three subplots show the sub-optimality gap when the batch size $K$ is 1, 50 and 100, respectively.}
         \label{fig1}
\end{figure}

This analysis confirms the superior performance of our proposed policy across different batch sizes $K$. Furthermore, we examine the computational efficiency of the batch version of our approach. The computation times for one repetition of \textbf{DualAL} are 1000.94 seconds, 27.69 seconds, and 17.55 seconds for batch sizes $K$ of 1, 50, and 100, respectively, showcasing significant reductions in computation time with increased batch sizes.

To investigate the practical manifestation of the theoretical convergence rate, we plot the sub-optimality gap of \textbf{DualAL} as a function of the sample size $T$ when $K=1$. The results are shown in Figure~\ref{fig:figure2}. We observe that the empirical sub-optimality gap decreases steadily with the sample size and begins to closely follow the theoretical $O(1/\sqrt{T})$ convergence trend once $T$ exceeds approximately $700$ samples. 
\begin{figure}[t]
\centering
\includegraphics[width=0.6\textwidth]{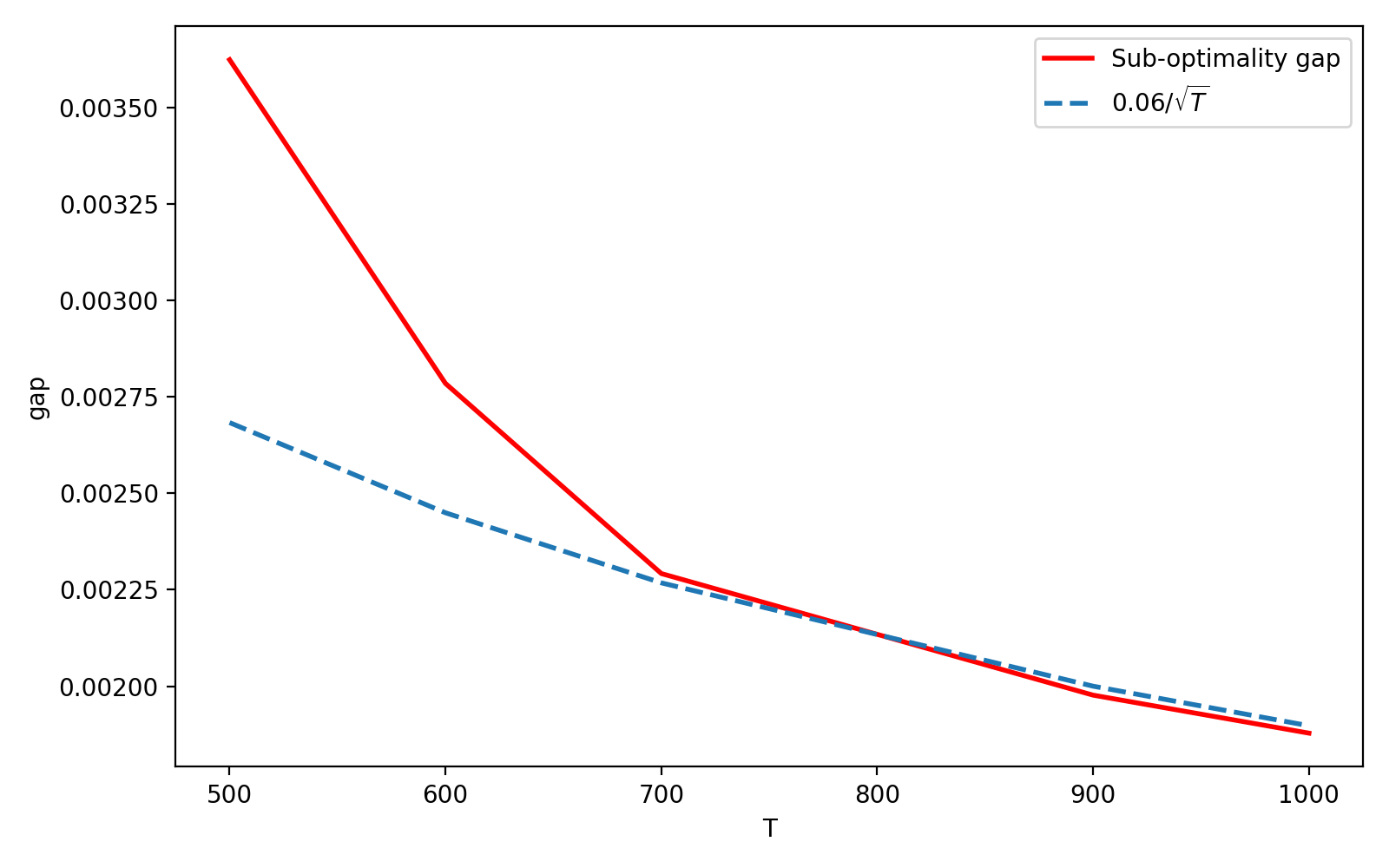}
\caption{Sub-optimality gap of \textbf{DualAL} as a function of sample size $T$. The dashed line represents the $O(1/\sqrt{T})$ reference curve.} 
\label{fig:figure2}
\end{figure}

\subsubsection{Role of Teachers}\label{6.1.2}
We now examine the influence of teacher rationality on the sub-optimality gaps under varying ranges of rationality. The rationality parameter $\beta$ is i.i.d. chosen from three different uniform distributions: Unif$(0, 3)$, Unif$(0, 2)$, and Unif$(0, 1)$. Figure \ref{fig2} illustrates that a broader range of rationality generally results in a smaller sub-optimality gap across different batch sizes when employing \textbf{DualAL}. This phenomenon suggests that a wider range of rationality choices allows for more selective and effective teacher querying, thus reducing the sub-optimality gap. The intuition behind this is that a broader range of rationality leads to a larger  $\det M(\xi_*, \theta_*)$, resulting in a better estimation of the reward parameter.
\begin{figure}[t]
    \centering
    \begin{tabular}{ccc}  
        \includegraphics[scale = 0.3420]{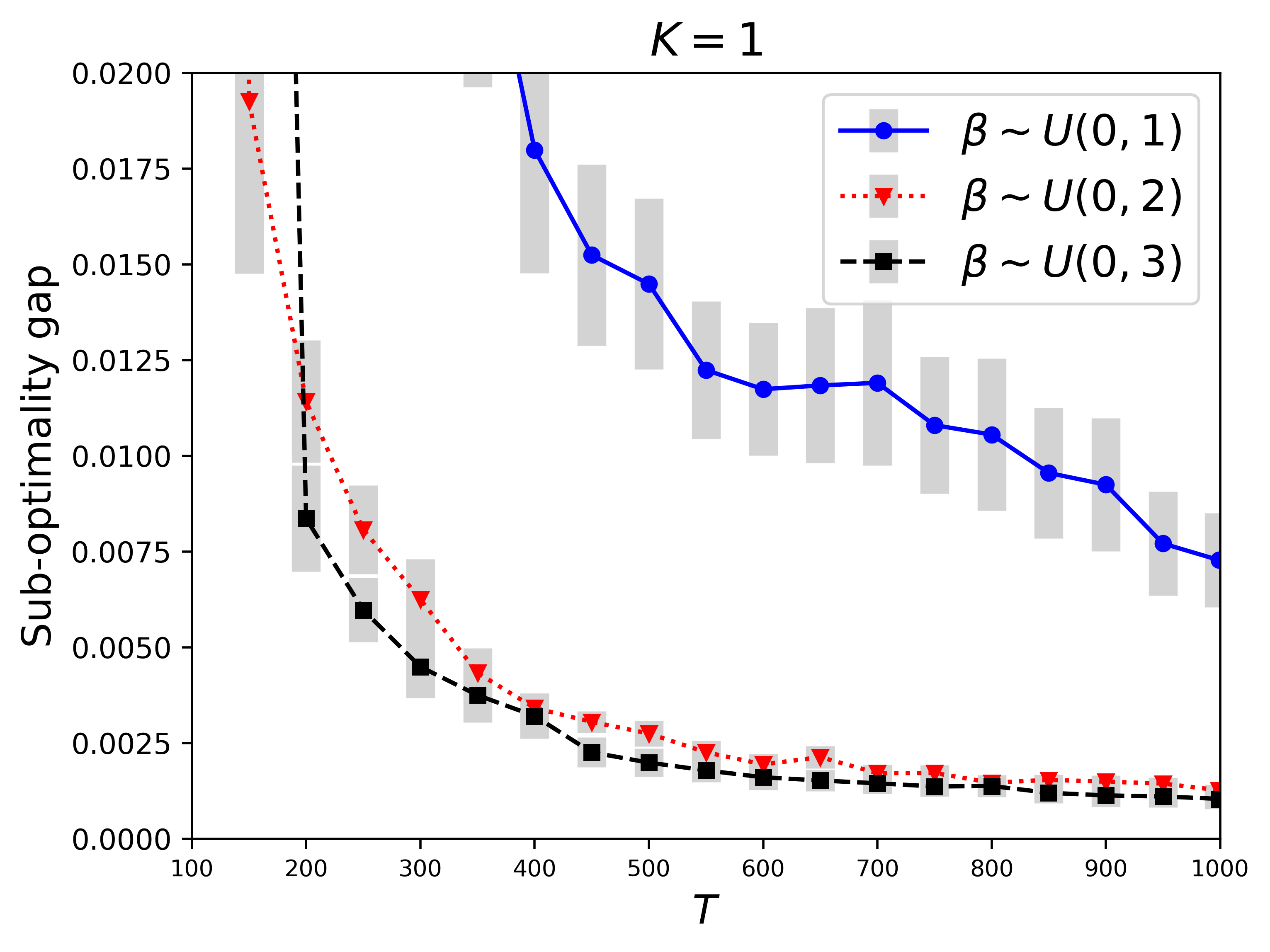}&
        \includegraphics[scale = 0.3420]{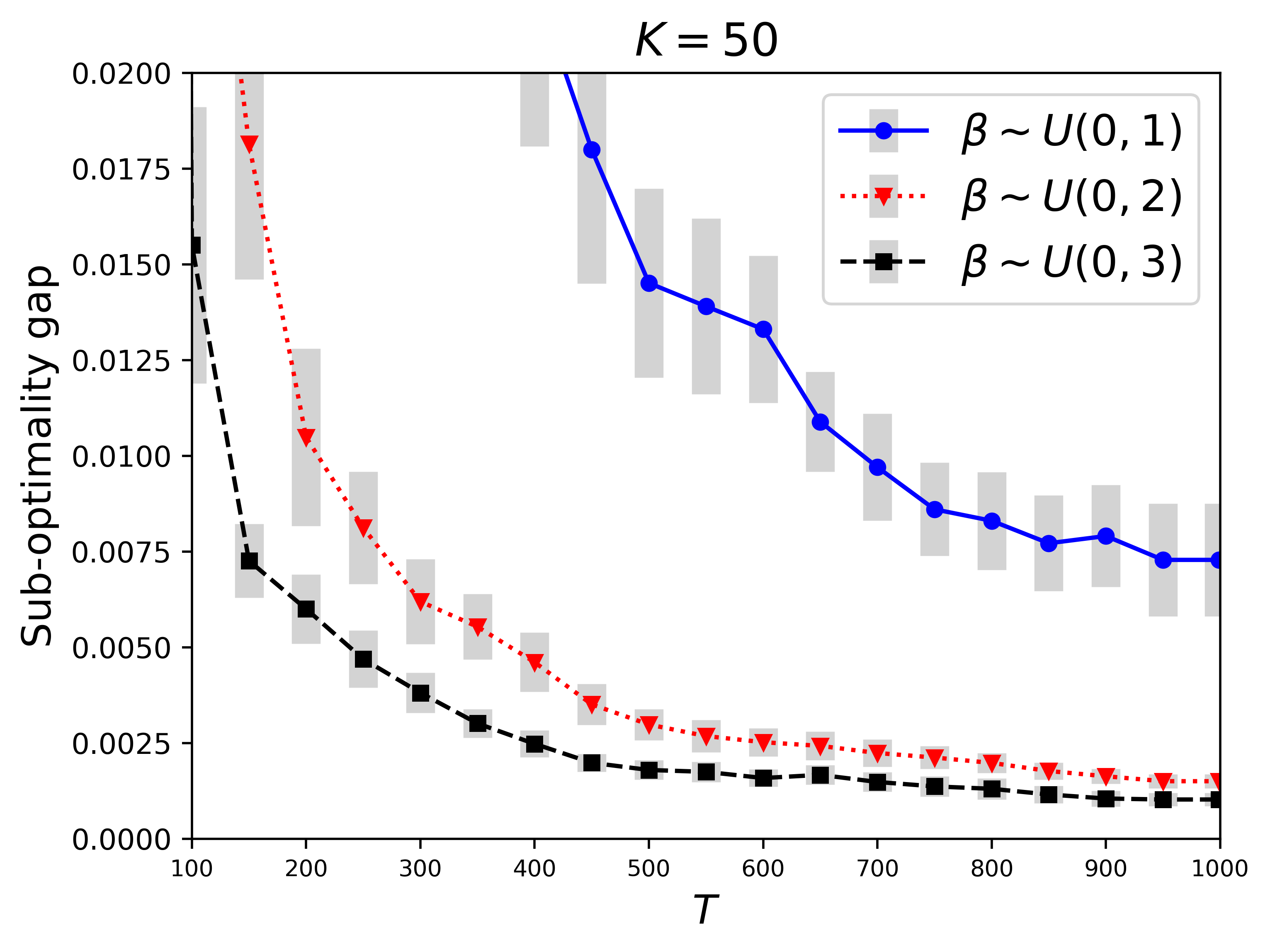}&
         \includegraphics[scale = 0.3420]{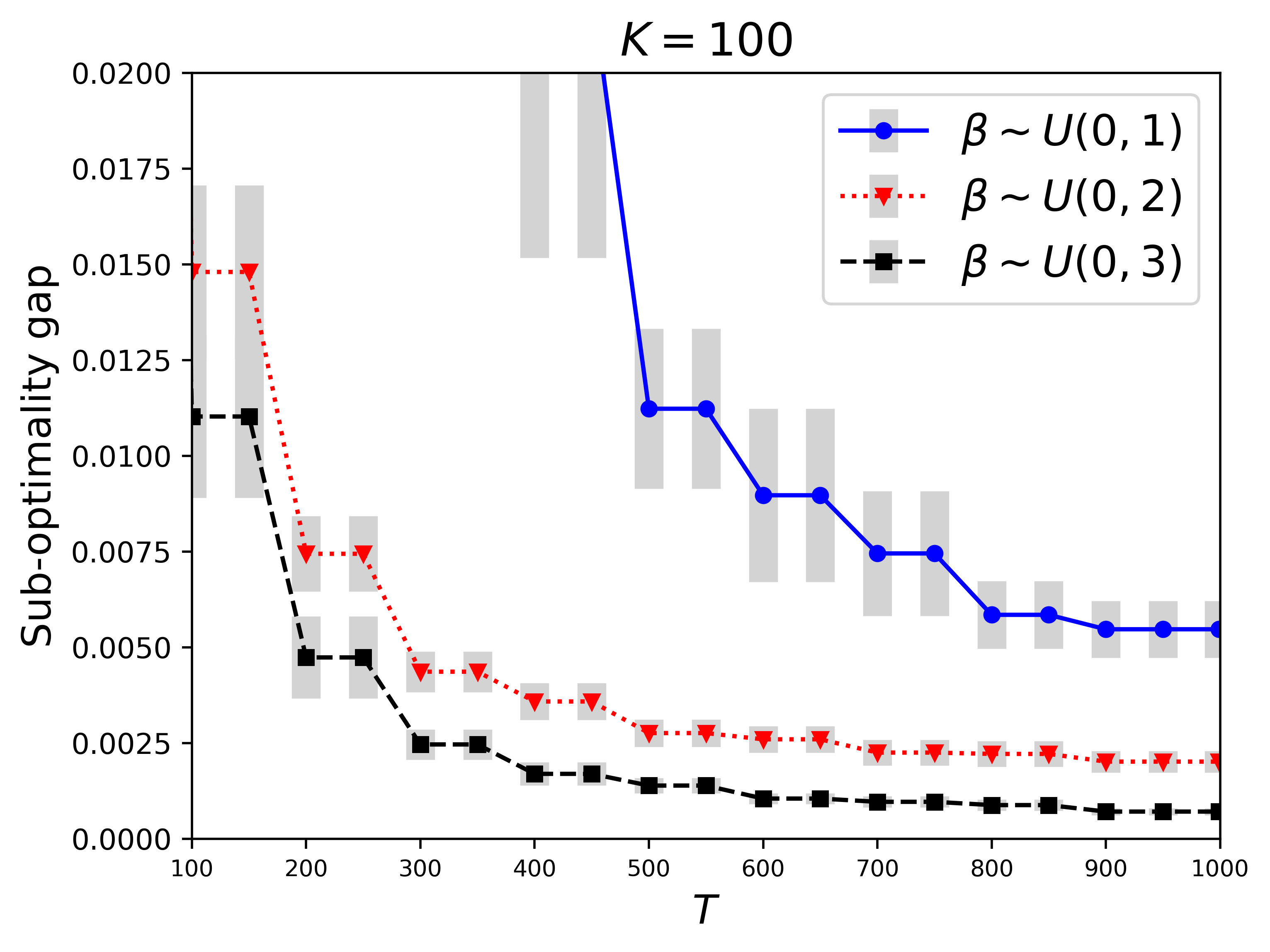}
    \end{tabular}
     \caption{Sub-optimality gap for \textbf{DualAL} at different ranges of teacher rationality.  }
         \label{fig2}
\end{figure}
\subsubsection{Effect of Dimension}
We evaluate the impact of dimensionality on the sub-optimality gap for \textbf{DualAL} across dimensions $d = 3, 5, 10$ and different batch sizes. The teacher rationality $\beta$ is sampled from Unif$(0, 1)$. The results depicted in Figure \ref{fig6} indicate that the sub-optimality gap increases with the dimension, consistent with the implications of Theorem \ref{thm3}.
\begin{figure}[t]
    \centering
    \begin{tabular}{ccc}  
    \includegraphics[scale = 0.3420]{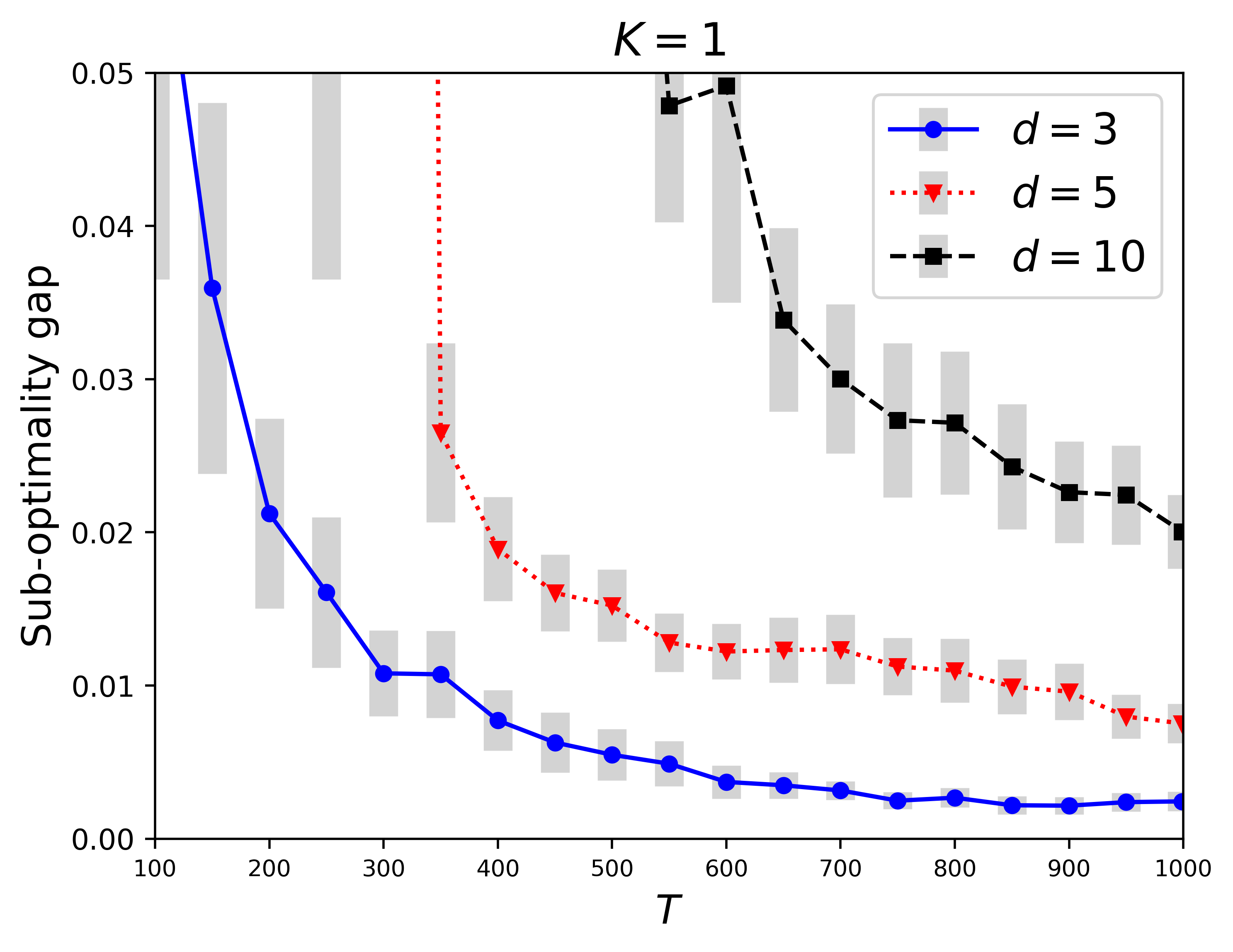}&
         \includegraphics[scale = 0.3420]{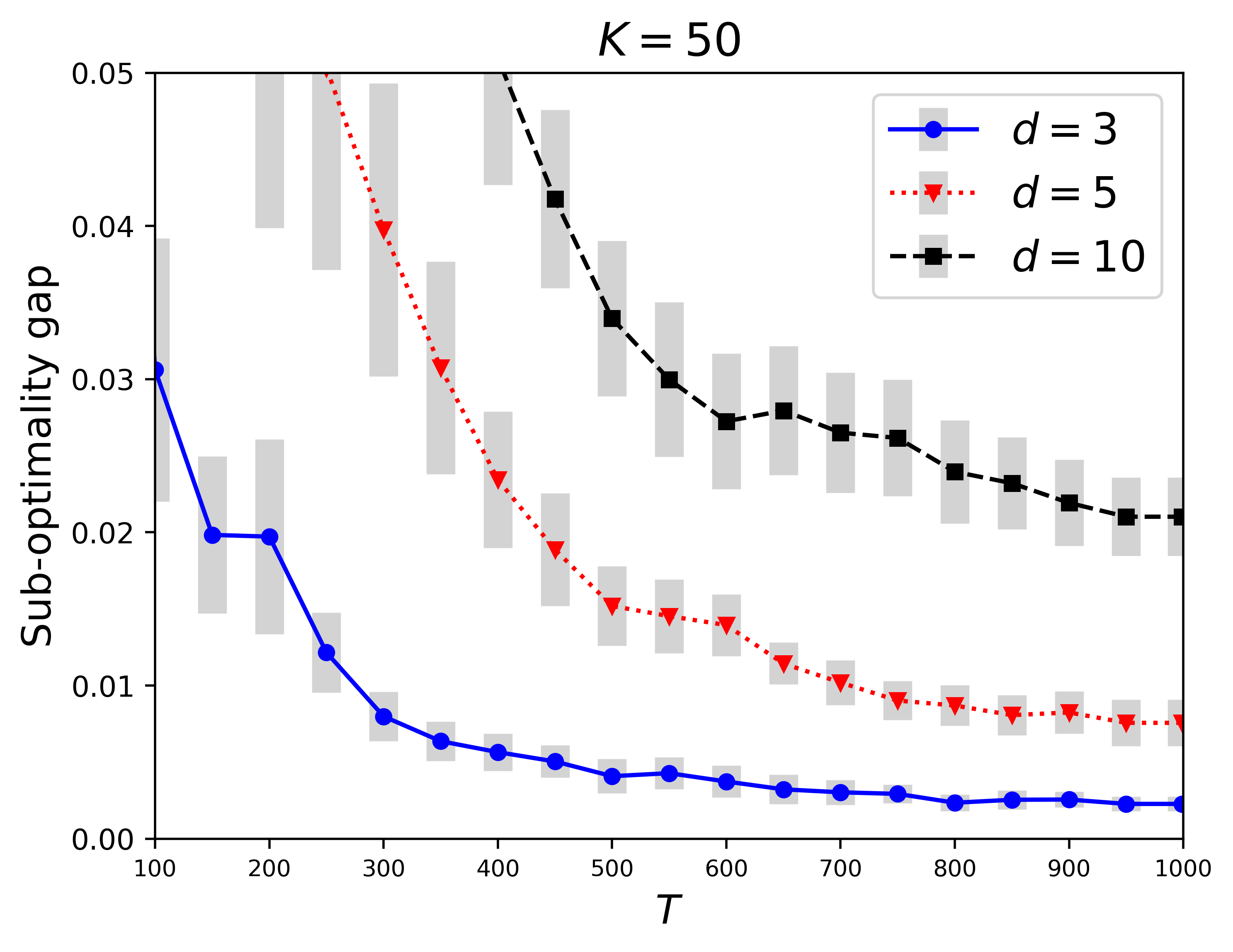}&
        \includegraphics[scale = 0.3420]{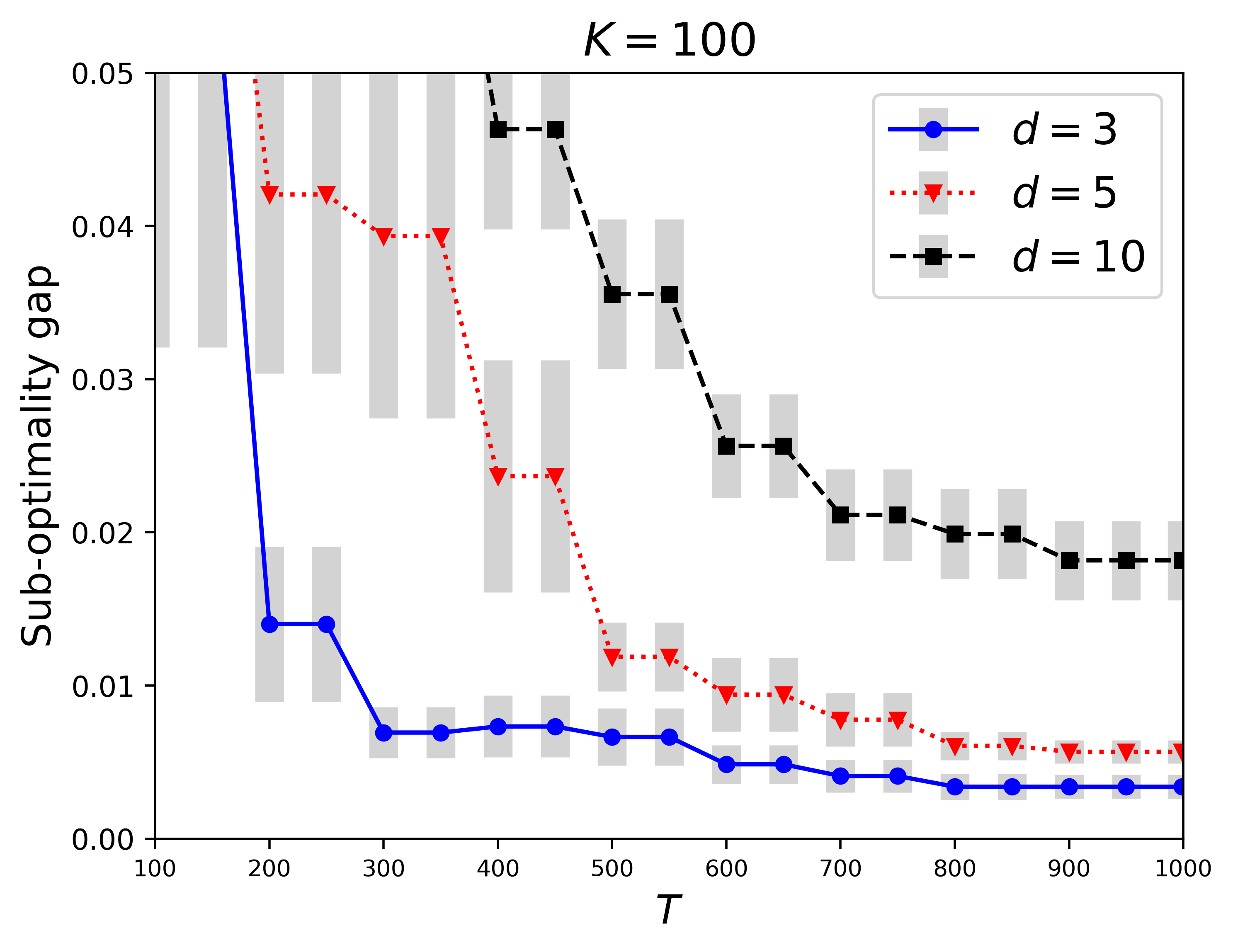}
    \end{tabular}
     \caption{Sub-optimality gap for \textbf{DualAL} at different dimensions.  }
         \label{fig6}
\end{figure}

\subsubsection{Nonlinear Reward Model}\label{5.1.4}
To explore the behavior of our approach in a more general setting, we conduct additional simulation studies where the feature map depends on the parameter, i.e., $\phi_\theta(x,a)$ includes tunable components. Specifically, we consider a nonlinear feature mapping where part of the representation is parameter-dependent, thereby mimicking the setting where all parameters are updated.

In this simulation, the context vector $x=(x_1,x_2,x_3,x_4,x_5)^\top \in \mathbb{R}^5$ is generated with $x_1 \sim \mathrm{Unif}(1,2)$ and $x_2,\ldots,x_5 \sim \mathrm{Unif}(-1/2,1/2)$. Denote $\theta=(\theta_1,\theta_2,\theta_3,\theta_4,\theta_5)^\top$. The feature map is defined as 
$
\phi_\theta(x,a) = \big(x_1 a^2 + \theta_2 x_1 x_2 a,\; x_2 a,\; x_3 a,\; x_4 a,\; x_5 a \big)^\top,
$
which introduces parameter dependence into the representation. The reward function is given by $r_{\theta_*}(x,a) = \theta_*^\top \phi_{\theta_*}(x,a)$. The true reward parameter is $\theta_* = (-1/2, 1/2, 1/2, 1/2, 1/2)^\top$.

The results, reported in Table~\ref{tab_non} and Figure~\ref{fig_non}, show that \textbf{DualAL} continues to outperform the competing methods in terms of both generalized variance (GV) and mean squared error (MSE), as well as sub-optimality gap across different sample sizes. These findings suggest that our approach remains effective even when the feature representation is not fixed.
 \begin{table}[!ht]
    \centering
        \caption{Performance of the reward parameter estimation when the reward model is nonlinear}

    \begin{tabular}{c|cc|cc|cc}
    \hline
          Policies & \multicolumn{2}{c|}{$K=1$} & \multicolumn{2}{c|}{$K=50$}  & \multicolumn{2}{c}{$K=100$} \\ \hline
        ~ & GV & MSE & GV & MSE & GV & MSE \\  \hline
        \textbf{DualAL} &0.30 & 0.605 & 1.12 & 0.647 & 1.01 & 0.684 \\ 
        \textbf{Conversation Selection Only} & 0.76 & 0.660 & 3.43 & 0.765 & 3.06 & 0.792 \\ 
        \textbf{Teacher Selection Only} & 261.43 & 1.136 & 57.42 & 0.914 & 140.22 & 1.076 \\ 
        \textbf{APO} & 112.05 & 1.058 & 92.24 & 1.006 & 353.28 & 1.180 \\
        \textbf{Linear Approximation} & 16.10 & 0.857 & 276.67 & 1.102 & 85.54 & 1.049 \\ 
        \textbf{Random} & 16301.43 & 1.566 & 10860.79 & 1.539 & 601.30 & 1.128 \\ \hline
        \multicolumn{2}{c}{GV is expressed in units of $10^{-13}$.}
    \end{tabular}
    \label{tab_non}
\end{table}

    \begin{figure}[h!]
\setlength{\tabcolsep}{-1pt} 
    \centering
    \begin{tabular}{c}  
        \includegraphics[scale =0.6]{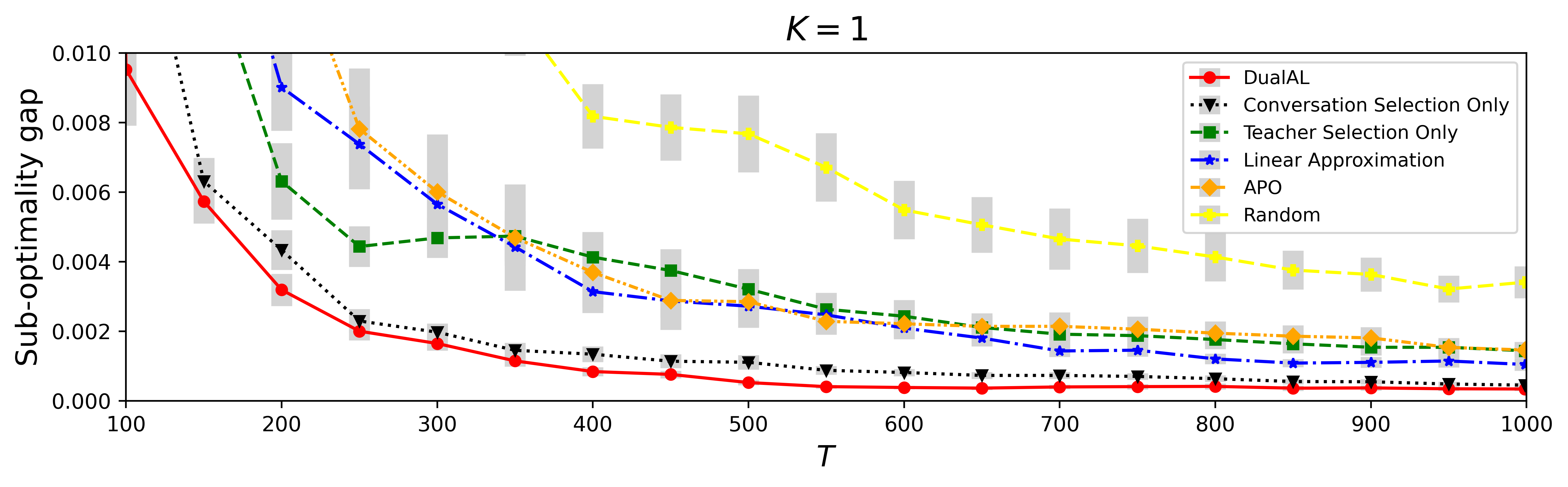}\\
        \includegraphics[scale = 0.6]{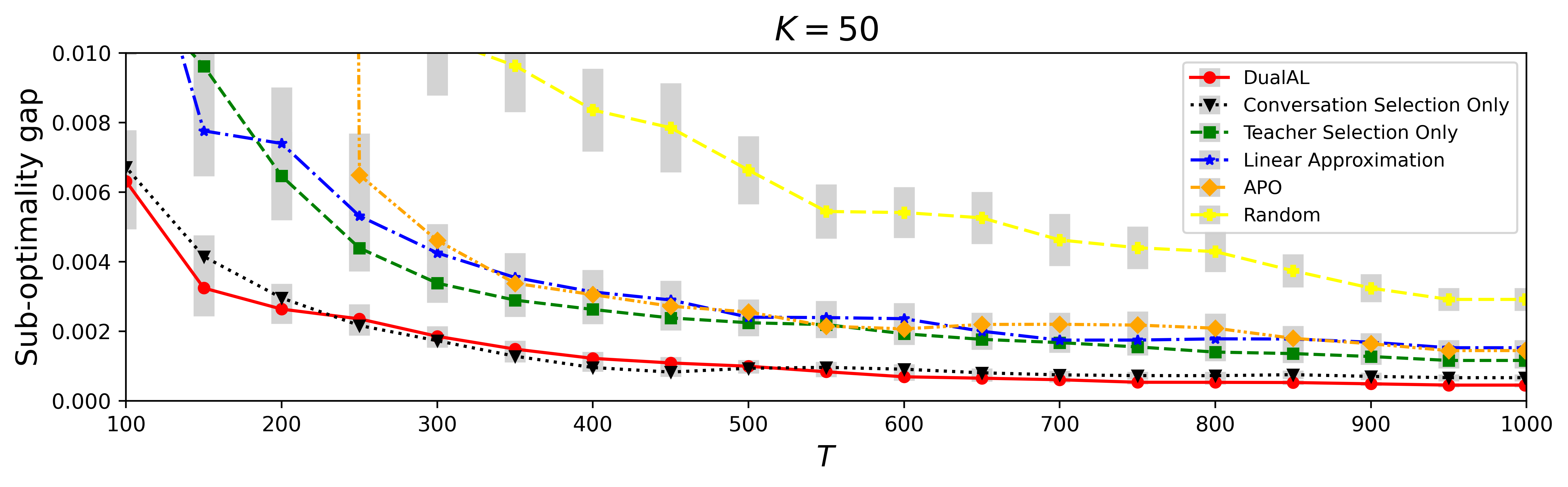}\\
        \includegraphics[scale = 0.6]{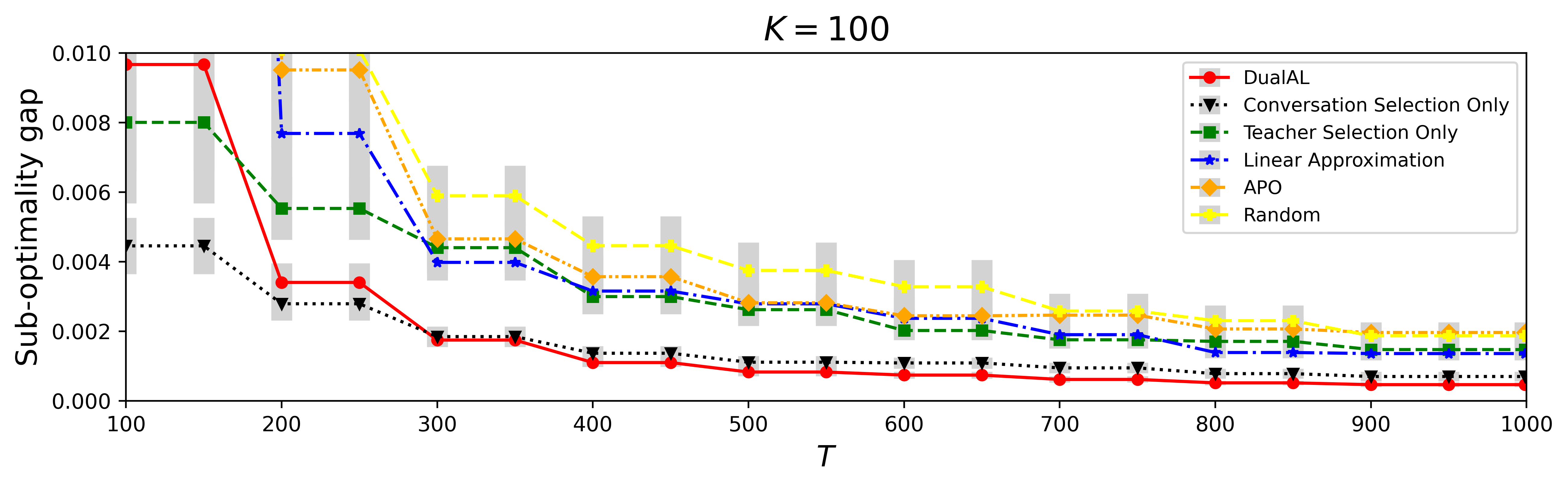}
    \end{tabular}
     \caption{Sub-optimality gaps for all policies when the reward model is nonlinear. The three subplots show the sub-optimality gap when the batch size $K$ is 1, 50 and 100, respectively.}
         \label{fig_non}
\end{figure}

\subsubsection{Robustness to Rationality Parameter Misspecification}\label{sec:robust_beta}
We further note that the proposed active selection strategy depends on the rationality parameters $\beta_j^{(k)}$. Consequently, estimation errors in $\beta_j^{(k)}$ may affect the conversation and teacher selection decisions. To evaluate the robustness of our method to such estimation errors, we conduct an additional simulation study.
The true rationality parameters $\beta_j^{(k)}$ are independently generated from Unif$(0,2)$. During the active selection process, we intentionally perturb the rationality parameters and use
$\beta_j^{(k)} + \mathrm{Unif}(0,\delta),$
where $\delta \in \{1,2,3\}$, in the Fisher information matrix for conversation and teacher selection. Therefore, the active learning policy operates under misspecified rationality parameters.
The results are reported in Figure~\ref{fig_robu}. We observe that the proposed policy remains stable under different levels of perturbation in $\beta_j^{(k)}$, demonstrating that the method is robust to moderate estimation errors in the rationality parameters.

\begin{figure}[h!]
\setlength{\tabcolsep}{-1pt}
\centering
\begin{tabular}{c}
    \includegraphics[scale=0.6]{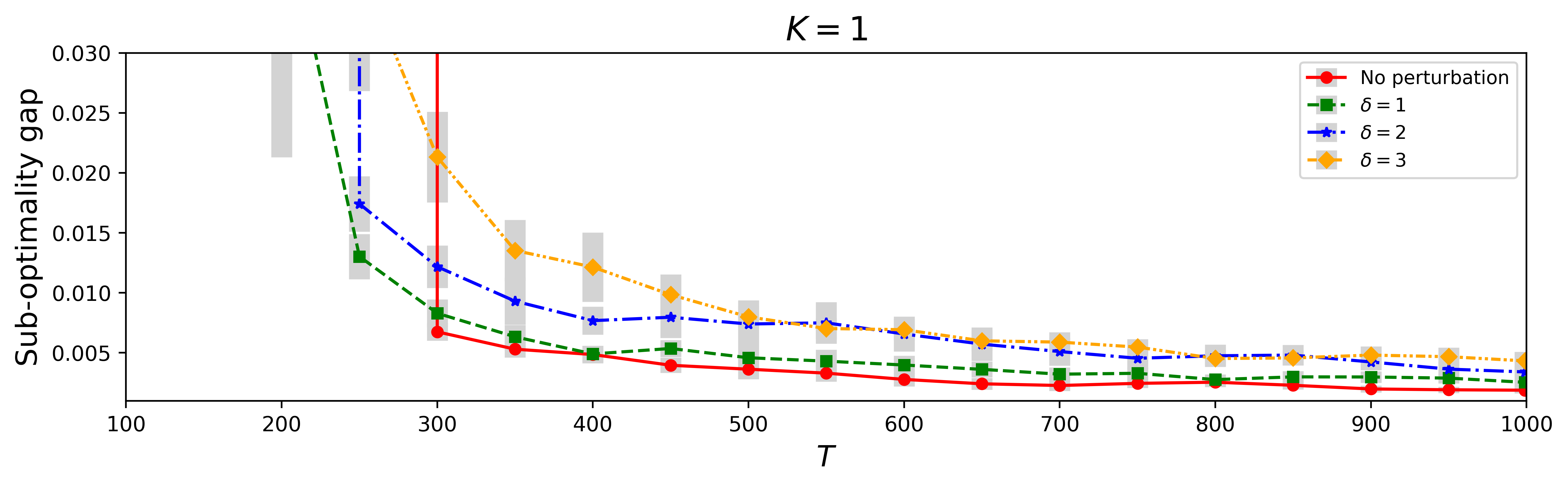}\\
    \includegraphics[scale=0.6]{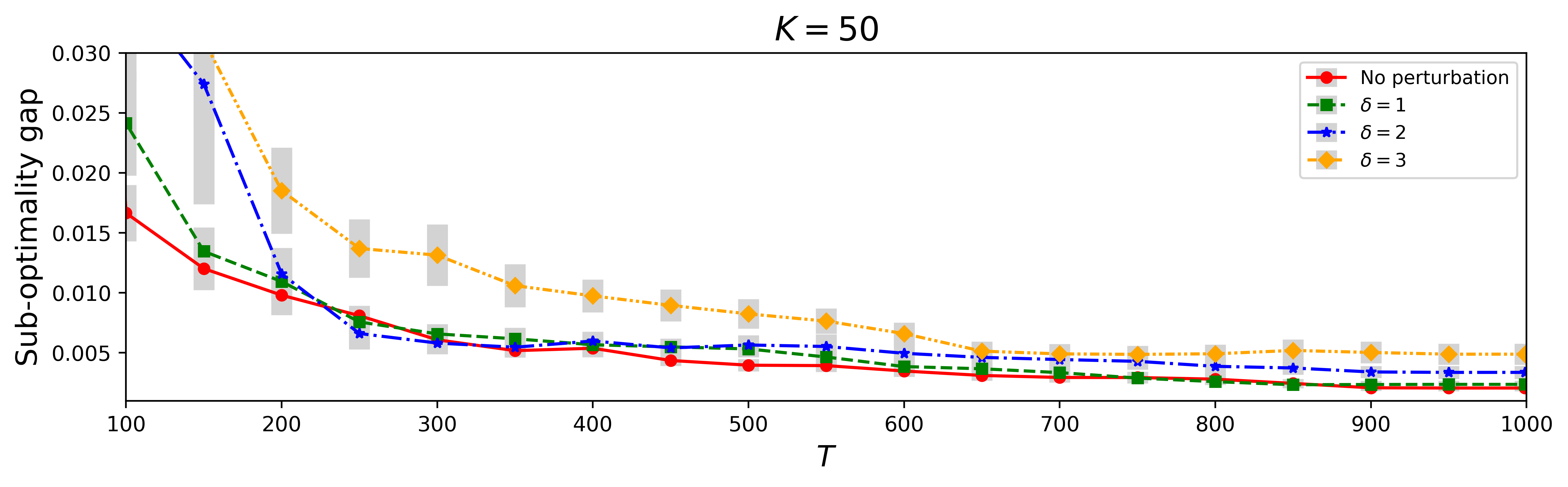}\\
    \includegraphics[scale=0.6]{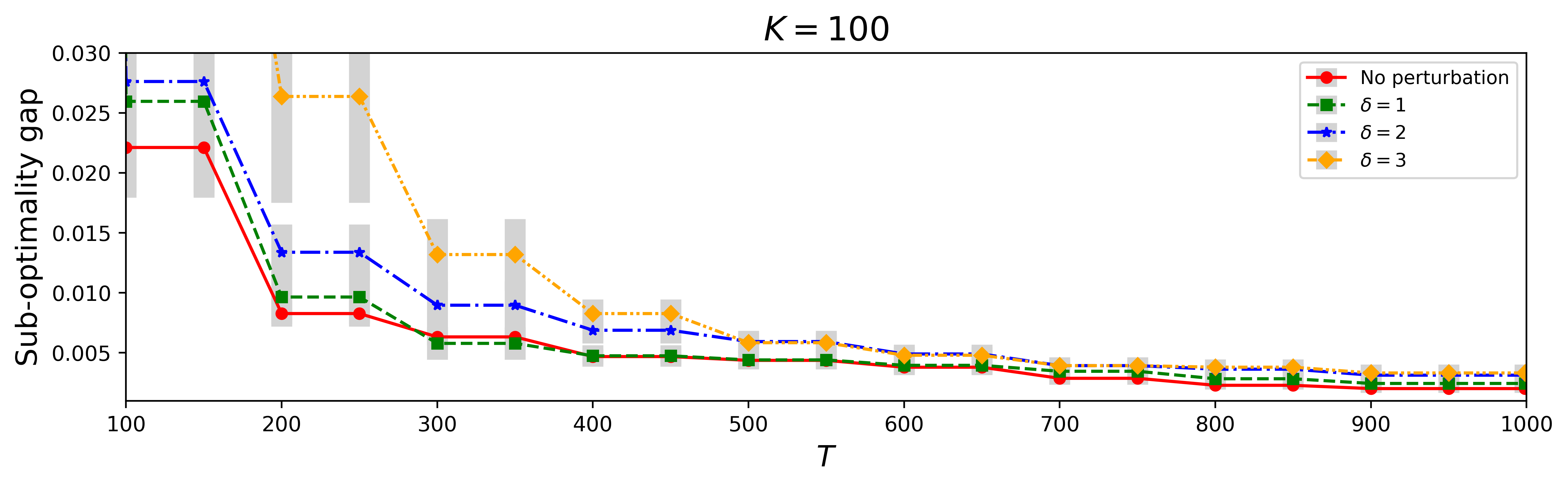}
\end{tabular}
\caption{Sub-optimality gaps under misspecified rationality parameters. The three subplots correspond to batch sizes $K=1$, $50$, and $100$, respectively.}
\label{fig_robu}
\end{figure}

\subsection{Applications to LLMs}\label{sec6.2}
In this experiment, we implement our policy within large language models, utilizing the public dataset  \texttt{prism-alignment} \citep{kirk2024prism}, which provides annotator-level information and is therefore well-suited to
our setting.  

We preprocess the dataset by retaining the first turn of each dialogue and constructing pairwise comparisons from the provided human evaluation scores, where the higher-scored response is treated as preferred. This results in 42,306 comparison samples, each consisting of a user profile $x$, a prompt $s$, two candidate responses $a^{(0)}$ and $a^{(1)}$, and a binary preference $y$. We divide them into a training subset and a test subset with a 4:1 ratio. Following \citet{kirk2024prism}, we cluster annotators into $m=14$ groups and questions into $g=22$ topic categories.

Before applying our policy, we need to learn the rationality parameters $\beta_j^{(k)}$.
A natural approach is to jointly estimate the reward parameter $\theta$ and the rationality parameters $\boldsymbol{\beta}=\{\beta_j^{(k)}: j=1,\ldots,m,\; k=1,\ldots,g\}$ via maximum likelihood \citep{jin2020rank}. Specifically, when both $\boldsymbol{\beta}$ and $\theta$ are unknown, we consider the log-likelihood
\begin{equation*}
L_T(\boldsymbol{\beta},\theta)=\frac{1}{T}\sum_{t=1}^T\left\{y_t\log \mu(\beta_{t}\theta^\top z_t)+(1-y_t)\log \big[1-\mu(\beta_{t} \theta^\top z_t)\big]\right\},
\end{equation*}
where each $\beta_t$ corresponds to one element in $\boldsymbol{\beta}$. To ensure identifiability, we impose the normalization constraint $\beta_j^{(1)}=1$ for all $j\in\{1,\ldots,m\}$. Maximizing $L_T(\boldsymbol{\beta},\theta)$ yields estimators for both $\theta_*$ and $\boldsymbol{\beta}$. After obtaining these estimates, we treat them as fixed and apply our proposed active learning algorithm to select conversations and teachers for estimating $\theta$.
In practice, teacher rationality estimates may contain noise
due to limited preference data. Since the estimation errors in $\beta$ affect both the quadratic
rationality factor and the curvature of the preference likelihood, small estimation errors may change the selection
when candidate scores are close. In this case, additional calibration
preference data may be needed to improve the rationality estimates, while
Section~\ref{sec:robust_beta} investigates robustness empirically under
the positive perturbations considered there.

The pretrained model employed is \texttt{Gemma-7b-it}\footnote{\href{https://huggingface.co/google/gemma-7b-it}{https://huggingface.co/google/gemma-7b-it}} \citep{team2024gemma}. The feature map $\phi$ in \eqref{reward} is derived by removing the last layer of the pretrained language model, yielding a $d$-dimensional vector, where the dimension $d=3072$ is determined by the \texttt{Gemma-7b-it} model. More details of the pretrained model and the real data are in Appendix \ref{appb}.

Our goal in this experiment is to learn the reward function specified in \eqref{reward} within a sample budget of $T=5000$. We briefly describe the process of the experiments. The question and two corresponding answers $(x, a^{(0)}, a^{(1)})$ are first input into the \texttt{Gemma-7b-it} model. After being processed through the last layer of the \texttt{Gemma-7b-it} model, the triple $(x, a^{(0)}, a^{(1)})$ is transformed into a 3072-dimensional vector $\phi(x, a^{(0)}, a^{(1)})$.  The preference $y$ denoting the preference between $a^{(0)}$ and $a^{(1)}$ follows the Bernoulli distribution as described in \eqref{e1}. Our objective is to estimate the parameter $\theta_*$ in \eqref{e1} using MLE. 

Since the true reward parameter in \eqref{e1} is unknown, we evaluate the effectiveness of different policies using the reward accuracy, which is widely used in assessing reward estimation in large language models \citep{yao2023deepspeed,das2024provably}. Using the estimator $\hat{\theta}$, we can obtain the estimated reward $\hat{\theta}^\top \phi(x, a)$. The reward accuracy is defined as the percentage of instances where the estimated reward of the chosen response exceeds that of the rejected one. A higher reward accuracy signifies a better policy.

We evaluate reward accuracy on test samples across various $T$ values and batch sizes ($K=50, 100$) using different methods. We run the experiments over 20 trials, each time splitting the training and test data using different random seeds, and then compute the mean reward accuracy. Figure \ref{fig11} showcases the results, highlighting the superior performance of \textbf{DualAL} compared to benchmark policies.

\begin{figure}[t]
    \centering
    \begin{tabular}{ccc}  
        \includegraphics[scale = 0.5]{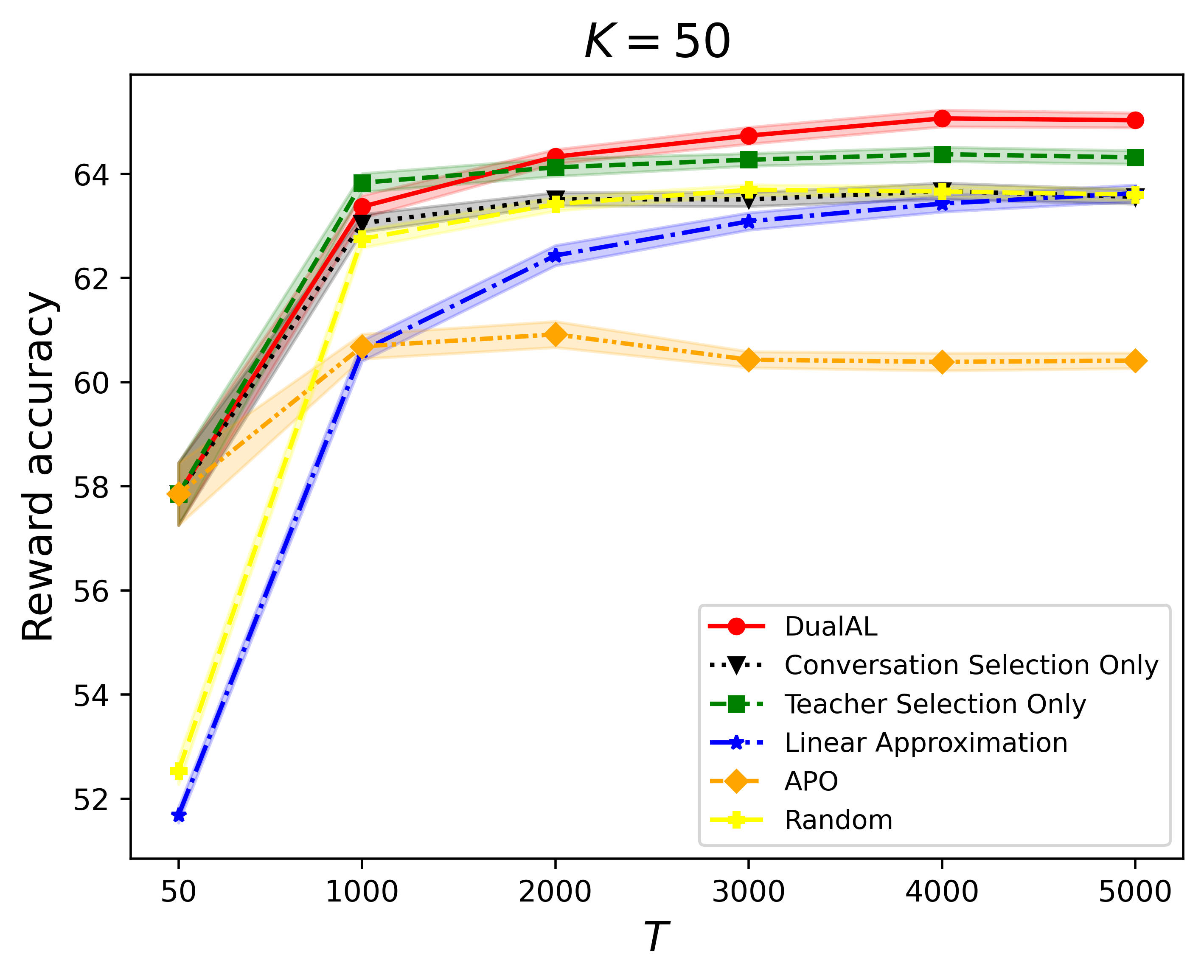}&
         \includegraphics[scale = 0.5]{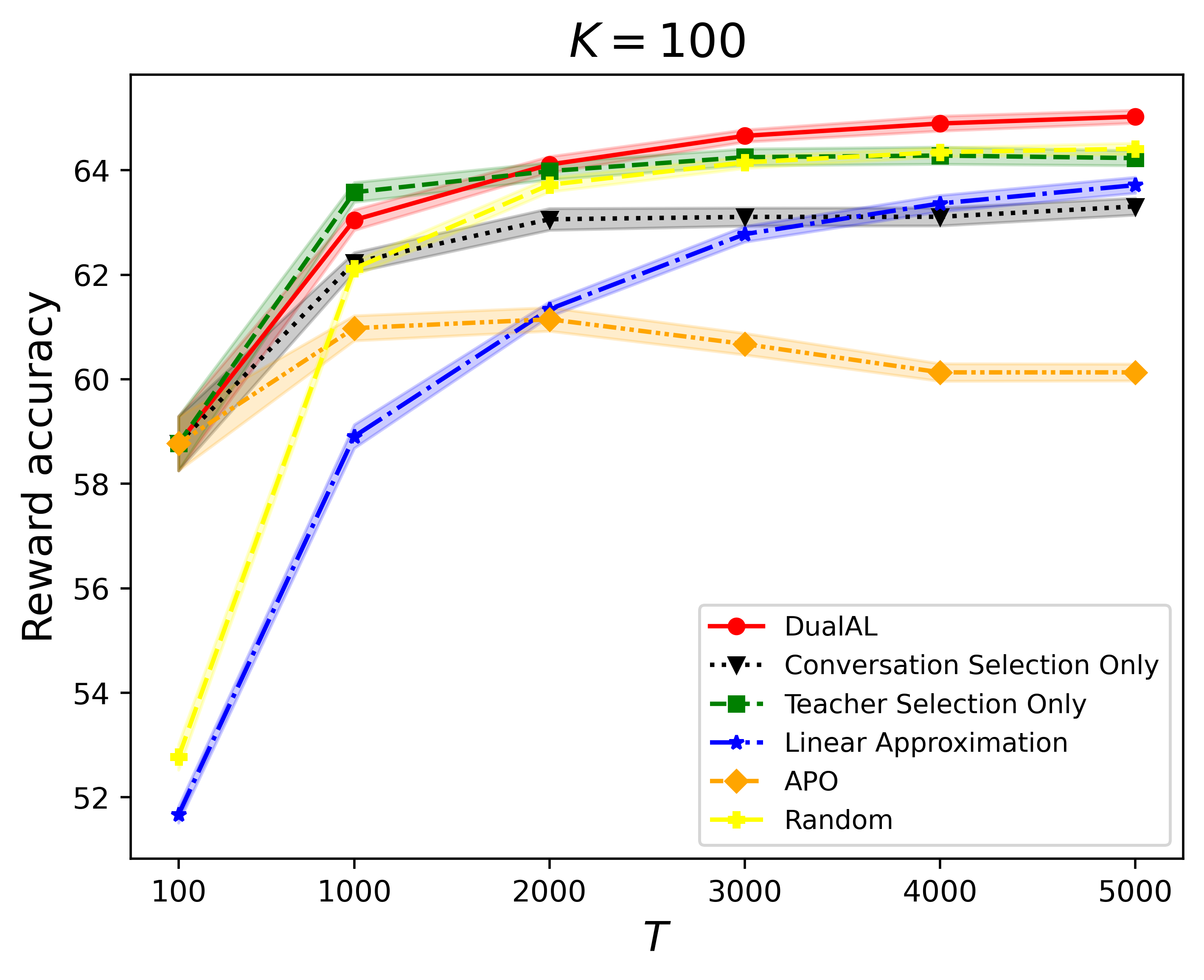}
    \end{tabular}
     \caption{Reward accuracy with different methods using the \texttt{prism-alignment} dataset.}
     \label{fig11}
\end{figure}

 We evaluate the reward models on additional benchmark datasets using RewardBench \citep{lambertrewardbench}, including AlpacaEval \citep{alpaca_eval} (easy, hard, and length) and HumanEvalPack \citep{muennighoff2024octopack} (coding tasks across multiple programming languages). The results are reported in Tables~\ref{rmben1} and \ref{rmben2}. Across most settings, our proposed method (\textbf{DualAL}) achieves the highest reward accuracy, demonstrating its robustness across diverse evaluation tasks.
 
 \begin{table}[!ht]
    \centering
    \setlength{\tabcolsep}{4pt} 
    \caption{Reward accuracy on RewardBench benchmarks with batch size $K=50$}
    \begin{tabular}{c|ccc|cccccc}
    \hline
    \multirow{2}{*}{Policies} & \multicolumn{3}{c|}{AlpacaEval} & \multicolumn{6}{c}{HumanEvalPack} \\ 
        & easy & hard & length & cpp & go & java & js & python & rust \\ \hline
       DualAL & \textbf{0.906} & \textbf{0.906} & \textbf{0.807} & \textbf{0.590} & \textbf{0.601} & 0.562 & \textbf{0.628} & \textbf{0.580} & \textbf{0.564} \\ 
         Conversation Selection Only & 0.819 & 0.834 & 0.755 & 0.583 & 0.571 & 0.508 & 0.557 & 0.568 & 0.540 \\ 
       Teacher Selection Only & 0.831 & 0.868 & 0.777 & 0.549 & 0.570 & 0.507 & 0.584 & 0.528 & 0.547 \\ 
        APO & 0.759 & 0.779 & 0.661 & 0.474 & 0.509 & 0.527 & 0.520 & 0.475 & 0.532 \\
       
        Linear Approximation & 0.839 & 0.901 & 0.695 & 0.511 & 0.570 & \textbf{0.563} & 0.585 & 0.511 & 0.525 \\ 
        Random & 0.790 & 0.850 & 0.679 & 0.558 & 0.575 & 0.531 & 0.574 & 0.534 & 0.548 \\ 
        \hline
    \end{tabular}
    \label{rmben1}
\end{table}

\begin{table}[!ht]
    \centering
    \setlength{\tabcolsep}{4pt} 
        \caption{Reward accuracy on RewardBench benchmarks with batch size $K=100$}
    \begin{tabular}{c|ccc|cccccc}
    \hline
    \multirow{2}{*}{Policies} & \multicolumn{3}{c|}{AlpacaEval} & \multicolumn{6}{c}{HumanEvalPack} \\ 
         & easy & hard & length & cpp & go & java & js & python & rust \\ \hline
        DualAL & \textbf{0.897} & \textbf{0.907} & \textbf{0.804} & \textbf{0.602} & 0.586 & 0.578 & \textbf{0.619} & \textbf{0.573} & \textbf{0.565} \\ 
          Conversation Selection Only & 0.782 & 0.819 & 0.689 & 0.553 & 0.571 & 0.506 & 0.540 & 0.543 & 0.540 \\ 
          Teacher Selection Only & 0.831 & 0.869 & 0.767 & 0.580 & 0.563 & 0.511 & 0.595 & 0.544 & 0.555 \\ 
           APO & 0.742 & 0.769 & 0.635 & 0.513 & 0.520 & 0.543 & 0.530 & 0.476 & 0.544 \\
        Linear Approximation & 0.855 & 0.902 & 0.704 & 0.530 & 0.588 & \textbf{0.582} & 0.611 & 0.556 & 0.531 \\ 
       
        Random & 0.807 & 0.872 & 0.697 & 0.563 & \textbf{0.597} & 0.550 & 0.591 & 0.562 & 0.556 \\   
        \hline
    \end{tabular}
    \label{rmben2}
\end{table}

To assess how improvements in reward accuracy translate into better downstream decision-making, we evaluate the quality of the induced policy using a Best-of-$N$ selection procedure \citep{stiennon2020learning,nakano2021webgpt,gao2023scaling,jinnai2024regularized,gui2024bonbon,liu2025pairwise,chow2025inferenceaware}. The Best-of-$N$ policy generates multiple candidate responses and  selects the response with the highest predicted reward. Specifically, in the \texttt{prism-alignment} dataset, each question is associated with 2–4 candidate responses, each annotated with a human evaluation score that reflects its true quality. For each learned reward model, we use the model to score all candidate responses and select the one with the highest predicted reward. We then compute the average ground-truth score of the selected responses.
The results, reported in Tables~\ref{tabbon3} and \ref{tabbon4}, show that \textbf{DualAL} consistently selects responses with the highest average true scores. Moreover, the improvements over competing methods are statistically significant across all comparisons. These findings demonstrate that the gains in reward accuracy achieved by our method translate into meaningful improvements in downstream policy performance.
 \begin{table}[!ht]
    \centering
    \caption{Evaluation of Selected Responses by Different Policies with $K=50$ }
    \begin{tabular}{c|ccc}
    \hline
        Policies & Mean of True Score & Mean Diff & $P$\_value \\ \hline
        DualAL & 69.741 & 0 & - \\ 
        Conversation Selection Only & 68.448 & 1.293 & $3.69\times 10^{-7}$ \\
        Teacher Selection Only  & 69.331 & 0.410 & $1.44\times 10^{-2}$ \\ 
        APO & 66.189 & 3.552 & $6.59\times 10^{-13}$ \\
        Linear Approximation & 68.756 & 0.985 & $1.45\times 10^{-5}$ \\ 
        Random  & 68.707 & 1.034 & $7.84\times 10^{-5}$ \\  \hline
    \end{tabular}
    \label{tabbon3}
\end{table}

\begin{table}[!ht]
    \centering
    \caption{Evaluation of Selected Responses by Different Policies with $K=100$ }
    \begin{tabular}{c|ccc}
    \hline
         Policies & Mean of True Score & Mean Diff & $P$\_value \\ \hline
        DualAL & 69.78 & 0 & - \\ 
        Conversation Selection Only & 68.43 & 1.35 & $1.38\times 10^{-10}$ \\ 
        Teacher Selection Only & 69.113 & 0.667 & $2.00\times 10^{-6}$ \\
        APO & 66.122 & 3.658 & $3.44\times 10^{-16}$ \\
        Linear Approximation & 68.725 & 1.055 & $6.51\times 10^{-8}$ \\ 
        Random & 69.273 & 0.507 & $2.96\times 10^{-3}$ \\  \hline
    \end{tabular}
    \label{tabbon4}
\end{table}

\section{Conclusion} \label{sec8}
In this paper, we introduce a comprehensive framework for dual active learning for RLHF, incorporating simultaneous conversation and teacher selection. Our theoretical analysis validates the effectiveness of our proposed algorithm. Furthermore, experimental results consistently demonstrate that our policy outperforms existing baselines. Based on the adaptively learned reward estimator, we develop a pessimistic policy for the offline RL problem.  This framework not only improves the accuracy of the reward estimation, but also optimizes the efficiency of data usage in the training of large language models, offering significant advancements in the field of RLHF.

Several promising directions exist for extending our work. The linear reward model can be restrictive in capturing complex human preferences. In particular, human preferences in RLHF are often highly nonlinear, context-dependent, and may involve intricate interactions between prompts, responses, reasoning quality, harmlessness, style, factuality, and safety considerations. A purely linear reward model may therefore fail to fully capture phenomena such as multimodal preferences, hierarchical reasoning patterns, or highly context-specific trade-offs between competing objectives.
Extending our dual active learning framework to RLHF with general reward structures \citep{ye2024online} is a promising direction.

One possible route toward such an extension is to replace the fixed-feature linear reward model with a neural or LoRA-based reward model \citep{hu2022lora} while retaining the same information-based acquisition principle. For a general reward model $r_\vartheta(x,a)$, define the reward difference
$
    \Delta_\vartheta(x,a^{(0)},a^{(1)})
    =
    r_\vartheta(x,a^{(1)})-r_\vartheta(x,a^{(0)}),
$
and let
$
    g_{\widehat\vartheta}
    =
    \nabla_\vartheta
    \Delta_\vartheta(x,a^{(0)},a^{(1)})
    \big|_{\vartheta=\widehat\vartheta}.
$
Under a local linearization around the current estimate $\widehat\vartheta$, the empirical Fisher contribution of a candidate comparison can be approximated by
$
    \dot\mu(\beta\Delta_{\widehat\vartheta})
    \beta^2
    g_{\widehat\vartheta}g_{\widehat\vartheta}^{\top}.
$
This expression reduces to the information matrix used in our fixed-feature model when $r_\theta(x,a)=\theta^\top\phi(x,a)$, since in that case $g_{\widehat\theta}=\phi(x,a^{(1)})-\phi(x,a^{(0)})$. Thus, the proposed $D$-optimal acquisition rule can be viewed as a special case of a local Fisher-based design for more general reward models. For large neural or LoRA-based reward models, diagonal, low-rank, or block-diagonal empirical Fisher approximations can be used to make this extension computationally feasible. Developing formal theoretical guarantees for such general reward structures remains an interesting open problem.

Additionally, incorporating a KL-divergence penalty and adapting our approach to an online RLHF framework could help mitigate the reward overoptimization issue \citep{xiong2024iterative}. Finally, we can extend our approach to more general ranking problems \citep{Negahban,liu2023lagrangian,wang2024ranking,fan2023spectral}, and investigate how to address infeasible tasks \citep{zhang2024defining} and integrate causal reasoning \citep{cai2023knowledge} into large language models using the dual active learning framework. 
In this work, we assume that the teacher rationality parameters are fixed. While this assumption simplifies the analysis, it may not fully capture real-world scenarios where teacher behavior evolves over time. An important direction for future research is to develop adaptive strategies that jointly update the teacher model and the reward function in an online manner.

Beyond the LLM setting, our framework offers a general approach for allocating limited human feedback resources in systems involving human–computer interaction. By jointly selecting which queries to present and which teachers to assign, the proposed method provides a principled way to improve information efficiency under budget constraints. This perspective is broadly relevant to a range of applications, including recommender systems and crowdsourcing platforms. These connections highlight the potential of our approach to support more efficient and effective decision-making in a variety of real-world systems.

{\baselineskip=20pt
\bibliographystyle{asa}
\bibliography{reference}
}

\newpage
\appendix 
\baselineskip=20pt
\setcounter{page}{1}
\setcounter{equation}{0}
\setcounter{section}{0}
\renewcommand{\thesection}{S.\arabic{section}}
\renewcommand{\thelemma}{S\arabic{lemma}}
\renewcommand{\theequation}{S\arabic{equation}}

\begin{center}
{\Large\bf Supplementary Materials} \\
\medskip
{\Large\bf ``Dual Active Learning for Reinforcement Learning from Human Feedback’’}  \\
\bigskip
\end{center}
\bigskip
In this supplement, we include experimental details for Figure \ref{fig_pess} in Section \ref{appa}, extend our framework to MDPs in Section \ref{sec7}, provide additional results in Section \ref{adde},  briefly describe the datasets and the pretrained model in Section \ref{appb}, and provide detailed proofs of the theoretical results, including Lemma \ref{lemma0}, Theorem \ref{cor1}, Theorem \ref{thm1}, Theorem \ref{thm3} and Theorem \ref{thm4} in Section \ref{appd}. Supporting lemmas are included in Section \ref{appsu}.

\section{Experimental Details for Figure \ref{fig_pess}}\label{appa}

We consider 4 actions with $\phi(x, a_1)=(0.2,0, 0.1)^\top, \phi(x, a_2)=(0.1,-0.9, 0.1)^\top, \phi(x, a_3)=(0.2,0.1, -0.1)^\top$ and $\phi(x, a_4)=(0,0.1, 0)^\top$. The true reward parameter is $\theta_*=(-1,0.1, 1)^\top$. The optimal action is $a_4=\mathop{\arg\max}_{a\in \{a_1, a_2,a_3,a_4\}}\theta_*^\top\phi(x, a)$. For the experiment, $T$ actions are randomly selected from $a_1, a_2$ and $a_3$ with probabilities $0.45, 0.45, 0.1$ to estimate $\theta_*$ as $\hat{\theta}$, intentionally excluding the optimal action $a_4$ from selection. Based on $\hat{\theta}$, we apply both the greedy policy and the pessimistic policy. The simulation results are derived from 100 independent runs.

\section{Extension to Markov Decision Processes}\label{sec7}
 Now, we extend our framework to MDPs. We consider a finite-horizon MDP characterized by the tuple $(\mathcal{S}, \mathcal{A}, N, \{P_i\}_{i=1}^N, \{r_i\}_{i=1}^N, \rho)$.  Here, $\mathcal{S}$ represents the state space, $\mathcal{A}$ is the action space, $N$ denotes the horizon length, $P_i: \mathcal{S}\times\mathcal{A}\mapsto \Delta(\mathcal{S})$ is the probability transition at step $i$, $r_i: \mathcal{S}\times\mathcal{A}\mapsto \mathbb{R}$ is the reward function, and $\rho$ is the initial state distribution. At each step $i$, after taking action $a$ in state $s$, the system transitions to a new state $s'$ with probability $P_i(s'|s,a)$, and a reward $r_i(s, a)$ is received.

We assume the availability of two trajectories starting from the same initial state for comparison. Initially, we sample the starting state $s_0$ from a fixed distribution $\rho$, followed by two trajectories $\tau^{(0)}=(s_0^{(0)},a^{(0)}_0,s^{(0)}_1,a^{(0)}_1,\cdots, s^{(0)}_N, a^{(0)}_N)$ and $\tau^{(1)}=(s_0^{(1)},a^{(1)}_0,s^{(1)}_1,a^{(1)}_1,\cdots, s^{(1)}_N, a^{(1)}_N)$, where both start from $s_0$, i.e., $s_0^{(0)}=s_0^{(1)}=s_0$. The preference of a teacher with rationality parameter $\beta$ over the two trajectories $\tau^{(0)}$ and $\tau^{(1)}$ is given by 
\begin{equation}
\mathbb{P}(Y=1|s_0,\tau^{(0)},\tau^{(1)},\beta,\theta_*)=\frac{e^{\beta\theta_*^\top\sum_{i=0}^N\phi(s_i^{(1)},a_i^{(1)})}}{e^{\beta\theta_*^\top\sum_{i=0}^N\phi(s_i^{(0)},a_i^{(0)})}+e^{\beta\theta_*^\top\sum_{i=0}^N\phi(s_i^{(1)},a_i^{(1)})}}.
\end{equation}
We have a dataset $\{(o^{(i)}, \tau_i^{(0)}, \tau_i^{(1)})\}_{i=1}^n$, where $o^{(i)}$ denotes the type of the trajectory, and define $z^{(i)}=\sum_{i=0}^N[\phi(s_i^{(1)},a_i^{(1)})-\phi(s_i^{(0)},a_i^{(0)})]$ for reward learning with a sample budget constraint. To estimate $\theta_*$, we select $T$ samples from $(z^{(1)}, \cdots, z^{(n)})$ and $T$ teachers from $\{\beta_1^{(k)},\cdots,\beta_m^{(k)}\}_{k=1}^g$ using Algorithm \ref{alg}, by modifying only the calculation of $z$. The conclusions regarding the MLE $\hat{\theta}_T$ derived from the contextual bandit setting using Algorithm \ref{alg} are applicable here.

A deterministic policy $\pi_i: \mathcal{S}\mapsto\mathcal{A}$ is a function that maps a state to an action at step $i$. We use $\pi$ to denote the collection of policies $\{\pi_i\}_{i=1}^N$. The associated value function $V^{\pi}(s)=\mathbb{E}[\sum_{i=0}^Nr_i(s_i,a_i)|s_0=s, a_i=\pi_i(s_i)]$ represents the expected cumulative reward from starting in state $s$ and adhering to $\pi_i$ at each step $i$. We define the state occupancy measure $d^\pi(s)=\sum_{i=1}^N\mathbb{P}_i(s_i=s|\pi)$ and the state-action occupancy measure $d^\pi(s,a)=\sum_{i=1}^N\mathbb{P}_i(s_i=s,a_i=a|\pi)$, where $\mathbb{P}_i(s_i=s|\pi)$ denotes the probability of visiting state $s_i=s$ (similarly for $s_i=s, a_i=a$) at step $i$ after executing policy $\pi$ and starting from $s_0\sim\rho$. 

For analyzing sub-optimality, we employ a pessimistic estimate of the rewards. When the transition distribution $P$ is known, the occupancy measure $d^\pi$ can be directly computed. If $P$ is unknown, it can be estimated by collecting state-action trajectories through interactions with the environment, as outlined in the method proposed by \cite{zhan2023query}. Given the definition of $d^\pi$, one has $\mathbb{E}_{s\sim\rho}[V^\pi(s)]=\mathbb{E}_{s, a\sim d^\pi}[r(s,a)]$. The pessimistic expected value function is formulated as
\begin{equation*}
\hat{J}_T(\pi)=\mathop{\min}_{\theta\in \mathcal{C}(\hat{\theta}_T,\delta)} \mathbb{E}_{s\sim d^{\pi}}\theta^\top \phi(s,\pi(s))=\hat{\theta}_T^\top\mathbb{E}_{s\sim d^{\pi}}\phi(s,\pi(s))-\|\mathbb{E}_{s\sim d^{\pi}}\phi(s,\pi(s))\|_{\bar{H}_T^{-1}(\hat{\theta}_T)}\gamma (T,d,\delta).
\end{equation*}
Then, the pessimistic policy is obtained as $\hat{\pi}_T=\mathop{\arg\max}_\pi \hat{J}_T(\pi)$.
\begin{theorem}\label{thm4} 
  Under Assumptions \ref{assu1} and \ref{assu2}, for any $0<\delta<1$, with probability at least $1-\delta$, when $T>T_0$ for some positive constant $T_0$, the sub-optimality of the pessimistic policy $\hat{\pi}_T$ for the offline MDPs is bounded by
   $$\textsf{SubOpt}(\hat{\pi}_T)\leq 2\sqrt{\frac{C_3}{T}\left[ d\log \left(e+\frac{C_4T}{d}\right)+\log \frac{2}{\delta}\right]}\|M^{-1/2}(\xi_*, \theta_*)\mathbb{E}_{s\sim d^{\pi^*}} \phi(s,\pi^*(s))\|_2,$$
   where $C_3$ and $C_4$ are some positive constants.
\end{theorem}

\section{Additional Experiments}\label{adde}
In this section, we implement our policy within large language models, utilizing the public datasets \texttt{Anthropic} \citep{bai2022training} and \texttt{UltraFeedback} \citep{cui2023ultrafeedback}. We collect all the prompts with single-turn dialogues from the dataset and process them into a pairwise training format, where each question is paired with two answers. Here, each question serves as context $x$, and the two answers serve as $a^{(0)}$ and $a^{(1)}$. Each answer is given a rating score\footnote{\href{https://huggingface.co/datasets/llm-blender/Unified-Feedback}{https://huggingface.co/datasets/llm-blender/Unified-Feedback}}, and the answer with the higher score is treated as the chosen one. We randomly select $40000$ samples and divide them into a training subset and a test subset with a 4:1 ratio. 

The pretrained model employed is \texttt{Gemma-7b-it}\footnote{\href{https://huggingface.co/google/gemma-7b-it}{https://huggingface.co/google/gemma-7b-it}} \citep{team2024gemma}. The feature map $\phi$ in \eqref{reward} is derived by removing the last layer of the pretrained language model, yielding a $d$-dimensional vector, where the dimension $d=3072$ is determined by the \texttt{Gemma-7b-it} model. More details of the pretrained model and the real data are in Appendix \ref{appb}.

Our goal in this experiment is to learn the reward function specified in \eqref{reward} within a sample budget of $T=5000$, i.e., selecting 5000 samples from $n=32000$ training samples. We briefly describe the process of the experiments. The question and two corresponding answers $(x, a^{(0)}, a^{(1)})$ are first input into the \texttt{Gemma-7b-it} model. After being processed through the last layer of the \texttt{Gemma-7b-it} model, the triple $(x, a^{(0)}, a^{(1)})$ is transformed into a 3072-dimensional vector $\phi(x, a^{(0)}, a^{(1)})$.  The preference $y$ denoting the preference between $a^{(0)}$ and $a^{(1)}$ follows the Bernoulli distribution as described in \eqref{e1}. Our objective is to estimate the parameter $\theta_*$ in \eqref{e1} using MLE. 

Since no information about the rationality of teachers is available in the dataset,  we engage various LLMs as synthetic teachers. The LLMs considered are \texttt{Qwen2.5-7B-Instruct} \citep{qwen2.5}, \texttt{Yi-1.5-34B-Chat}\footnote{\href{https://huggingface.co/01-ai/Yi-1.5-34B-Chat/}{https://huggingface.co/01-ai/Yi-1.5-34B-Chat}} and \texttt{glm-4-9b-chat} \citep{glm2024chatglm}. These LLMs provide their preference $y$ on two answer options for a single question. Questions are categorized into $g=5$ groups using $k$-means clustering \citep{sklearn_api}. We first derive the MLE $\hat{\theta}$ from the original data. Using $\hat{\theta}$, the information of the type of question, and the preference of each LLM, we estimate the rationality of each LLM $\beta_j^{(k)}$ under the constraints $\sum_{k=1}^{5}\beta_j^{(k)}=10$ and $\beta_j^{(k)}\in[0, 10]$ using MLE. We implement different policies to select conversations and query the synthetic teacher for the preference between the two answers, and estimate $\theta_*$. 

Since the true reward parameter in \eqref{e1} is unknown, we evaluate the effectiveness of different policies using the reward accuracy, which is widely used in assessing reward estimation in large language models \citep{yao2023deepspeed,das2024provably}. Using the estimator $\hat{\theta}$, we can obtain the estimated reward $\hat{\theta}^\top \phi(x, a)$. The reward accuracy is defined as the percentage of instances where the estimated reward of the chosen response exceeds that of the rejected one. A higher reward accuracy signifies a better policy.

We evaluate reward accuracy on test samples across various $T$ values and batch sizes ($K=50, 100$) using different methods. All experiments were conducted on a single Nvidia A100 GPU. Due to the prohibitive computational cost demonstrated in simulations, the case with $K=1$ is excluded. Figures \ref{fig9} and \ref{fig10} showcase the results, highlighting the superior performance of \textbf{DualAL} compared to benchmark policies. 
Additionally, we explore the computational efficiency of the batch version of \textbf{DualAL}, observing marked reductions in computation time with increased batch sizes: for \texttt{Anthropic}, times are 6.06 hours ($K=50$), and 2.70 hours ($K=100$); for \texttt{UltraFeedback}, times are 6.05 hours ($K=50$), and 2.68 hours ($K=100$). 
\begin{figure}[t]
    \centering
    \begin{tabular}{ccc}  
        \includegraphics[scale = 0.5]{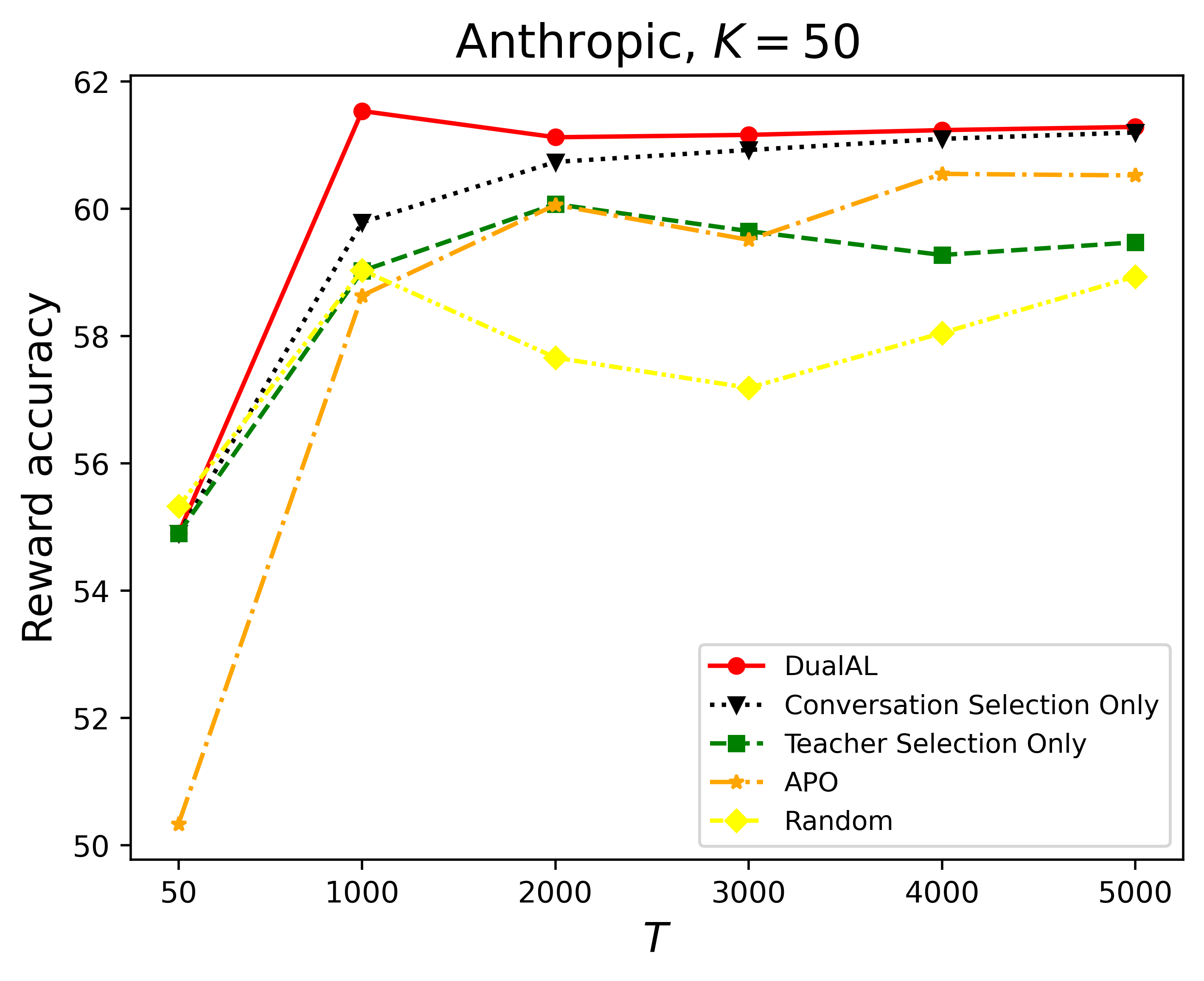}&
         \includegraphics[scale = 0.5]{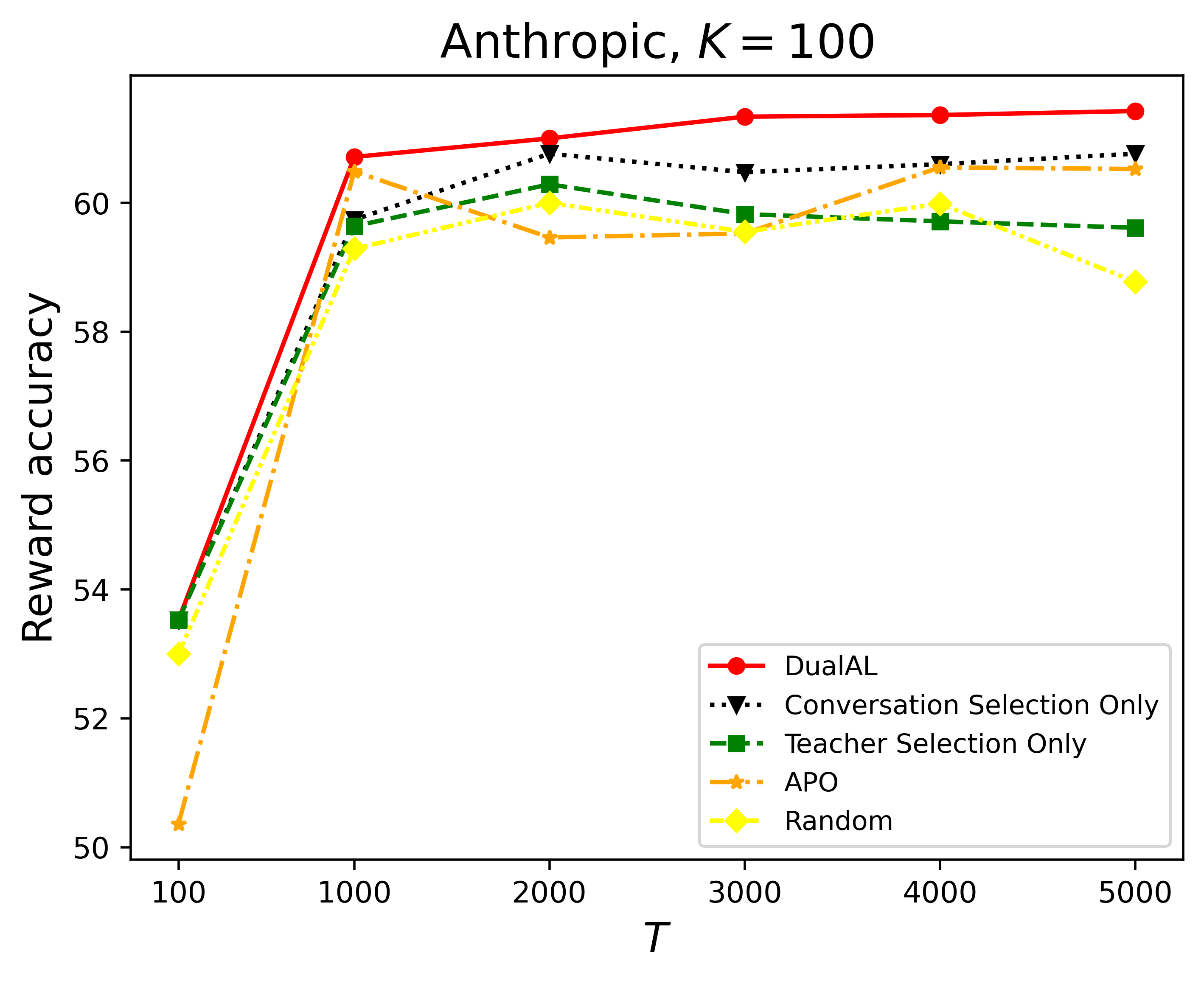}
    \end{tabular}
     \caption{Reward accuracy with different methods using the \texttt{Anthropic} dataset.}
     \label{fig9}
\end{figure}
\begin{figure}[t]
    \centering
    \begin{tabular}{ccc}  
        \includegraphics[scale = 0.5]{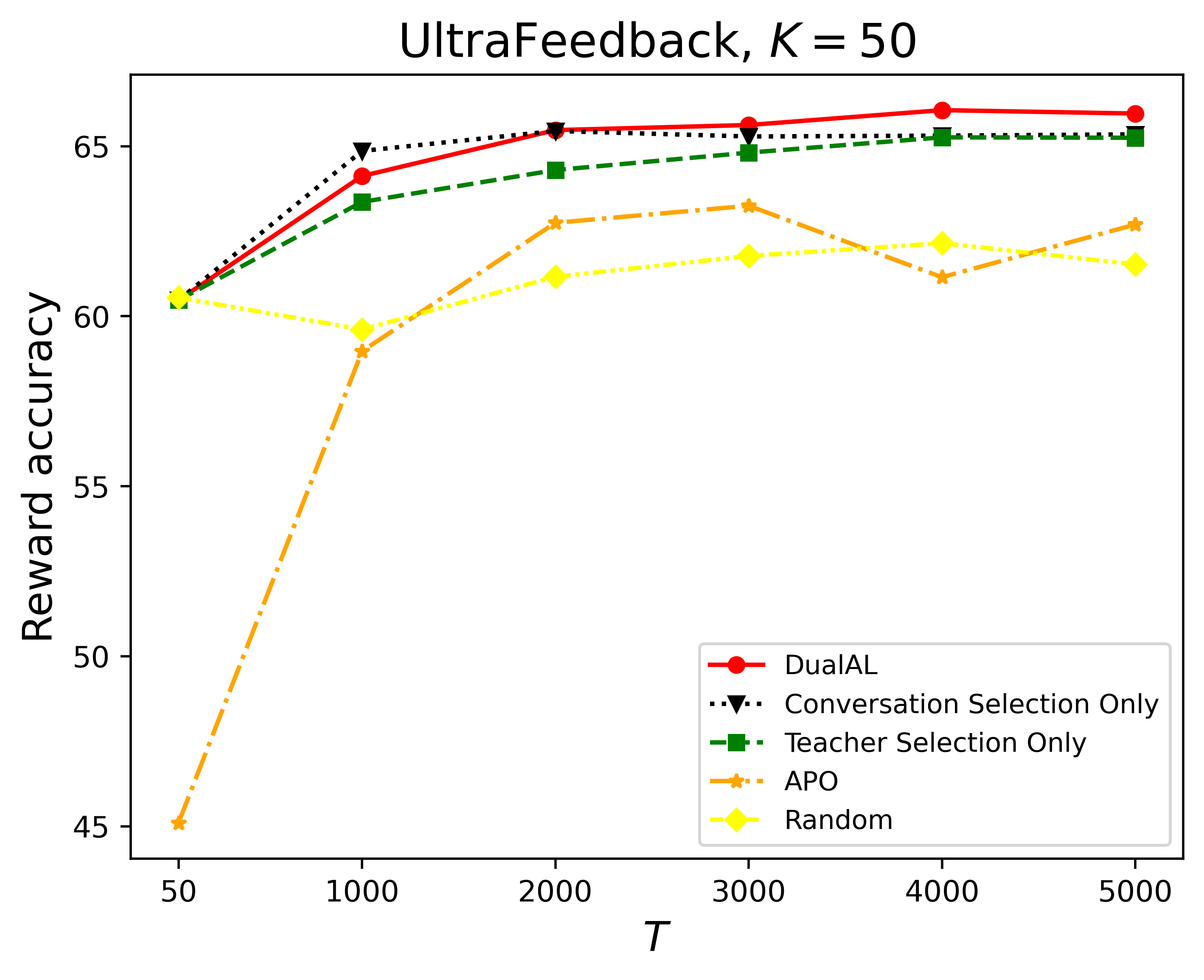}&
         \includegraphics[scale = 0.5]{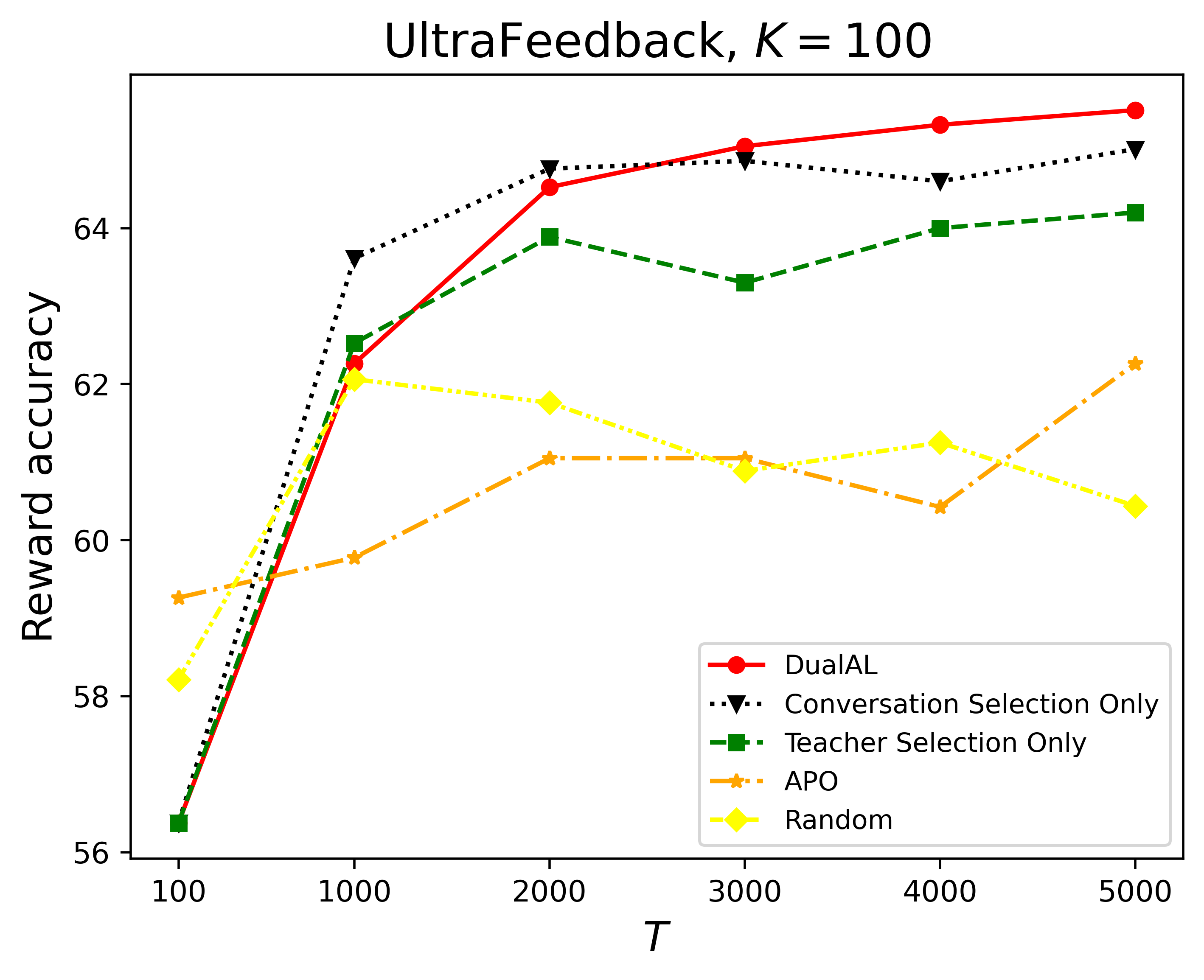}
    \end{tabular}
     \caption{Reward accuracy with different methods using the \texttt{UltraFeedback} dataset.}
     \label{fig10}
\end{figure}
\section{The Datasets and the Pretrained Model}\label{appb}

In this section, we give a brief description of the datasets and the pretrained model used in the experiments. All the descriptions are adapted from \texttt{Hugging Face}\footnote{\href{https://huggingface.co/}{https://huggingface.co/}}. The dataset \texttt{prism-alignment}\footnote{\href{https://huggingface.co/datasets/HannahRoseKirk/prism-alignment}{https://huggingface.co/datasets/HannahRoseKirk/prism-alignment}} \citep{kirk2024prism} is a diverse human feedback dataset for preference and value alignment in LLMs. It maps the characteristics and stated preferences of humans from a detailed survey onto their real-time interactions with LLMs and contextual preference ratings.
The dataset \texttt{Anthropic}\footnote{\href{https://huggingface.co/datasets/Anthropic/hh-rlhf}{https://huggingface.co/datasets/Anthropic/hh-rlhf}} \citep{bai2022training} is about helpfulness and harmlessness, and is meant to train preference (or reward) models for subsequent RLHF training. These data are not meant for supervised training of dialogue agents. For helpfulness, the data are grouped into train/test splits in three tranches: from our base models (context-distilled 52B language models), via rejection sampling (mostly with best-of-16 sampling) against an early preference model, and a dataset sampled during our iterated ``online'' process. For harmlessness, the data are only collected for our base models, but otherwise formatted in the same way. 
The dataset \texttt{UltraFeedback}\footnote{\href{https://huggingface.co/datasets/openbmb/UltraFeedback}{https://huggingface.co/datasets/openbmb/UltraFeedback}}  \citep{cui2023ultrafeedback} is a large-scale, fine-grained, diverse preference dataset, used for training powerful reward models and critic models. About 64k prompts are collected from diverse resources (including UltraChat, ShareGPT, Evol-Instruct, TruthfulQA, FalseQA, and FLAN). These prompts are then used to query multiple LLMs and generate 4 different responses for each prompt, resulting in a total of 256k samples. The \texttt{Gemma-7b-it}\footnote{\href{https://huggingface.co/google/gemma-7b-it}{https://huggingface.co/google/gemma-7b-it}} model is among the 
Gemma \citep{team2024gemma} family, which is a collection of lightweight, state-of-the-art open models from Google, built from the same research and technology used to create the Gemini models. They are text-to-text, decoder-only large language models, available in English, with open weights, pretrained variants, and instruction-tuned variants. Gemma models are well-suited for a variety of text generation tasks, including question answering, summarization, and reasoning. 
\section{Proofs}\label{appd}

\subsection{Proof of Theorem \ref{cor1}}

Recall the definition of $H_t(\hat{\theta}_{t-1})$ in \eqref{sinf}. When $H_{t-1}(\hat{\theta}_{t-1})$ is nonsingular, we have
\begin{equation}\label{tsel}
 \det [H_{t-1}(\hat{\theta}_{t-1})+\Dot{\mu}(\beta \hat{\theta}_{t-1}^\top z)\beta^2zz^\top]
 =\det H_t(\hat{\theta}_{t-1}) \left[1+\dot{\mu}(\beta  z^\top\hat{\theta}_{t-1})\beta^2 z^\top H^{-1}_{t-1}(\hat{\theta}_{t-1})z\right].     
\end{equation}
The above equation follows from Lemma \ref{harville} with $R=H_{t-1}(\hat{\theta}_{t-1})$, $\tilde{T}=1, S=\Dot{\mu}(\beta \hat{\theta}_{t-1}^\top z)\beta^2z, U=z^\top$.
At step $t$, $H_{t-1}(\hat{\theta}_{t-1})$ is fixed. From \eqref{tsel}, the maximization of $ \det[H_{t}(\hat{\theta}_{t-1})+\Dot{\mu}(\beta \hat{\theta}_{t-1}^\top z)\beta^2zz^\top]$
is equivalent to the maximization of $\dot{\mu}(\beta  z^\top\hat{\theta}_{t-1})\beta^2 z^\top H^{-1}_{t-1}(\hat{\theta}_{t-1})z$. For ease of presentation, we denote $h(\beta|z,\hat{\theta}_{t-1})=\dot{\mu}(\beta  z^\top\hat{\theta}_{t-1})\beta^2 z^\top H^{-1}_{t-1}(\hat{\theta}_{t-1})z$. The rationality parameter $\beta$ influences $h(\beta|z,\hat{\theta}_{t-1})$ through two aspects: $\Dot{\mu}(\beta \hat{\theta}_{t-1}^\top z)$ and $\beta^2$. Recall that $\beta>0$. On the one hand, a large $\beta$ leads to a larger $\beta^2$, which contributes to the increase of $h(\beta|z,\hat{\theta}_{t-1})$. On the other hand, $\beta$ affects $h(\beta|z,\hat{\theta}_{t-1})$ through $\Dot{\mu}(\beta \hat{\theta}_{t-1}^\top z)$. By simple calculation, we have $\Dot{\mu}(\beta \hat{\theta}_{t-1}^\top z)=\mu(\beta \hat{\theta}_{t-1}^\top z)[1-\mu(\beta \hat{\theta}_{t-1}^\top z)]$. Clearly, an increase in $\beta$ does not always lead to an increase in $\Dot{\mu}(\beta \hat{\theta}_{t-1}^\top z)$. Thus, a more rational teacher is not always the most informative.

\subsection{Proof of Lemma \ref{lemma0}}

We denote $\mathcal{L}_T(\theta)=-TL_T(\theta)$, where $L_T(\theta)$ is defined in \eqref{lik}. Then, the MLE is $\hat{\theta}_T=\mathop{\arg\min}_{\theta\in\Theta}\mathcal{L}_T(\theta)$.
By the Taylor expansion \citep{lee2024}, we have
\begin{equation}\label{1eq1}
\mathcal{L}_T(\theta)=\mathcal{L}_T(\theta_*)+\nabla\mathcal{L}_T(\theta_*)^\top (\theta-\theta_*)+\|\theta-\theta_*\|^2_{G_T(\theta_*,\theta)},   
\end{equation}
where
$$G_T(\theta_*,\theta)=\sum_{t=1}^T\left[\int_{0}^1(1-v)\dot{\mu}(\beta_tz_t^\top (\theta_*+v(\theta-\theta_*)))dv\right]\beta_t^2z_tz_t^\top.$$ 
By the definition of $H_T(\theta)$ in \eqref{sinf}, we have
\begin{align*}
H_{T}(\theta)&=\sum_{t=1}^{T}\Dot{\mu}(\beta_{t} \theta^\top z_t)\beta_{t}^2z_tz_t^\top\\
&\preceq \sum_{t=1}^T\left[C(2+|\beta_tz_t^\top (\theta-\theta_*)|)^2\int_{0}^1(1-v)\dot{\mu}(\beta_tz_t^\top (\theta_*+v(\theta-\theta_*)))dv\right]\beta_t^2z_tz_t^\top\\
&\preceq \sum_{t=1}^T\left[C(2+2C_\beta C_\theta C_z)^2\int_{0}^1(1-v)\dot{\mu}(\beta_tz_t^\top (\theta_*+v(\theta-\theta_*)))dv\right]\beta_t^2z_tz_t^\top\\
&=C(2+2C_\beta C_\theta C_z)^2G_T(\theta_*,\theta),
\end{align*}
where the first inequality follows from  Lemma \ref{lemma11} with some constant $C>1$, and the second inequality is due to Assumption \ref{assu2}.
Then $H_T(\hat{\theta}_T)\preceq C(2+2C_\beta C_\theta C_z)^2G_T(\theta_*,\hat{\theta}_T)$ for some $C>1$. Together with \eqref{1eq1}, we have
\begin{equation}\label{1eq0}
\begin{aligned}
\|\hat{\theta}_T-\theta_*\|_{H_T(\hat{\theta}_T)}^2&\leq C(2+2C_\beta C_\theta C_z)^2\|\hat{\theta}_T-\theta_*\|_{G_T(\theta_*,\hat{\theta}_T)}^2\\
&=C(2+2C_\beta C_\theta C_z)^2[\mathcal{L}_T(\hat{\theta}_T)-\mathcal{L}_T(\theta_*)+\nabla \mathcal{L}_T(\theta_*)^\top (\theta_*-\hat{\theta}_T)]\\
&\leq C(2+2C_\beta C_\theta C_z)^2\nabla \mathcal{L}_T(\theta_*)^\top (\theta_*-\hat{\theta}_T),
\end{aligned}    
\end{equation}
where the last inequality is from $\mathcal{L}_T(\hat{\theta}_T)\leq \mathcal{L}_T(\theta_*)$.
Now, we bound $\nabla \mathcal{L}_T(\theta_*)^\top (\theta_*-\hat{\theta}_T)$.
We define $\xi_t=\mu (\beta_t z_t^\top \theta_*)-y_t$. Then,
\begin{equation}\label{1eq2}
 \nabla \mathcal{L}_T(\theta_*)^\top (\theta_*-\theta)=\sum_{t=1}^T [\mu (\beta_t z_t^\top \theta_*)-y_t]\beta_t z_t^\top (\theta_*-\theta)=\sum_{t=1}^T \xi_t\beta_t z_t^\top (\theta_*-\theta).   
\end{equation}
Here $\xi_t$ is a martingale difference sequence w.r.t. $\mathcal{F}_{t-1}=\sigma(z_1,\beta_1, y_1,\cdots,z_{t-1}, \beta_{t-1}, y_{t-1}, z_t, \beta_t)$. Then $\xi_t\beta_t z_t^\top (\theta_*-\theta)$ is a martingale difference sequence. Since $|\xi_t\beta_t z_t^\top (\theta_*-\theta)|\leq 2 C_\beta C_z C_\theta$ and $\mathbb{E}\{[\xi_t\beta_t z_t^\top (\theta_*-\theta)]^2|\mathcal{F}_{t-1}\}=\dot{\mu} (\beta_t z_t^\top \theta_*)[\beta_t z_t^\top (\theta_*-\theta)]^2$, by Lemma \ref{lemma10}, for any $\eta\in (0, \frac{1}{2 C_\beta C_z C_\theta}]$, with probability at least $1-\frac{\delta}{2}$, we have
\begin{equation}\label{1eq3}
\begin{aligned}
\sum_{t=1}^T \xi_t\beta_t z_t^\top (\theta_*-\theta)&\leq (e-2)\eta \sum_{t=1}^T \dot{\mu} (\beta_t z_t^\top \theta_*)[\beta_t z_t^\top (\theta_*-\theta)]^2+\frac{1}{\eta}\log \frac{2}{\delta}\\
&=(e-2)\eta\|\theta_*-\theta\|_{H_T(\theta_*)}^2+\frac{1}{\eta}\log \frac{2}{\delta}.
\end{aligned}
\end{equation}
By \eqref{1eq2} and \eqref{1eq3}, replacing $\theta$ with $\hat{\theta}_T$, with probability at least $1-\frac{\delta}{2}$, we have
\begin{equation}\label{1eq02}
 \nabla \mathcal{L}_T(\theta_*)^\top (\theta_*-\hat{\theta}_T)\leq (e-2)\eta\|\theta_*-\hat{\theta}_T\|_{H_T(\theta_*)}^2+\frac{1}{\eta}\log \frac{2}{\delta}.   
\end{equation}
By setting $\eta=\frac{1}{(e-2)(4+4C_\beta C_z C_\theta)}$, similar to the arguments in Lemma 6 of \cite{lee2024}, with probability at least $1-\frac{\delta}{2}$, we can obtain
\begin{equation}\label{1eq4}
\|\theta_*-\hat{\theta}_T\|_{H_T(\theta_*)}^2\leq C' (C_\beta C_z C_\theta)^2 \left[d\log \left(e+\frac{C_\beta C_z C_\theta T}{d}\right)+\log \frac{2}{\delta}\right]
\end{equation}
for some positive constant $C'$. By \eqref{1eq02} and \eqref{1eq4}, with probability at least $1-\delta$, we have
\begin{equation}\label{1eq5}
\nabla \mathcal{L}_T(\theta_*)^\top (\theta_*-\hat{\theta}_T)\leq \frac{C' (C_\beta C_z C_\theta)^2 }{4+4C_\beta C_z C_\theta} \left[d\log \left(e+\frac{C_\beta C_z C_\theta T}{d}\right)+\log \frac{2}{\delta}\right]+(e-2)(4+4C_\beta C_z C_\theta)\log \frac{2}{\delta}.    
\end{equation}
We define $C_1=\frac{CC' (C_\beta C_z C_\theta)^2(2+2C_\beta C_z C_\theta) }{2}+2C(e-2)(2+2C_\beta C_z C_\theta)^3$ and $C_2=C_\beta C_z C_\theta$.
By \eqref{1eq0} and \eqref{1eq5}, with probability at least $1-\delta$, we have
\begin{align*}
\|\hat{\theta}_T-\theta_*\|_{H_T(\hat{\theta}_T)}&\leq \sqrt{C_1 \left[d\log \left(e+\frac{C_2 T}{d}\right)+\log \frac{2}{\delta}\right]}.
\end{align*}
We define $\bar{H}_T(\hat{\theta}_T)=\frac{1}{T}H_T(\hat{\theta}_T)$. Then, with probability at least $1-\delta$, it follows that
\begin{align*}
\|\hat{\theta}_T-\theta_*\|_{\bar{H}_T(\hat{\theta}_T)}&\leq \sqrt{\frac{C_1}{T}\left[ d\log \left(e+\frac{C_2 T}{d}\right)+\log \frac{2}{\delta}\right]}.
\end{align*}

\subsection{Proof of Theorem \ref{thm1}}

Now we prove the asymptotic property of $\hat{\theta}_T$.
The convergence of this adaptively generated information matrix $M(\xi_T, \hat{\theta}_T)$ is established in the following theorem.
\begin{theorem}\label{thm2}
Assuming that Assumptions \ref{assu1} and \ref{assu2} are satisfied and $\hat{\theta}_T$ is the estimator derived from Algorithm \ref{alg}, let $M(\xi_T,\hat{\theta}_T)$ be as defined in \eqref{infor} and $M(\xi_*,\theta_*)$ as in \eqref{d-opti}. It follows that
$$M(\xi_T,\hat{\theta}_T)\xrightarrow{a.s.}M(\xi_*,\theta_*),\ as \ T\rightarrow \infty,$$
    where $\xrightarrow{a.s.}$ denotes convergence almost surely.
\end{theorem}
Theorem \ref{thm2}  asserts that the information matrix $M(\xi_T,\hat{\theta}_T)$ converges almost surely to $M(\xi_*,\theta_*)$, which maximizes $\det M(\xi,\theta_*)$ over the set of all designs. The proof of Theorem \ref{thm2} is deferred to Section \ref{appt2}.

We take the gradient of $L_T(\theta)$ with respect to $\theta$ as follows,
 \begin{equation}\label{t1e0}
\frac{\partial L_T(\theta)}{\partial \theta}=\frac{1}{T}\sum_{t=1}^T\Dot{\mu}(\beta_{t} \theta^\top z_t)\left[\frac{y_t\beta_{t} z_t}{\mu(\beta_{t} \theta^\top z_t)}-\frac{(1-y_t)\beta_{t}z_t}{1-\mu(\beta_{t} \theta^\top z_t)}\right]
=\frac{1}{T}\sum_{t=1}^T[y_t-\mu(\beta_{t} \theta^\top z_t)]\beta_{t} z_t.
\end{equation}
We denote $S_T(\theta)=\frac{\partial L_T(\theta)}{\partial \theta}$ as the score function.
Since $S_T(\hat{\theta}_T)=0$, by the Taylor expansion, we have
\begin{align*}
S_T(\theta^*)&=S_T(\theta^*)-S_T(\hat{\theta}_T)\\
&=\frac{1}{T}\sum_{t=1}^T[\mu(\beta_{t} \hat{\theta}_T^\top z_t)-\mu(\beta_{t} z_t^\top\theta_* )]\beta_{t} z_t\\
&=\frac{1}{T}\sum_{t=1}^T\dot{\mu}(\beta_{t} \tilde{\theta}_t^\top z_t)\beta^2_{t} z_tz_t^\top ( \hat{\theta}_T-\theta_*)\\
&=\frac{1}{T}\left[\sum_{t=1}^T\dot{\mu}(\beta_{t} \tilde{\theta}_t^\top z_t)\beta^2_{t} z_tz_t^\top-\sum_{t=1}^T\dot{\mu}(\beta_{t} z_t^\top\theta_* )\beta^2_{t} z_tz_t^\top+\sum_{t=1}^T\dot{\mu}(\beta_{t} z_t^\top\theta_* )\beta^2_{t} z_tz_t^\top\right] ( \hat{\theta}_T-\theta_*)\\
&=\left\{\frac{1}{T}\sum_{t=1}^T[\dot{\mu}(\beta_{t} \tilde{\theta}_t^\top z_t)-\dot{\mu}(\beta_{t} z_t^\top\theta_* )]\beta^2_{t} z_tz_t^\top+M(\xi_T,\theta_*)\right\}( \hat{\theta}_T-\theta_*),
\end{align*}
where $\tilde{\theta}_t$ is on the line segment joining $\theta_*$ and $\hat{\theta}_T$. We denote $M_*=M(\xi_*,\theta_*)$. Therefore,
\begin{equation}\label{t1e8}
\sqrt{T}M_*^{-1/2}S_T(\theta_*)=M_*^{-1/2}\left\{\frac{1}{T}\sum_{t=1}^T[\dot{\mu}(\beta_{t} \tilde{\theta}_t^\top z_t)-\dot{\mu}(\beta_{t} z_t^\top\theta_* )]\beta^2_{t} z_tz_t^\top+M(\xi_T,\theta_*)\right\}\sqrt{T}( \hat{\theta}_T-\theta_*).
\end{equation}
We propose a lemma to show that the left side of \eqref{t1e8} converges to a multivariate normal distribution.
\begin{lemma} \label{t1e7}
Let $M_*=M(\xi_*,\theta_*)$ be defined in \eqref{d-opti} and $S_T(\theta_*)$ be the score function defined in \eqref{t1e0}. Under Assumption \ref{assu2}, we have
$$\sqrt{T}M_*^{-1/2}S_T(\theta_*)\xrightarrow{d}N(0, I_d).$$
\end{lemma}
\begin{proof}
Let $\tilde{v}\in\mathbb{R}^d$ and $v=\tilde{v}/\|\tilde{v}\|$. Then, $\|v\|=1$. Recall that $e_i=y_i-\mu(\beta_iz_i^\top\theta^*)$ is defined in \eqref{error}. By \eqref{t1e0}, we have
\begin{equation}\label{t1e4}
\sqrt{T}v^\top M_*^{-1/2}S_T(\theta_*)=\frac{1}{\sqrt{T}}\sum_{t=1}^Te_t\beta_tv^\top M_*^{-1/2}z_t.  
\end{equation}
We define the $\sigma$-field generated by the historical data as follows,
\begin{equation}\label{field}
 \mathcal{F}_t=\sigma(z_1,\dots,z_t;\beta_1,\cdots,\beta_t; y_1,\dots,y_t)   
\end{equation}
Under Assumption \ref{assu2}, $|\sum_{t=1}^Te_t\beta_tv^\top M_*^{-1/2}z_t|$ is bounded. Since $z_t$ and $\beta_t$ in Algorithm \ref{alg} are determined by $\mathcal{F}_{t-1}$, $z_t$ and $\beta_t$ are measurable with respect to $\mathcal{F}_{t-1}$. Therefore, we have
\begin{align*}
 \mathbb{E}\left(\sum_{t=1}^Te_t\beta_tv^\top M_*^{-1/2}z_t|\mathcal{F}_{T-1}\right)&=  \mathbb{E}\left(\sum_{t=1}^{T-1}e_t\beta_tv^\top M_*^{-1/2}z_t|\mathcal{F}_{T-1}\right)+\mathbb{E}(e_T\beta_Tv^\top M_*^{-1/2}z_T|\mathcal{F}_{T-1})\\
 &=\sum_{t=1}^{T-1}e_t\beta_tv^\top M_*^{-1/2}z_t+\mathbb{E}(e_T|\mathcal{F}_{T-1})\beta_Tv^\top M_*^{-1/2}z_T\\
 &=\sum_{t=1}^{T-1}e_t\beta_tv^\top M_*^{-1/2}z_t,
\end{align*}
Thus, the sequence of partial sums $\sum_{t=1}^Te_t\beta_tv^\top M_*^{-1/2}z_t$ is a martingale with respect to $\mathcal{F}_T$. Since
\begin{align*}
\mathbb{E}(e_t^2|\mathcal{F}_{t-1})&=\mathbb{E}\{[y_t-\mu(\beta_t\theta_*^\top z_t)]^2|\mathcal{F}_{t-1}\}\\
&=\mathbb{E}(y_t^2|\mathcal{F}_{t-1})+\mu^2(\beta_t\theta_*^\top z_t)-2\mathbb{E}(y_t|\mathcal{F}_{t-1})\mu(\beta_t\theta_*^\top z_t)\\
&=\mu(\beta_t\theta_*^\top z_t)-\mu^2(\beta_t\theta_*^\top z_t)\\
&=\dot{\mu}(\beta_t\theta_*^\top z_t),
\end{align*}
we have
\begin{equation}\label{t1e2}
\begin{aligned}
\frac{1}{T}\sum_{t=1}^T\mathbb{E}[(e_t\beta_tv^\top M_*^{-1/2}z_t)^2|\mathcal{F}_{t-1}]&=\frac{1}{T}\sum_{t=1}^T\mathbb{E}[e_t^2|\mathcal{F}_{t-1}](\beta_tv^\top M_*^{-1/2}z_t)^2\\
&=\frac{1}{T}\sum_{t=1}^Tv^\top M_*^{-1/2}\dot{\mu}(\beta_t\theta_*^\top z_t)\beta_t^2z_tz_t^\top M_*^{-1/2}v\\
&=v^\top M_*^{-1/2}M(\xi_T,\theta_*) M_*^{-1/2}v\\
&\xrightarrow{a.s.}1,
\end{aligned}
\end{equation}
where the convergence follows from Theorem \ref{thm2}. For all $\epsilon>0$, we have
\begin{equation}\label{t1e3}
\begin{aligned}
&~~~\frac{1}{T}\sum_{t=1}^T\mathbb{E}[(e_t\beta_tv^\top M_*^{-1/2}z_t)^2\mathbb{I}(|e_t\beta_tv^\top M_*^{-1/2}z_t|>\sqrt{T}\epsilon)|\mathcal{F}_{t-1}]\\
&\leq \frac{1}{\epsilon^2T^2}\sum_{t=1}^T\mathbb{E}[(e_t\beta_tv^\top M_*^{-1/2}z_t)^4|\mathcal{F}_{t-1}]\\
&=\frac{1}{\epsilon^2T^2}\sum_{t=1}^T(\beta_tv^\top M_*^{-1/2}z_t)^4\mathbb{E}(e_t^4|\mathcal{F}_{t-1})\\
&\xrightarrow{a.s.}0,
\end{aligned}
\end{equation}
where the first inequality follows from $$(e_t\beta_tv^\top M_*^{-1/2}z_t)^2\mathbb{I}(|e_t\beta_tv^\top M_*^{-1/2}z_t|>\sqrt{T}\epsilon)\leq \frac{(e_t\beta_tv^\top M_*^{-1/2}z_t)^4}{\epsilon^2T},$$
and the convergence is from the fact that $(\beta_tv^\top M_*^{-1/2}z_t)^4\mathbb{E}(e_t^4|\mathcal{F}_{t-1})$ is bounded under Assumption \ref{assu2}.
By \eqref{t1e2}, \eqref{t1e3} and Lemma \ref{lemma9}, we have 
$$\frac{1}{\sqrt{T}}\sum_{t=1}^Te_t\beta_tv^\top M_*^{-1/2}z_t\xrightarrow{d}N(0,1).$$
Using \eqref{t1e4}, we have
\begin{equation*}
\sqrt{T}v^\top M_*^{-1/2}S_T(\theta^*)\xrightarrow{d}N(0,1).
\end{equation*}
Let $\tilde{Z}$ be a normal vector with $\tilde{Z}\sim N(0, I_d)$. Then, $v^\top \tilde{Z}\sim N(0, 1)$ because of $\|v\|=1$. Therefore, for any $\tilde{v}\in\mathbb{R}^d$, we have
\begin{equation*}
\sqrt{T}\frac{\tilde{v}^\top}{\|\tilde{v}\|} M_*^{-1/2}S_T(\theta^*)\xrightarrow{d}\frac{\tilde{v}^\top}{\|\tilde{v}\|}\tilde{Z}.
\end{equation*}
Thus,
$$\sqrt{T}\tilde{v}^\top M_*^{-1/2}S_T(\theta^*)\xrightarrow{d}\tilde{v}^\top\tilde{Z}. $$
By Lemma \ref{billingsley}, we have
\begin{equation*}
\sqrt{T} M_*^{-1/2}S_T(\theta^*)\xrightarrow{d}N(0,I_d).   
\end{equation*}
\end{proof}
We now return to the proof of Theorem \ref{thm1}. By \eqref{t1e8} and Lemma \ref{t1e7}, we have
\begin{equation}
M_*^{-1/2}\left\{\frac{1}{T}\sum_{t=1}^T[\dot{\mu}(\beta_{t} \tilde{\theta}_t^\top z_t)-\dot{\mu}(\beta_{t} z_t^\top\theta_* )]\beta^2_{t} z_tz_t^\top+M(\xi_T,\theta_*)\right\}\sqrt{T}( \hat{\theta}_T-\theta_*)\xrightarrow{d}N(0,I_d).
\end{equation}
Under Assumption \ref{assu2}, we have
\begin{equation}\label{t1e5}
    \begin{aligned}
\frac{1}{T}\sum_{t=1}^T[\dot{\mu}(\beta_{t} \tilde{\theta}_t^\top z_t)-\dot{\mu}(\beta_{t} z_t^\top\theta_* )]\beta^2_{t} z_tz_t^\top &\leq \mathop{\max}_{1\leq t\leq T}|\dot{\mu}(\beta_{t} \tilde{\theta}_t^\top z_t)-\dot{\mu}(\beta_{t} \theta_*^\top z_t)|\frac{1}{T}\sum_{t=1}^T\beta_{t}^2\|z_t\|^2\\
&\leq \mathop{\max}_{1\leq t\leq T}|\dot{\mu}(\beta_{t} \tilde{\theta}_t^\top z_t)-\dot{\mu}(\beta_{t} \theta_*^\top z_t)|C_\beta^2C_z^2.
    \end{aligned}
\end{equation}
We denote $\ddot{\mu}(w)=\frac{d \dot{\mu}(w)}{d w}=\frac{d \mu(w)[1-\mu(w)]}{d w}=\mu(w)[1-\mu(w)][1-2\mu(w)]$ for $w\in\mathbb{R}$. Under Assumption \ref{assu2}, there exists a positive constant $C_\mu$ such that $\ddot{\mu}(\beta \theta^\top z )\leq C_\mu$ for any $\beta\in\mathcal{B}$, $z\in\mathcal{Z}$ and $\theta\in \Theta$. By the Taylor expansion, we have
\begin{equation}\label{t1e6}
\begin{aligned}
 \mathop{\max}_{1\leq t\leq T}|\dot{\mu}(\beta_{t} \tilde{\theta}_t^\top z_t)-\dot{\mu}(\beta_{t} \theta_*^\top z_t )|&=\mathop{\max}_{1\leq t\leq T}|\ddot{\mu}(\beta_{t} \bar{\theta}_t^\top z_t)\beta_tz_t^\top(\tilde{\theta}_t-\theta_*)| &\leq  C_\mu C_\beta C_z\mathop{\max}_{1\leq t\leq T}\|\tilde{\theta}_t- \theta_* \|^2.
 \end{aligned}
\end{equation}
Recall that $\tilde{\theta}_t$ is between $\hat{\theta}_t$ and $\theta_*$. We have 
\begin{equation}\label{t1e10}
\mathop{\max}_{1\leq t\leq T}\|\tilde{\theta}_t- \theta_* \|^2\leq  \mathop{\max}_{1\leq t\leq T}\|\hat{\theta}_t- \theta_* \|^2 \xrightarrow{a.s.} 0,
\end{equation}
where the convergence follows from Lemma \ref{thecon}. By \eqref{t1e5}, \eqref{t1e6} and \eqref{t1e10}, we have
\begin{align*}
  \frac{1}{T}\sum_{t=1}^T[\dot{\mu}(\beta_{t} \tilde{\theta}_t^\top z_t)-\dot{\mu}(\beta_{t} z_t^\top\theta_* )]\beta^2_{t} z_tz_t^\top\xrightarrow{a.s.}0.  
\end{align*}
Together with Theorem \ref{thm2}, we have
$$M_*^{-1/2}\left\{\frac{1}{T}\sum_{t=1}^T[\dot{\mu}(\beta_{t} \tilde{\theta}_t^\top z_t)-\dot{\mu}(\beta_{t} z_t^\top\theta_* )]\beta^2_{t} z_tz_t^\top+M(\xi_T,\theta_*)\right\}\xrightarrow{a.s.} M_*^{1/2}.$$
By \eqref{t1e8} and Lemma \ref{t1e7}, we have
$$M_*^{1/2}\sqrt{T}( \hat{\theta}_T-\theta^*)\xrightarrow{d}N(0, I_d).$$ 
It follows that
$$ \sqrt{T}(\hat{\theta}_T-\theta^*)\xrightarrow{d}N(0, M_*^{-1}).$$

The proofs of the asymptotic properties of $\hat{\theta}^c_T$ and $\hat{\theta}^t_T$ follow steps similar to those for $\hat{\theta}_T$, and are therefore omitted.

\subsection{Proof of Theorem \ref{thm2}}\label{appt2}
In this section, we first propose Lemma \ref{thecon} to show the strong consistency of the adaptive MLE $\hat{\theta}_T$. This lemma plays a pivotal role as a fundamental
component in the proof of Theorem \ref{thm2}. Under Assumption \ref{assu1}, the initial information matrix $M(\xi_{t_0}, \theta)$ is constructed as positive definite for any $\theta\in\Theta$ for a theoretical requirement.

\begin{lemma}\label{thecon}
Denote $\hat{\theta}_T$ as the estimator from Algorithm \ref{alg}. We have
$$\hat{\theta}_T\xrightarrow{a.s.} \theta_*.$$
\end{lemma}
\begin{proof}
According to \eqref{lik}, we calculate the log-likelihood difference between $\theta_*$ and $\theta\in \Theta$ as
\begin{equation}\label{let1}
\begin{aligned}
L_T(\theta_*)-L_T(\theta)&=\frac{1}{T}\sum_{t=1}^T\left\{y_t\log \frac{\mu(\beta_t z_t^\top\theta_*)}{\mu(\beta_t z_t^\top\theta)}+(1-y_t)\log \frac{1-\mu(\beta z_t^\top\theta_*)}{1-\mu(\beta_t z_t^\top\theta)}\right\}\\
&=\frac{1}{T}\sum_{t=1}^T\left\{y_t\left[\log \frac{\mu(\beta_t z_t^\top\theta_*)}{1-\mu(\beta z_t^\top\theta_*)}-\log \frac{\mu(\beta_t z_t^\top\theta)}{1-\mu(\beta_t z_t^\top\theta)}\right]+\log \frac{1-\mu(\beta z_t^\top\theta_*)}{1-\mu(\beta_t z_t^\top\theta)}\right\}\\
&=\frac{1}{T}\sum_{t=1}^T\{y_t\beta_tz_t^\top (\theta^*-\theta)+\log[1-\mu(\beta z_t^\top\theta^*)]-\log[1-\mu(\beta z_t^\top\theta)]\},
\end{aligned}
\end{equation}
where the last equality is from 
\begin{equation*}
 \log \frac{\mu(\beta_t z_t^\top\theta_*)}{1-\mu(\beta z_t^\top\theta_*)}-\log \frac{\mu(\beta_t z_t^\top\theta)}{1-\mu(\beta_t z_t^\top\theta)}=\log e^{\beta_t z_t^\top\theta_*}-   \log e^{\beta_t z_t^\top\theta}=\beta_t z_t^\top(\theta_*-\theta).
\end{equation*}
Taking the first-order derivative of $\log [1-\mu(w)]$ with respect to $w$, we obtain
$$ \frac{d\log [1-\mu(w)]}{d w}=-\frac{\dot{\mu}(w)}{1-\mu(w)}=-\frac{\mu(w)[1-\mu(w)]}{1-\mu(w)}=-\mu(w),$$
and the second-order derivative is 
$$\frac{d^2\log [1-\mu(w)]}{d w^2}=-\dot{\mu}(w).$$
Therefore, by the second-order Taylor expansion of $\log[1-\mu(\beta z^\top\theta)]$ at $\beta z^\top\theta_*$, we have
\begin{equation*}
 \log[1-\mu(\beta z^\top\theta)]=\log[1-\mu(\beta z^\top\theta_*)]-\mu(\beta z^\top\theta_*)\beta z^\top(\theta-\theta_*)-\frac{1}{2}\dot{\mu}(\beta z^\top\tilde{\theta})[\beta z^\top(\theta-\theta_*)]^2,   
\end{equation*}
where $\tilde{\theta}$ is between $\theta$ and $\theta_*$.
Therefore,
\begin{equation}\label{let2}
\log[1-\mu(\beta z^\top\theta_*)]-\log[1-\mu(\beta z^\top\theta)]=\mu(\beta z^\top\theta_*)\beta z^\top(\theta-\theta_*)+\frac{1}{2}\dot{\mu}(\beta z^\top\tilde{\theta})[\beta z^\top(\theta-\theta_*)]^2.
\end{equation}
Now, we define the error terms as 
\begin{equation}\label{error}
 e_t=y_t-\mu(\beta_tz_t^\top\theta_*).   
\end{equation}
Combining \eqref{let1} and \eqref{let2}, we obtain
\begin{equation}\label{let3}
\begin{aligned}
L_T(\theta_*)-L_T(\theta)&=\frac{1}{T}\sum_{t=1}^T\{[\mu(\beta_tz_t^\top \theta_*)+e_t]\beta_tz_t^\top (\theta_*-\theta)+\log[1-\mu(\beta z_t^\top\theta_*)]-\log[1-\mu(\beta z_t^\top\theta)]\}\\
&=\frac{1}{T}\sum_{t=1}^Te_t\beta_tz_t^\top (\theta_*-\theta)+\frac{1}{2T}\sum_{t=1}^T\dot{\mu}(\beta_t z_t^\top\tilde{\theta}_t)[\beta_t z_t^\top(\theta-\theta_*)]^2\\
&\geq \frac{1}{T}\sum_{t=1}^Te_t\beta_tz_t^\top (\theta_*-\theta)+\frac{\kappa }{2T}\sum_{t=1}^T[\beta_t z_t^\top(\theta-\theta_*)]^2.
\end{aligned}    
\end{equation}
For any $\delta>0$, we define the parameter subset $C(\theta_*, \delta)=\{\theta\in\Theta: \|\theta-\theta_*\|\geq \delta\}$.
Then, for any $\delta>0$,  by \eqref{let3}, we have
\begin{equation}\label{t0e0}
L_T(\theta^*)-\mathop{\sup}_{\theta\in C(\theta_*, \delta)}L_T(\theta)\geq -\frac{1}{T}\mathop{\sup}_{\theta\in \Theta}\left|\sum_{t=1}^Te_t\beta_tz_t^\top (\theta^*-\theta)\right|+\frac{\kappa}{2T}\mathop{\inf}_{\theta\in C(\theta_*, \delta)}\sum_{t=1}^T[\beta_t z_t^\top(\theta-\theta_*)]^2.
\end{equation}
Letting $n_{i,j}^{(k)}$ be the number of observations taken at $(z^{(i)}, \beta^{(k)}_j)$ under the generated design $\xi_T$, we have
\begin{equation}\label{t0e1}
\begin{aligned}
\frac{1}{T}\mathop{\inf}_{\theta\in C(\theta_*, \delta)}\sum_{t=1}^T[\beta_t z_t^\top(\theta-\theta_*)]^2&=\frac{1}{T}\mathop{\inf}_{\theta\in C(\theta_*, \delta)}\sum_{i=1}^{n}\sum_{j=1}^m\sum_{k=1}^gn_{i,j}^{(k)}[\beta_{j}^{(k)} (\theta-\theta_*)^\top z^{(i)}]^2\\
&=\mathop{\inf}_{\theta\in C(\theta_*, \delta)}\sum_{i=1}^{n}\sum_{j=1}^m\sum_{k=1}^g\xi_T(\beta_{j}^{(k)},z^{(i)})[\beta_{j}^{(k)} (\theta-\theta_*)^\top z^{(i)}]^2.
\end{aligned}
\end{equation}
For any $\theta\in\Theta$ and $\theta\neq \theta_*$, we define $c_\theta=(\theta-\theta_*)/\|\theta-\theta_*\|$.
By Theorem 2.6 in \cite{Freise2021}, there exist $t_0>0, \epsilon>0$ and $\alpha\in(0,1)$ such that for all $ T\geq t_0$ and $ \theta\neq \theta_*$, $$\sum_{\beta\in \mathcal{B}, z\in\mathcal{Z}}\xi_T(\beta, z)\mathbb{I}(| \sqrt{\dot{\mu}(\beta z^\top \theta)}\beta c_\theta^\top z|\leq \epsilon)\leq \alpha.$$  Noting that $\dot{\mu}(\beta z^\top \theta)\leq 1/4$, we have $| \sqrt{\dot{\mu}(\beta z^\top \theta)}\beta c_\theta^\top z|\leq | \beta c_\theta^\top z|/2$. Therefore, 
$| \beta c_\theta^\top z|\leq 2\epsilon$ implies $| \sqrt{\dot{\mu}(\beta z^\top \theta)}\beta c_\theta^\top z|\leq \epsilon$. Then, $\mathbb{I}(|\beta c_\theta^\top z|\leq 2\epsilon)\leq \mathbb{I}(| \sqrt{\dot{\mu}(\beta z^\top \theta)}\beta c_\theta^\top z|\leq \epsilon)$. Thus, there exist $t_0>0, \epsilon>0$ and $\alpha\in(0,1)$ such that for all $ T\geq t_0$ and $ \theta\neq \theta_*$, $$\sum_{\beta\in \mathcal{B}, z\in\mathcal{Z}}\xi_T(\beta, z)\mathbb{I}(|\beta c_\theta^\top z|\leq 2\epsilon)\leq \alpha.$$ Because $\sum_{\beta\in \mathcal{B}, z\in\mathcal{Z}}\xi_T(\beta,z)=1$, we have 
\begin{equation}\label{t0e5}
\sum_{\beta\in \mathcal{B}, z\in\mathcal{Z}}\xi_T(\beta, z)\mathbb{I}(|\beta c_\theta^\top z|> 2\epsilon)\geq 1-\alpha.    
\end{equation} 
Then,
\begin{equation}\label{t0e2}
\begin{aligned}
~~~~&\mathop{\inf}_{\theta\in C(\theta_*, \delta)}\sum_{i=1}^{n}\sum_{j=1}^m\sum_{k=1}^g\xi_T(\beta_{j}^{(k)},z^{(i)})[\beta_{jk} (\theta-\theta_*)^\top z^{(i)}]^2\\
&\geq \mathop{\inf}_{\theta\in C(\theta_*, \delta)}\sum_{i=1}^{n}\sum_{j=1}^m\sum_{k=1}^g\xi_T(\beta_{j}^{(k)},z^{(i)})[\beta_{jk} (\theta-\theta_*)^\top z^{(i)}]^2\mathbb{I}(|\beta_{jk} (\theta-\theta^*)^\top z^{(i)}|> 2\epsilon \|\theta-\theta^*\|)\\
&\geq 4\mathop{\inf}_{\theta\in C(\theta_*, \delta)}\sum_{i=1}^{n}\sum_{j=1}^m\sum_{k=1}^g\xi_T(\beta_{j}^{(k)},z^{(i)})\epsilon^2 \|\theta-\theta^*\|^2\mathbb{I}(|\beta_{jk} (\theta-\theta^*)^\top z^{(i)}|> 2\epsilon \|\theta-\theta^*\|)\\
&\geq 4\epsilon^2\delta^2(1-\alpha),
\end{aligned}
\end{equation}
where the last inequality is from \eqref{t0e5} and the fact that $|\beta c_\theta^\top z|> 2\epsilon$ is equivalent to $|\beta (\theta-\theta_*)^\top z|> 2\epsilon \|\theta-\theta_*\|$. 
By \eqref{t0e1} and \eqref{t0e2} we have
\begin{equation}\label{t0e30}
\mathop{\inf}_{\theta\in C(\theta_*, \delta)}\frac{1}{T}\sum_{t=1}^T[\beta_t z_t^\top(\theta-\theta_*)]^2\geq    4\epsilon^2\delta^2(1-\alpha).
\end{equation}
By Lemma A.1 in \cite{Freise2021}, we have
\begin{equation}\label{t0e4}
\frac{1}{T}\mathop{\sup}_{\theta\in \Theta}\left|\sum_{t=1}^Te_t\beta_tz_t^\top (\theta_*-\theta)\right|\xrightarrow{a.s.}0. 
\end{equation}
By \eqref{t0e0}, \eqref{t0e30} and \eqref{t0e4}, we have
$$L_T(\theta_*)-\mathop{\sup}_{\theta\in C(\theta_*, \delta)}L_T(\theta)\geq 2\kappa \epsilon^2\delta^2(1-\alpha)\ a.s.$$
 It follows
$$\mathop{\lim\inf}_{T\rightarrow\infty}[L_T(\theta_*)-\mathop{\sup}_{\theta\in C(\theta_*, \delta)}L_T(\theta)]>0\
 a.s.$$
By Lemma \ref{pron}, we have $\hat{\theta}_T\xrightarrow{a.s.}\theta_*.$
\end{proof}
Now we proceed with the proof of Theorem \ref{thm2}. By the definition of $M(\xi_T,\theta)$ as shown in \eqref{infor}, we calculate the difference of the Fisher matrices between $\hat{\theta}_T$ and $\theta_*$ at the design $\xi_T$ as follows,
\begin{equation}\label{bdM}
 \begin{aligned}
\|M(\xi_T,\hat{\theta}_T)-M(\xi_T,\theta_*)\|&=\left\|\frac{1}{T}\sum_{t=1}^T\dot{\mu}(\beta_t\hat{\theta}_T^\top z_t)\beta_t^2z_tz_t^\top-\frac{1}{T}\sum_{t=1}^T\dot{\mu}(\beta_t \theta_*^\top z_t)\beta_t^2z_tz_t^\top\right\|\\
&\leq \frac{1}{T}\sum_{t=1}^T\|\dot{\mu}(\beta_t\hat{\theta}_T^\top z_t)\beta_t^2z_tz_t^\top-\dot{\mu}(\beta_t \theta_*^\top z_t)\beta_t^2z_tz_t^\top\|\\
&\leq \frac{1}{T}\sum_{t=1}^T\|[\dot{\mu}(\beta_t\hat{\theta}_T^\top z_t)-\dot{\mu}(\beta_t\theta_*^\top z_t )]\beta_t^2z_tz_t^\top\|\\
&\leq \frac{C_\beta^2 C_z^2}{T}\sum_{t=1}^T\|\dot{\mu}(\beta_t\hat{\theta}_T^\top z_t)-\dot{\mu}(\beta_t \theta_*^\top z_t)\|\\
&\leq C_\beta^2 C_z^2\mathop{\max}_{(z, \beta)\in\mathcal{Z}\times \mathcal{B}} \|\dot{\mu}(\beta\hat{\theta}_T^\top z)-\dot{\mu}(\beta \theta_*^\top z)\|
\end{aligned}   
\end{equation}
By Lemma \ref{thecon}, we know $\hat{\theta}_T\xrightarrow{a.s.} \theta_*$. Since the real-valued function $(z,\beta, \theta)\mapsto \dot{\mu}(\beta\theta^\top z)$  is uniformly continuous on its compact domain $\mathcal{Z}\times \mathcal{B} \times \Theta$, we have 
$$\|\dot{\mu}(\beta\hat{\theta}_T^\top z)-\dot{\mu}(\beta \theta_*^\top z)\|\xrightarrow{a.s.} 0.$$ 
Using \eqref{bdM}, we have
\begin{equation}\label{conM}
    \|M(\xi_T,\hat{\theta}_T)-M(\xi_T,\theta_*)\|\xrightarrow{a.s.} 0.
\end{equation}
Under Assumption \ref{assu1} and the design in the initialization of Algorithm \ref{alg}, we have $\lambda_0:=\lambda_{min}(M(\xi_T,\hat{\theta}_T))>0$. On the other hand, by Assumption \ref{assu2}, the trace of $M(\xi_T,\hat{\theta}_T)$ is 
$$\text{tr}(M(\xi_T,\hat{\theta}_T))=\frac{1}{T}\sum_{t=1}^T\dot{\mu}(\beta_t\hat{\theta}_T^\top z_t)\beta_t^2z_t^\top z_t\leq \frac{C_\beta^2C_z^2}{4}.$$
Let $\mathcal{M}$ be the set of all non-negative definite $d\times d$ matrices $M$ such that $\lambda_{min}(M)\geq \lambda_0$ and $\text{tr}(M)\leq C_\beta^2C_z^2/4$. Obviously, $\mathcal{M}$ is compact.  We define a real-valued function $G$ on $\mathcal{Z}\times \mathcal{B}\times \Theta \times \mathcal{M}$ by
$$G(z,\beta,\theta, A)=\dot{\mu}(\beta  z^\top \theta)\beta^2z^\top A^{-1}z,$$ 
which is uniformly continuous on its compact domain $\mathcal{Z}\times \mathcal{B}\times \Theta \times \mathcal{M}$. Since $M(\xi_T,\hat{\theta}_T)\in\mathcal{M}$ and $M(\xi_T, \theta_*)\in\mathcal{M}$, by \eqref{conM}, we have
$$\mathop{\max}_{(z,\beta)\in\mathcal{Z}\times \mathcal{B}}|G(z,\beta,\hat{\theta}_T, M(\xi_T,\hat{\theta}_T))-G(z,\beta,\theta_*, M(\xi_T,\theta_*))|\xrightarrow{a.s.}0.$$
Therefore, for a given $\epsilon\in(0,1)$, there exists $t_1$ such that for all $(z,\beta)\in\mathcal{Z}\times \mathcal{B}$
\begin{equation}\label{mub}
 |\dot{\mu}(\beta \hat{\theta}_T^\top z)\beta^2z^\top M^{-1}(\xi_T,\hat{\theta}_T)z-\dot{\mu}(\beta\theta_*^\top z)\beta^2z^\top M^{-1}(\xi_{T},\theta_*)z|<\frac{\epsilon}{2}\ \text{for all}\ T\geq t_1.   
\end{equation}
 Since $H_{t-1}(\hat{\theta}_{t-1})=(t-1)M(\xi_{t-1},\hat{\theta}_{t-1})$, by the generation process of Algorithm \ref{alg} and \eqref{tsel}, equivalently, we have
 $z_t, \beta_t=\mathop{\arg \max}_{z\in \mathcal{Z}}\mathop{\max}_{\beta\in \mathcal{B}_{k}}\det [H_{t-1}(\hat{\theta}_{t-1})+\Dot{\mu}(\beta \hat{\theta}_{t-1}^\top z)\beta^2zz^\top]$ with $k$ being the type of $z$.
\begin{equation}\label{zmax}
z_{t+1},\beta_{t+1}=\mathop{\arg \max}_{z\in \mathcal{Z}}\mathop{\max}_{\beta\in \mathcal{B}_{k}}\dot{\mu}(\beta \hat{\theta}_t^\top z)\beta^2z^\top M^{-1}(\xi_t,\hat{\theta}_t)z \text{ with}\ k\ \text{being the type of}\ z.
\end{equation}
We define
\begin{equation}\label{zmax1}
z^*_{t+1},\beta^*_{t+1}=\mathop{\arg \max}_{z\in \mathcal{Z}}\mathop{\max}_{\beta\in \mathcal{B}_{k}}\dot{\mu}(\beta z^\top \theta_*)\beta^2z^\top M^{-1}(\xi_t,\theta_*)z \text{ with}\ k\ \text{being the type of}\ z.     
\end{equation} 
Then, for all $t\geq t_1$, we have
\begin{equation}\label{e11}
 \begin{aligned}
\dot{\mu}(\beta_{t+1}  z_{t+1}^\top \theta_*)\beta_{t+1}^2z_{t+1}^\top M^{-1}(\xi_{t},\theta_*)z_{t+1}&\geq \dot{\mu}(\beta_{t+1}  z_{t+1}^\top \hat{\theta}_t)\beta_{t+1}^2z_{t+1}^\top M^{-1}(\xi_{t},\hat{\theta}_t)z_{t+1}-\frac{\epsilon}{2}\\
&\geq\dot{\mu}(\beta^*_{t+1}  \hat{\theta}_t^\top z^*_{t+1})\beta_{t+1}^{*2}z^{*\top}_{t+1} M^{-1}(\xi_{t},\hat{\theta}_t)z^*_{t+1}-\frac{\epsilon}{2}\\
&\geq \dot{\mu}(\beta^*_{t+1}  \theta_*^{\top} z^*_{t+1})\beta_{t+1}^{*2}z^{*\top}_{t+1} M^{-1}(\xi_{t},\theta_*)z^*_{t+1}-\epsilon\\
&\geq d-\epsilon,
\end{aligned}   
\end{equation}
where the first and third inequalities are from \eqref{mub}, the second inequality is due to \eqref{zmax}, and the last inequality is from \eqref{zmax1} and the Kiefer–Wolfowitz equivalence theorem \citep{Kiefer_Wolfowitz_1960,Lynda,Freise2021}.
By the definition of $M(\xi_{t},\theta_*)$, we have
$$(t+1)M(\xi_{t+1},\theta_*)=tM(\xi_{t},\theta_*)+\dot{\mu}(\beta_{t+1}  \theta_*^\top z_{t+1})\beta_{t+1}^2z_{t+1}z_{t+1}^\top.$$
Then, by Lemma \ref{harville} with $R=tM(\xi_{t},\theta_*)$, $\tilde{T}=1, S=\dot{\mu}(\beta_{t+1}  \theta_*^\top z_{t+1})\beta_{t+1}^2z_{t+1}, U=z_{t+1}^\top$, we obtain
\begin{align*}
 \det[(t+1)M(\xi_{t+1},\theta_*)]
 &=\det[tM(\xi_{t},\theta_*)]\left[1+\frac{\dot{\mu}(\beta_{t+1}  z_{t+1}^\top\theta_*) \beta_{t+1}^2z_{t+1}^\top M^{-1}(\xi_{t},\theta_*)z_{t+1}}{t}\right]. 
\end{align*}
Therefore,
\begin{align*}
 \det M(\xi_{t+1},\theta_*)=\left(\frac{t}{t+1}\right)^d\det M(\xi_{t},\theta_*)\left[1+\frac{\dot{\mu}(\beta_{t+1}  z_{t+1}^\top\theta_*)\beta_{t+1}^2z_{t+1}^\top M^{-1}(\xi_{t},\theta_*)z_{t+1}}{t}\right]. 
\end{align*}
Then,
\begin{equation}\label{e2}
\begin{aligned}
&~~~\log\det(M(\xi_{t+1},\theta_*))-\log\det(M(\xi_{t},\theta_*))\\
&=\log \frac{\det(M(\xi_{t+1},\theta_*))}{\det(M(\xi_{t},\theta_*))}\\
&=\log\frac{\left(\frac{t}{t+1}\right)^d\det M(\xi_{t},\theta_*)\left[1+\frac{\dot{\mu}(\beta_{t+1}  z_{t+1}^\top \theta_*)\beta_{t+1}^2z_{t+1}^\top M^{-1}(\xi_{t},\theta_*) z_{t+1}}{t}\right]}{\det(M(\xi_{t},\theta_*))}\\
&=\log \left[1+\frac{\dot{\mu}(\beta_{t+1}  z_{t+1}^\top \theta_*)\beta_{t+1}^2z_{t+1}^\top M^{-1}(\xi_{t},\theta_*) z_{t+1}}{t}\right]-d\log \left(1+ \frac{1}{t}\right).
\end{aligned}
\end{equation}
By \eqref{e11} and \eqref{e2}, for all $t\geq t_1$, we have
\begin{equation}\label{diff1}
 \begin{aligned}
  \log\det(M(\xi_{t+1},\theta^*))-\log\det(M(\xi_{t},\theta^*))&\geq \log \left(1+\frac{d-\epsilon}{t}\right)-d\log \left(1+ \frac{1}{t}\right) \\
  &=\log\frac{1+(d-\epsilon)/t}{(1+1/t)^d}.
\end{aligned}   
\end{equation}
On the other hand, we have
\begin{equation}\label{diff2}
 \log\frac{1+(d-\epsilon)/t}{(1+1/t)^d}=\log\frac{1+(d-\epsilon)/t}{1+(d+c_t)/t},
\end{equation}
where we have used that $(1+1/t)^d=1+(d+c_t)/t$ with $c_t\geq 0, c_t\rightarrow 0$ as $t\rightarrow \infty$. We choose $t_2\geq t_1$ such that $c_t\leq (d-\epsilon)\epsilon$ for all $t\geq t_2$. Then for all $t\geq t_2$, we have
\begin{equation}\label{diff3}
\begin{aligned}
\log\frac{1+(d-\epsilon)/t}{1+(d+c_t)/t}   &\geq  \log\frac{1+(d-\epsilon)/t}{1+[d+(d-\epsilon)\epsilon]/t} \\
&=-\log \left\{1+\frac{1+[d+(d-\epsilon)\epsilon]/t-1-(d-\epsilon)/t}{1+(d-\epsilon)/t}\right\}\\
&=-\log \left[1+\frac{\epsilon(1+d-\epsilon)/t}{1+(d-\epsilon)/t}\right]\\
&\geq -\frac{1}{1+(d-\epsilon)/t}\frac{\epsilon(1+d-\epsilon)}{t}\\
&=-\frac{\epsilon(1+d-\epsilon)}{t+d-\epsilon}\\
&\geq -\epsilon,
\end{aligned}    
\end{equation}
where the second inequality is due to the fact that $\log(1+x)\leq x$ for $x\geq 0$.
By \eqref{diff1}, \eqref{diff2} and \eqref{diff3},  for all $t\geq t_2$, we conclude
\begin{equation}\label{222}
\log \det M(\xi_{t+1},\theta_*)-\log \det M(\xi_{t},\theta_*)\geq -\epsilon.
\end{equation}
Now we choose $t_3\geq t_2$ such that for all $t\geq t_3$,
\begin{equation}\label{ee0}
\begin{aligned}
 \log \left(1+\frac{d+\epsilon}{t}\right)-d\log \left(1+ \frac{1}{t}\right)&=\log \left(1+\frac{d+\epsilon}{t}\right)-\log \left(1+ \frac{d+c_t}{t}\right)\\
 &\geq \log \left(1+\frac{d+\epsilon}{t}\right)-\log \left[1+ \frac{d+\epsilon(1-\frac{t+d+\epsilon}{2t})}{t}\right]\\
 &=\log \frac{t+d+\epsilon}{t+d+\epsilon(1-\frac{t+d+\epsilon}{2t})}\\
 &=\log \frac{1}{1-\frac{\epsilon}{2t}}\\
 &\geq \frac{\epsilon}{2t},   
\end{aligned}
\end{equation}
where the first equality follows from $(1+1/t)^d=1+(d+c_t)/t$ with $c_t\geq 0, c_t\rightarrow 0$ as $t\rightarrow\infty$, and the first inequality is achieved by choosing $t_3\geq t_2$ such that $c_t\leq \epsilon(1-\frac{t+d+\epsilon}{2t})$ for all $t\geq t_3$, and the last inequality is due to the fact that $\log(1-x)\leq -x$ for $x<1$. Now, we propose the following lemma. 
\begin{lemma}\label{state}
Let $t\geq t_3$ and $\epsilon\in(0,1)$. If $\log \det M(\xi_t,\theta_*) \leq  \log \det M(\xi_*,\theta_*)-2\epsilon$, then, $\log \det M(\xi_{t+1},\theta_*)-\log \det M(\xi_t,\theta_*)\geq \frac{\epsilon}{2t}.$ 
\end{lemma}
\begin{proof}
Since the log-determinant function $\log \det (\cdot)$ is concave on the space of symmetric positive definite matrices \citep{Boyd2004}, by the first-order condition for the concave function, we have
\begin{align*}
 \log \det M(\xi_*,\theta_*)&\leq \log \det M(\xi_t,\theta_*)+ \langle M^{-1}(\xi_t,\theta_*), M(\xi_*,\theta_*)-M(\xi_t,\theta_*)\rangle\\
 &=\log \det M(\xi_t,\theta_*)+ \langle M^{-1}(\xi_t,\theta_*), \sum_{(z, \beta)\in \mathcal{Z}\times\mathcal{B}}[\xi_*(z, \beta)-\xi_t(z, \beta)]\dot{\mu}(\beta z^\top \theta_*)\beta^2z z^\top \rangle\\
 &=\log \det M(\xi_t,\theta_*)+  \sum_{(z, \beta)\in \mathcal{Z}\times\mathcal{B}}[\xi_*(z, \beta)-\xi_t(z, \beta)]\dot{\mu}(\beta z^\top \theta_*)\beta^2z^\top M^{-1}(\xi_t,\theta_*) z \\
 &\leq \log \det M(\xi_t,\theta_*)+\mathop{\max}_{(z, \beta)\in\mathcal{Z}\times\mathcal{B}}\dot{\mu}(\beta z^\top \theta_*)\beta^2z^\top M^{-1}(\xi_t,\theta_*)z,   
\end{align*}
where the first equality is from the fact that $\frac{\partial \log \det M}{\partial M}=(M^{-1})^\top$ for an invertible matrix $M$ \citep{harville2008matrix}, and the last inequality is because $\sum_{(z, \beta)\in \mathcal{Z}\times\mathcal{B}}\xi(z, \beta)=1$ with $\xi(z, \beta)\geq 0$. Therefore,
\begin{align*}
\log \det M(\xi_*,\theta_*)-\log \det M(\xi_t,\theta_*)
&\leq \mathop{\max}_{(z, \beta)\in\mathcal{Z}\times\mathcal{B}}\dot{\mu}(\beta z^\top \theta_*)\beta^2z^\top M^{-1}(\xi_t,\theta_*)z   \\
&\leq \dot{\mu}(\beta_{t+1} z_{t+1}^\top \theta_*)\beta_{t+1}^2z_{t+1}^\top M^{-1}(\xi_t,\theta_*)z_{t+1}-d+\epsilon,
\end{align*}
where the last inequality is from \eqref{e11}. Using the condition $\log \det M(\xi_t,\theta_*) \leq  \log \det M(\xi_*,\theta_*)-2\epsilon$ in Lemma \ref{state}, we obtain $$\dot{\mu}(\beta_{t+1} z_{t+1}^\top \theta^*)\beta_{t+1}^2z_{t+1}^\top M^{-1}(\xi_t,\theta^*)z_{t+1}\geq d+\epsilon.$$ Together with \eqref{e2}, we have
\begin{equation*}
\log \det M(\xi_{t+1},\theta^*)-\log \det M(\xi_t,\theta^*)\geq\log \left(1+\frac{d+\epsilon}{t}\right)-d\log\left(1+\frac{1}{t}\right)\geq \frac{\epsilon}{2t},
\end{equation*}
where the last inequality follows from \eqref{ee0}.
\end{proof}
There is some $t_4\geq t_3$ such that for all $t\geq t_4$
\begin{equation}\label{t0e3}
 \log\det M(\xi_t,\theta_*)>\log\det M(\xi_*,\theta_*)-2\epsilon   
\end{equation}
since otherwise $\log\det M(\xi_t,\theta^*)\rightarrow\infty$ from Lemma \ref{state}, which contradicts the fact that $\log\det M(\xi_t,\theta^*)$ is a bounded value, which follows from
\begin{align*}
\det M(\xi_t,\theta_*)&\leq \left[\frac{\text{tr}(M(\xi_t,\theta_*))}{d}\right]^{d}\\
&\leq \left[\frac{\sum_{(z, \beta)\in\mathcal{Z}\times\mathcal{B}}\xi_t(z, \beta) \dot{\mu}(\beta z^\top  \theta_*)\beta^2 \|z\|^2}{d}\right]^{d}\\
&\leq \left(\frac{C_\beta^2C_z^2}{4d}\right)^{d},
\end{align*}
where the last inequality is from Assumption \ref{assu2} and the facts $0\leq \dot{\mu}(\cdot)\leq 1/4$ and $\sum_{(z, \beta)\in\mathcal{Z}\times\mathcal{B}}\xi_t(z, \beta)=1$.
Combining \eqref{222} and \eqref{t0e3}, we have
\begin{equation}\label{t31}
 \log\det M(\xi_{t_4+1},\theta_*)\geq \log\det M(\xi_{t_4},\theta_*)-\epsilon> \log\det M(\xi_*,\theta^*)-3\epsilon.   
\end{equation}
If $\log\det M(\xi_{t_4+1},\theta_*)\leq \log\det M(\xi_*,\theta_*)-2\epsilon$, by Lemma \ref{state} and \eqref{t31}, we have
$$\log\det M(\xi_{t_4+2},\theta^*)\geq \log\det M(\xi_{t_4+1},\theta^*)>\log\det M(\xi_*,\theta^*)-3\epsilon.$$
If $\log\det M(\xi_{t_4+1},\theta_*)> \log\det M(\xi_*,\theta_*)-2\epsilon$, by \eqref{222} and \eqref{t31}, we have
$$\log\det M(\xi_{t_4+2},\theta_*)\geq \log\det M(\xi_{t_4+1},\theta_*)-\epsilon> \log\det M(\xi_*,\theta_*)-3\epsilon.$$
Continuing this argument inductively, we obtain, for all $t\geq t_4$,
$$\log\det M(\xi_{t},\theta_*)> \log\det M(\xi_*,\theta_*)-3\epsilon.$$
Therefore,
$$\mathop{\lim\inf}_{t\rightarrow \infty}\log\det M(\xi_t,\theta_*)\geq \log\det M(\xi_*,\theta_*)-3\epsilon.$$ Since $\epsilon\in(0,1)$ is arbitrary, we have
$$\mathop{\lim\inf}_{t\rightarrow \infty}\log\det M(\xi_t,\theta_*)\geq \log\det M(\xi_*,\theta_*).$$ Since $\xi_*=\mathop{\arg\max}_{\xi\in\mathcal{D}(\mathcal{Z},\mathcal{B})}M(\xi,\theta_*)$, we have
$$\mathop{\lim}_{t\rightarrow \infty}\log\det M(\xi_t,\theta_*)=\log\det M(\xi_*,\theta_*).$$ Due to the strict concavity of the criterion $\log\det(\cdot)$, the information matrix at $\theta_*$ of a locally $D$-optimal design at $\theta_*$ is unique. Therefore, $\mathop{\lim}_{t\rightarrow \infty} M(\xi_t,\theta_*)= M(\xi_*,\theta_*)$. By \eqref{conM}, we have 
$$\mathop{\lim}_{T\rightarrow \infty} M(\xi_T,\hat{\theta}_T)\xrightarrow{a.s.}M(\xi_*,\theta_*).$$

\subsection{Proof of Theorem \ref{thm3}}
By the definition of the sub-optimality \eqref{subopt}, we have
\begin{equation}\label{sub1}
 \textsf{SubOpt}(\hat{\pi}_T)=J(\pi^*)-J(\hat{\pi}_T)
=[J(\pi^*)-\hat{J}(\pi^*)]+[\hat{J}(\pi^*)-\hat{J}(\hat{\pi}_T)]+[\hat{J}(\hat{\pi}_T)-J(\hat{\pi}_T)].   
\end{equation}
Since $\hat{\pi}_T$ is the optimal policy under $\hat{J}(\pi)$, we have 
\begin{equation}\label{sub2}
\hat{J}(\pi^*)-\hat{J}(\hat{\pi}_T)\leq 0.    
\end{equation}
 By the definition of the pessimistic expected value function \eqref{pessj}, we obtain
\begin{align*}
\hat{J}(\hat{\pi}_T)-J(\hat{\pi}_T)=\mathop{\min}_{\theta\in \mathcal{C}(\hat{\theta}_T,\delta)} \mathbb{E}\theta^\top \phi(x,\hat{\pi}_T(x))-\mathbb{E}\theta_*^\top \phi(x,\hat{\pi}_T(x)).
\end{align*}
By Lemma \ref{lemma0}, we know that $\theta_*\in \mathcal{C}(\hat{\theta}_T,\delta)$ with probability at least $1-\delta$. Therefore, with probability at least $1-\delta$, we have
\begin{equation}\label{sub3}
\hat{J}(\hat{\pi}_T)-J(\hat{\pi}_T)\leq 0.    
\end{equation}
Combining \eqref{sub1}, \eqref{sub2} and \eqref{sub3},  with probability at least $1-\delta$, we have
\begin{align*}
 \textsf{SubOpt}(\hat{\pi}_T)&\leq J(\pi^*)-\hat{J}(\pi^*)\\
 &=\mathbb{E}\theta_*^\top \phi(x,\pi^*(x))-\mathop{\min}_{\theta\in \mathcal{C}(\hat{\theta}_T,\delta)} \mathbb{E}\theta^\top \phi(x,\pi^*(x))\\
 &=\mathop{\max}_{\theta\in \mathcal{C}(\hat{\theta}_T,\delta)} \mathbb{E}(\theta_*-\theta)^\top \phi(x,\pi^*(x))\\
 &=\mathop{\max}_{\theta\in \mathcal{C}(\hat{\theta}_T,\delta)} \mathbb{E}(\theta_*-\hat{\theta}_T+\hat{\theta}_T-\theta)^\top \phi(x,\pi^*(x))\\
 &=\mathbb{E}(\theta_*-\hat{\theta}_T)^\top \phi(x,\pi^*(x))+\mathop{\max}_{\theta\in \mathcal{C}(\hat{\theta}_T,\delta)} \mathbb{E}(\hat{\theta}_T-\theta)^\top \phi(x,\pi^*(x)).
\end{align*}
By the definition of $\mathcal{C}(\hat{\theta}_T,\delta)$ in \eqref{ci}, we obtain 
\begin{equation*}
\begin{aligned}
\mathop{\max}_{\theta\in \mathcal{C}(\hat{\theta}_T,\delta)} \mathbb{E}(\hat{\theta}_T-\theta)^\top \phi(x,\pi^*(x))&\leq \mathop{\max}_{\theta\in \mathcal{C}(\hat{\theta}_T,\delta)} \mathbb{E}\|\hat{\theta}_T-\theta\|_{\bar{H}_T(\hat{\theta}_T)} \|\phi(x,\pi^*(x))\|_{\bar{H}^{-1}_T(\hat{\theta}_T)}\\
&\leq \gamma(T, d, \delta) \mathbb{E} \|\bar{H}^{-1/2}_T(\hat{\theta}_T)\phi(x,\pi^*(x))\|_2
\end{aligned}
\end{equation*}
By Lemma \ref{lemma0}, we know that $\theta_*\in \mathcal{C}(\hat{\theta}_T,\delta)$ with probability at least $1-\delta$. Therefore,
$$ \textsf{SubOpt}(\hat{\pi}_T)\leq 2\gamma(T, d, \delta)  \mathbb{E} \|\bar{H}^{-1/2}_T(\hat{\theta}_T)\phi(x,\pi^*(x))\|.$$
By Theorem \ref{thm2}, we have $\bar{H}_T(\hat{\theta}_T)=M(\xi_T, \hat{\theta}_T)\xrightarrow{a.s.}M(\xi_*, \theta_*)$. Therefore, there exists a constant $T_0$ such that $\mathbb{E} \|\bar{H}^{-1/2}_T(\hat{\theta}_T)\phi(x,\pi^*(x))\|_2\leq 2\|M^{-1/2}(\xi_*, \theta_*)\mathbb{E} \phi(x,\pi^*(x))\|_2$ for all $T>T_0$ with probability 1. Thus, when $T>T_0$, with probability at least $1-\delta$, we have 
\begin{align*}
 \textsf{SubOpt}(\hat{\pi}_T)&\leq 2\gamma(T, d, \delta) \|M^{-1/2}(\xi_*, \theta_*)\mathbb{E} \phi(x,\pi^*(x))\|\\
 &=2\sqrt{\frac{C_1}{T} \left[d\log \left(e+\frac{C_2 T}{d}\right)+\log \frac{2}{\delta}\right]}\|M^{-1/2}(\xi_*, \theta_*)\mathbb{E} \phi(x,\pi^*(x))\|_2.
\end{align*}
\subsection{Proof of Theorem \ref{thm4}}\label{appj}
The proof of Theorem \ref{thm4} follows a similar strategy to that of Theorem \ref{thm3}. For completeness, we provide the proof here.
By the definition of the sub-optimality \eqref{subopt}, we have
\begin{equation}\label{4sub1}
 \textsf{SubOpt}(\hat{\pi}_T)=J(\pi^*)-J(\hat{\pi}_T)
=[J(\pi^*)-\hat{J}(\pi^*)]+[\hat{J}(\pi^*)-\hat{J}(\hat{\pi}_T)]+[\hat{J}(\hat{\pi}_T)-J(\hat{\pi}_T)].   
\end{equation}
Since $\hat{\pi}_T$ is the optimal policy under $\hat{J}(\pi)$, we have 
\begin{equation}\label{4sub2}
\hat{J}(\pi^*)-\hat{J}(\hat{\pi}_T)\leq 0.    
\end{equation}
 By the definition of the pessimistic expected value function \eqref{pessj}, we obtain
\begin{align*}
\hat{J}(\hat{\pi}_T)-J(\hat{\pi}_T)=\mathop{\min}_{\theta\in \mathcal{C}(\hat{\theta}_T,\delta)} \mathbb{E}\theta^\top \phi(x,\hat{\pi}_T(x))-\mathbb{E}\theta_*^\top \phi(x,\hat{\pi}_T(x)).
\end{align*}
As in the proof of Lemma \ref{lemma0}, we can show that $\theta_*\in \mathcal{C}(\hat{\theta}_T,\delta)$ with probability at least $1-\delta$. Therefore, with probability at least $1-\delta$, we have
\begin{equation}\label{4sub3}
\hat{J}(\hat{\pi}_T)-J(\hat{\pi}_T)\leq 0.    
\end{equation}
Combining \eqref{4sub1}, \eqref{4sub2} and \eqref{4sub3},  with probability at least $1-\delta$, we have
\begin{align*}
 \textsf{SubOpt}(\hat{\pi}_T)&\leq J(\pi^*)-\hat{J}(\pi^*)\\
 &=\mathbb{E}\theta_*^\top \phi(x,\pi^*(x))-\mathop{\min}_{\theta\in \mathcal{C}(\hat{\theta}_T,\delta)} \mathbb{E}\theta^\top \phi(x,\pi^*(x))\\
 &=\mathop{\max}_{\theta\in \mathcal{C}(\hat{\theta}_T,\delta)} \mathbb{E}(\theta_*-\theta)^\top \phi(x,\pi^*(x))\\
 &=\mathop{\max}_{\theta\in \mathcal{C}(\hat{\theta}_T,\delta)} \mathbb{E}(\theta_*-\hat{\theta}_T+\hat{\theta}_T-\theta)^\top \phi(x,\pi^*(x))\\
 &=\mathbb{E}(\theta_*-\hat{\theta}_T)^\top \phi(x,\pi^*(x))+\mathop{\max}_{\theta\in \mathcal{C}(\hat{\theta}_T,\delta)} \mathbb{E}(\hat{\theta}_T-\theta)^\top \phi(x,\pi^*(x)).
\end{align*}
By the definition of $\mathcal{C}(\hat{\theta}_T,\delta)$ in \eqref{ci}, we obtain 
\begin{equation*}
\begin{aligned}
\mathop{\max}_{\theta\in \mathcal{C}(\hat{\theta}_T,\delta)} \mathbb{E}(\hat{\theta}_T-\theta)^\top \phi(x,\pi^*(x))&\leq \mathop{\max}_{\theta\in \mathcal{C}(\hat{\theta}_T,\delta)} \mathbb{E}\|\hat{\theta}_T-\theta\|_{\bar{H}_T(\hat{\theta}_T)} \|\phi(x,\pi^*(x))\|_{\bar{H}^{-1}_T(\hat{\theta}_T)}\\
&\leq \gamma(T, d, \delta) \mathbb{E} \|\bar{H}^{-1/2}_T(\hat{\theta}_T)\phi(x,\pi^*(x))\|
\end{aligned}
\end{equation*}
As in the proof of Lemma \ref{lemma0}, we can show that $\theta_*\in \mathcal{C}(\hat{\theta}_T,\delta)$ with probability at least $1-\delta$. Therefore,
$$ \textsf{SubOpt}(\hat{\pi}_T)\leq 2\gamma(T, d, \delta)  \mathbb{E} \|\bar{H}^{-1/2}_T(\hat{\theta}_T)\phi(x,\pi^*(x))\|.$$
As in the proof of Theorem \ref{thm2}, we can show $\bar{H}_T(\hat{\theta}_T)=M(\xi_T, \hat{\theta}_T)\xrightarrow{a.s.}M(\xi_*, \theta_*)$. Therefore, there exists a constant $T_0$ such that $\mathbb{E} \|\bar{H}^{-1/2}_T(\hat{\theta}_T)\phi(x,\pi^*(x))\|\leq 2\|M^{-1/2}(\xi_*, \theta_*)\mathbb{E} \phi(x,\pi^*(x))\|$ for all $T>T_0$ with probability 1. Thus, when $T>T_0$, with probability at least $1-\delta$, we have 
\begin{align*}
 \textsf{SubOpt}(\hat{\pi}_T)&\leq 2\gamma(T, d, \delta) \|M^{-1/2}(\xi_*, \theta_*)\mathbb{E} \phi(x,\pi^*(x))\|\\
 &=2\sqrt{\frac{C_3}{T} \left[d\log \left(e+\frac{C_4 T}{d}\right)+\log \frac{2}{\delta}\right]}\|M^{-1/2}(\xi_*, \theta_*)\mathbb{E} \phi(x,\pi^*(x))\|,
\end{align*}
for some positive constants $C_3$ and $C_4$.

\section{Supporting Lemmas}\label{appsu}
\begin{lemma}(Theorem 18.1.1. \citep{harville2008matrix})\label{harville}
Let $R$ represent an $n\times n$ matrix, $S$ an $n\times m$ matrix, $\tilde{T}$ an $m\times m$ matrix, and $U$ an $m\times n$ matrix. If $R$ and $\tilde{T}$ are nonsingular, then
$$\det (R+S\tilde{T}U)=\det R\det \tilde{T}\det(\tilde{T}^{-1}+UR^{-1}S).$$
\end{lemma}

\begin{lemma} (Theorem 1.1 \citep{Tropp2012}) \label{tropp1}
Consider a finite sequence $\{\mathbf{X}_k\}$ of independent, random, self-adjoint matrices with dimension $d$. Assume that each random matrix satisfies
$$\mathbf{X}_k\succeq \boldsymbol{0}\ \textit{and}\ \lambda_{max}(\mathbf{X}_k)\leq R\ \textit{almost surely.}$$
Define
$$\mu_{min}:=\lambda_{min}\bigg(\sum_{k}\mathbb{E}\mathbf{X}_k\bigg)\ \text{and}\ \mu_{max}:=\lambda_{max}\bigg(\sum_{k}\mathbb{E}\mathbf{X}_k\bigg).$$ Then for $\zeta\in [0,1]$,
\begin{align*}
&\mathbb{P}\bigg\{\lambda_{min}\bigg(\sum_{k}\mathbf{X}_k\bigg)\leq (1-\delta)\mu_{min}\bigg\}\leq d\bigg[\frac{e^{-\delta}}{(1-\delta)^{1-\delta}}\bigg]^{\mu_{min}/R}\ \text{for}\ \delta\in[0,1],\ \text{and}\\
&\mathbb{P}\bigg\{\lambda_{max}\bigg(\sum_{k}\mathbf{X}_k\bigg)\leq (1+\delta)\mu_{max}\bigg\}\leq d\bigg[\frac{e^{\delta}}{(1+\delta)^{1+\delta}}\bigg]^{\mu_{max}/R}\ \text{for}\ \delta\geq 0.
\end{align*}
\end{lemma}

\begin{lemma}(Lemma 4 \citep{pronzato2010one})\label{pron}
If for any $\delta>0$
$$\mathop{\lim\inf}_{N\rightarrow\infty}\mathop{\inf}_{\|\theta-\theta_*\|\geq \delta}[L_N(\theta_*)-L_N(\theta)]>0 \ \text{almost surely},$$
then $\hat{\theta}_{ML}^N\xrightarrow{a.s.}\theta_*.$
\end{lemma}

\begin{lemma} (Theorem 3.2 \citep{Martingale})\label{lemma8}
Let $\{S_{ni}, \mathcal{F}_{n,i}, 1\leq i\leq k_n, n\geq 1\}$ be a zero-mean, square-integrable martingale array with differences $X_{ni}$, and let $\eta^2$ be an a.s. finite r.v. Suppose that 
\begin{equation}\label{l8e1}
    \mathop{\max}_i|X_{ni}|\xrightarrow{p} 0,
\end{equation}
\begin{equation}\label{l8e2}
    \sum_iX_{ni}^2\xrightarrow{p} \eta^2,
\end{equation}
\begin{equation}\label{l8e3}
\mathbb{E}\mathop{\max}_i X_{ni}^2   \ \text{is bounded in}\ n,
\end{equation}
and the $\sigma$-fields are nested:
\begin{equation}\label{l8e4}
\mathcal{F}_{n,i}\subseteq \mathcal{F}_{n+1,i}\ \text{for}\ 1\leq i\leq k_n, n\geq 1.
\end{equation}
Then $S_{nk_n}=\sum_i X_{ni}\xrightarrow{d} Z$ (stably), where the r.v. $Z$ has characteristic function $\mathbb{E}e^{-\frac{1}{2}\eta^2t^2}$.
\end{lemma}

\begin{lemma} (Corollary 3.1 \citep{Martingale})\label{lemma9}
If \eqref{l8e1} and \eqref{l8e3} are replaced by the conditional Lindeberg condition 
$$\text{for all}\  \epsilon>0,\ \sum_i\mathbb{E}[X_{ni}^2\mathbb{I}(|X_{ni}|>\epsilon)|\mathcal{F}_{n, i-1}]\xrightarrow{p} 0,$$
if \eqref{l8e2} is replaced by an analogous condition on the conditional variance:
$$V_{nk_n}^2=\sum \mathbb{E}(X_{ni}^2|\mathcal{F}_{n, i-1})\xrightarrow{p}\eta^2,$$
and if \eqref{l8e4} holds, then the conclusion of Lemma \ref{lemma8} remains true.
\end{lemma}

\begin{lemma} (Theorem 29.4 \citep{billingsley1995})\label{billingsley}
For random vectors $X_n=(X_{n1}, \cdots,X_{nk})$ and $Y=(Y_1,\cdots,Y_k)$,  a necessary and sufficient condition for $X_n\xrightarrow{d} Y$ is that $\sum_{u=1}^kt_uX_{nu}\xrightarrow{d}\sum_{u=1}^kt_uY_u$ for each $(t_1,\cdots,t_k)\in\mathbb{R}^k$.
\end{lemma}
\begin{lemma}(Lemma 3 \citep{lee2024}) \label{lemma10}
Let $X_1,\cdots, X_t$ be a martingale difference sequence satisfying $\mathop{\max}_s|X_s|\leq R$ a.s., and let $\mathcal{F}_s$ be the $\sigma$-field generated by $(X_1,\cdots, X_s)$. Then for any $\delta\in (0,1)$ and any $\eta\in(0, 1/R]$, the following holds with probability at least $1-\delta$:
$$\sum_{s=1}^tX_s\leq (e-2)\eta\sum_{s=1}^t\mathbb{E}(X_s^2|\mathcal{F}_{s-1})+\frac{1}{\eta}\log\frac{1}{\delta}, \forall t\geq 1.$$
\end{lemma}
\begin{lemma}(Lemma 8 \citep{das2024provably}) \label{lemma11}
Let $z, z'\in\mathbb{R}$ and $\tilde{\alpha}(z, z'):=\int_{0}^1(1-v)\dot{\mu}(z+v(z'-z))dv$. Then for some $C>1$ (1.01 suffices), 
$$\tilde{\alpha}(z, z')\geq \frac{\dot{\mu}(z')}{C(2+|z-z'|)^2}.$$
\end{lemma}

\end{document}